\documentclass{article}

\PassOptionsToPackage{numbers, sort&compress}{natbib}
\usepackage[preprint]{neurips_2026}

\usepackage{graphicx}
\usepackage{subcaption}
\usepackage{booktabs}
\usepackage{wrapfig}
\usepackage{enumitem}
\usepackage{csquotes}
\usepackage{hyperref}

\usepackage[table]{xcolor}
\usepackage{xspace}

\usepackage{amsmath}

\usepackage{amssymb}
\usepackage{mathtools}
\usepackage{amsthm}

\usepackage{multirow}

\usepackage[most]{tcolorbox}
\usepackage{tabularx}
\usepackage{makecell}
\usepackage{array}
\newcommand{\eg}{{\it e.g.,~}}
\newcommand{\ie}{{\it i.e.,~}}

\usepackage[capitalize,noabbrev]{cleveref}

\usepackage[textsize=tiny]{todonotes}

\usepackage{algorithm}
\usepackage{algpseudocode}
\usepackage{placeins}

\usepackage{xcolor}
\usepackage{pifont}

\newcommand{\cmark}{\textcolor{green!60!black}{\ding{51}}}
\newcommand{\xmark}{\textcolor{red!70!black}{\ding{55}}}

\theoremstyle{plain}
\newtheorem{theorem}{Theorem}[section]
\newtheorem{proposition}[theorem]{Proposition}

\theoremstyle{definition}
\newtheorem{definition}[theorem]{Definition}

\theoremstyle{remark}

\newcommand{\xhdr}[1]{\vspace{0mm}\noindent{\bfseries #1.}}

\newcommand{\K}{\mathcal{K}}

\newcommand{\W}{\mathcal{W}}
\newcommand{\X}{\mathcal{X}}

\newcommand{\PD}{\rm{PD}}

\newcommand{\method}{\textsc{TopoTuner}\xspace}

\newcommand{\graycell}[1]{\textcolor{black!60}{#1}}

\title{TopoTuner: Topological Finetuning of Large Language Models}

\author{
  Abdulkadir Erol \\
  Department of Computer Science \\
  University of Central Florida \\
  Orlando, USA \\
  \texttt{abdulkadir.erol@ucf.edu}
  \And
  Yash Mahajan \\
  Department of Computer Science \\
  University of Central Florida \\
  Orlando, USA \\
  \texttt{yash.mahajan@ucf.edu}
  \And
  Vepaul Hariprashad \\
  Department of Computer Science \\
  University of Central Florida \\
  Orlando, USA \\
  \texttt{vepaul.hariprashad@ucf.edu}
  \And
  Baha Rababah \\
  Department of Computer Science \\
  University of Manitoba \\
  Canada \\
  \texttt{rababahb@myumanitoba.ca}
  \And
  Santu Karmaker \\
  Department of Computer Science \\
  University of Central Florida \\
  Orlando, USA \\
  \texttt{santu.karmaker@ucf.edu}
  \And
  Cuneyt G. Akcora \\
  Department of Finance \\
  University of Central Florida \\
  Orlando, USA \\
  \texttt{cuneyt.akcora@ucf.edu}
  \And
  Mubarak Shah \\
  Department of Computer Science \\
  University of Central Florida \\
  Orlando, USA \\
  \texttt{mubarak.shah@ucf.edu}
}
 
\begin{document}

\maketitle

\begin{abstract}
Full fine-tuning remains a strong way to adapt pretrained LLMs, but it updates all weights and can be expensive. LoRA reduces the number of trainable parameters, but it does not directly answer which pretrained components should be trained and which can be frozen during adaptation. We introduce \method, a topology-guided fine-tuning framework for selective freezing of attention projection matrices. \method treats each projection matrix as a row cloud and uses Wasserstein distances between persistence diagrams to measure how its topology changes during fine-tuning.

\method learns a reusable freezing profile from a source dataset and transfers it to efficiently fine-tune models on out-of-domain datasets, evaluating whether task-specific topological drift generalizes across question answering and sentiment analysis tasks.

 Across LLaMA-3.1-8B, Mistral-7B-v0.3, and Qwen3-8B-Base, \method is competitive with full fine-tuning while training only 1-2\% of the model parameters, and outperforms LoRA in 7 out of 9 model-dataset settings, which can change up to 39.57\% of the projection parameters. Along with minimized updates, \method reduces training time by 20.4\% relative to full fine-tuning and 5.5\% relative to LoRA on average. \method opens a new direction for reusable freezing profiles, where fine-tuning behavior learned on one dataset can be shared across multiple tasks.
\end{abstract}

%
%
%

\section{Introduction}

Fine-tuning~\cite{kumar2025llm,xu2026parameter} is the standard mechanism for adapting pretrained large language models (LLMs) to downstream tasks, such as question answering~\cite{yue2025survey} . Full fine-tuning updates all model weights and remains a strong baseline, but it is expensive and can damage capabilities acquired during pretraining with risks of catastrophic forgetting~\cite{luo2025empirical}, overfitting to narrow task distributions~\cite{huang2025confident}, and instability in alignment behavior~\cite{liu2025targeted}. 

Parameter-efficient methods (PEFT)~\cite{xu2026parameter}, such as LoRA~\cite{hu2022lora}, reduce the number of trainable parameters and often recover strong task performance. However, matching downstream accuracy does not answer a basic structural question: which parts of the pretrained model actually need to change during adaptation? Furthermore, although LoRA uses low-rank updates, it can still apply updates to many, or even all, weight matrices without explicit control over pretrained components. 

The question matters because fine-tuning is rarely single-task in practice~\cite{hu2025beyond}. Models are adapted across domains and often reused across datasets. A method that only controls the number of trainable parameters does not show which parts of the model change, whether the same parts change across related tasks, or at what epoch the main structural change has already occurred.  Standard summaries such as spectra, norms, and sparsity measure the size or rank structure of an update, but they do not directly describe how the structure of a weight matrix changes~\cite{shuttleworth2025lora}. 

Topology~\cite{carlsson2009topology} offers a principled way to study how model parameters reorganize during fine-tuning. We can treat each attention projection matrix as a point cloud formed by its rows and use topological data analysis~\cite{chazal2021introduction} to obtain stable multiscale summaries of this geometry. The distance between the persistence diagrams of pretrained and fine-tuned matrices then gives a matrix-level measure of structural change, which can be used to rank matrices for freezing. This view also suggests a sufficient condition for reusing such freezing profiles across tasks: the fine-tuning trajectories need not be identical; it suffices that the induced gradient fields remain close to the pretrained model.

This insight allows us to introduce \method, a topology-guided framework for measuring and controlling LLM fine-tuning. The method tracks topological drift, the structural change of each weight matrix relative to its pretrained state, and uses this signal to build reusable freezing profiles and to stop training once the main structural change has occurred. This reframes fine-tuning from a parameter-count problem into a structural-control problem: instead of asking only how many parameters to update, \method asks which pretrained components should be allowed to reorganize.

We evaluate \method on three language models across six task families: question answering, sentiment analysis, information retrieval, summarization, instruction following, and code generation. Our results show that drift profiles are model-dependent, ruling out a fixed projection-level rule across architectures. Yet, within a fixed model, profiles learned from one dataset can transfer to other tasks, reducing full fine-tuning time by $20.4\%$ and LoRA fine-tuning time by $5.5\%$ on average while preserving competitive accuracy.

Topological drift also reveals how freezing changes the adaptation path. Freezing low-drift matrices can leave the main route open and spread updates across many layers, while freezing high-drift matrices, as shown in Figure~\ref{fig:driftintro}, can block that route and force the remaining trainable matrices to absorb updates in a more concentrated pattern, reducing overfitting. This redistribution explains why drift is useful for matrix selection, task comparison, and early stopping.

%
%

\begin{figure}
    \centering
    \vspace{-35px}
    \includegraphics[width=0.9\linewidth]{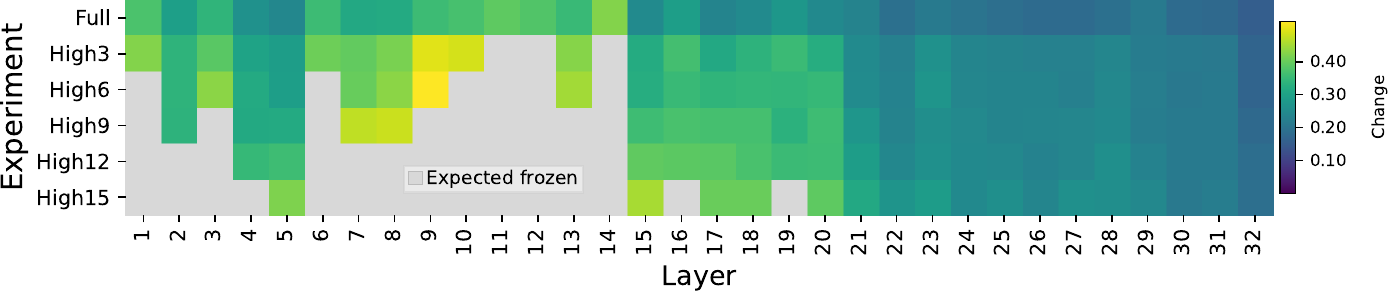}
    \caption{LLaMA $O$-projection change under high-drift Wasserstein freezing ($3$ to $15$ layers). As more matrices are frozen during training, the remaining trainable matrices absorb updates in a more concentrated manner;   selective freezing redistributes adaptation across layers.}
    \vspace{-10px}
    \label{fig:driftintro}
\end{figure}

Our contributions are as follows.
\begin{itemize}[leftmargin=*] 
\item We introduce \method, a topology-guided framework that measures how attention projection matrices reorganize during fine-tuning and uses this signal to observe and control which matrices remain trainable. \textbf{To our knowledge, this is the first work that uses topological summaries of LLM weights to construct reusable freezing profiles across tasks.}

\item We define matrix-level topological drift through Wasserstein distances between persistence diagrams of pretrained and fine-tuned weights. We show that this drift captures structural change beyond entrywise magnitude, and use it for selective matrix freezing and topology-based early stopping.

\item We compare \method against PEFT and full-fine-tuning baselines. Across transferred-knowledge experiments, \method outperforms LoRA in $7$ out of $9$ model-dataset settings while reducing wall-clock training time by $20.4\%$ on average and using $88.36\%$ fewer parameter updates than LoRA. \method stops fine-tuning after two to three epochs in most runs, reducing the six-epoch budget by $50\%$--$66.7\%$ while retaining near-final validation accuracy.
\end{itemize}


%

\begin{figure}
\vspace{-15px}
    \centering
    \includegraphics[width=\linewidth]{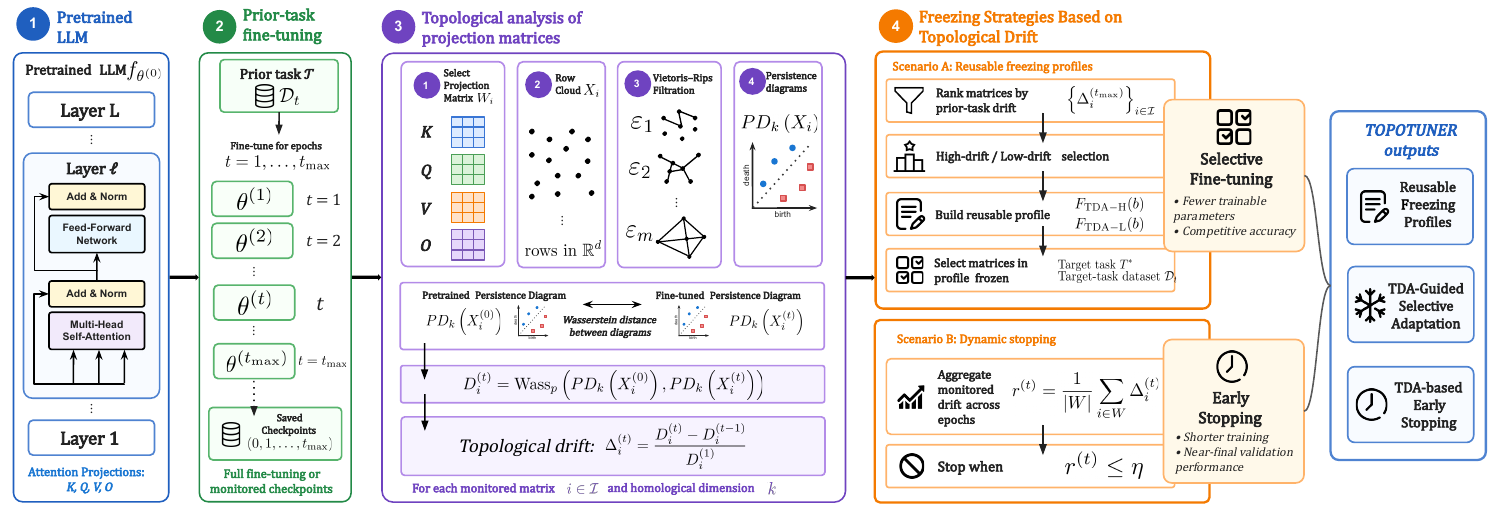}
    \caption{\textbf{Overview of \method.} \method uses a prior fine-tuning task to compute the topological drift of attention projection matrices. The computed topological drift is used to (A) build reusable freezing profiles for selective fine-tuning on new datasets and tasks, and to (B) track structural change across epochs for early stopping. Freezing profiles are provided in Appendix~\ref{app:freezing-profiles}.}
    \vspace{-20px}
    \label{fig:system}
\end{figure}

\vspace{-5px}
\section{Related Work}
\vspace{-10px}

\xhdr{Topological methods} Persistent homology (PH) has emerged as a powerful tool for analyzing neural network representations. Pérez-Fernández et al. treat a neural network as an abstract simplicial complex and extract its topological fingerprint via persistent homology~\cite{perez2021characterizing}. The authors define a PH-based descriptor that can serve as a similarity measure between trained models. Moreover, persistent homology can reveal layer-wise organizational differences. Purvine et al. analyze hidden representations of CNNs by treating each layer’s activations as a high-dimensional point cloud~\cite{purvine2023experimental}. The authors compute persistence-based distances between layers’ activation spaces and show that these distances capture meaningful differences in layer behavior. These works establish that persistent homology, via persistence diagrams and their distances, provides a rigorous framework to quantify structural similarities in deep networks.  However, to the best of our knowledge, no prior work has used topological signals to guide LLM fine-tuning via selective freezing.

\xhdr{Fine-tuning strategies} Fine-tuning large transformer models can be approached at the granularity of individual weight matrices or entire layers. A variety of matrix-level adaptation techniques, often termed parameter-efficient fine-tuning (PEFT~\cite{xu2026parameter}), restrict which parameters are updated to reduce computational cost. A prominent example is LoRA~\cite{hu2022lora}, which inserts low-rank update matrices into weight layers. Building on this idea,  “Deep LoRA”~\cite{yarascompressible} applies a low-rank decomposition to multiple transformer layers simultaneously, to confine the fine-tuning updates to a low-dimensional subspace at each layer. This strategy mitigates overfitting with limited data. On the other end of the spectrum, layer-level strategies decide which whole layers to train or freeze. Zheng et al.~\cite{zhengspurious} freeze the bottom ~$25\%$ of transformer layers, and achieve better downstream performance in continual learning scenarios. In addition to static freezing, recent methods dynamically adjust training at the layer level. A representative approach is DropBP, which randomly omits gradient backpropagation through certain layers at each update, guided by a sensitivity metric that estimates each layer’s impact on the training loss~\cite{woo2024dropbp}. Layers with lower sensitivity receive a higher drop rate, effectively being temporarily frozen during some training steps. This adaptive schedule trains only the most \textquote{critical} layers at full frequency and reduces the time costs. In contrast, we use topological drift to measure how each attention projection matrix reorganizes relative to the pretrained model, and use this signal to build reusable freezing profiles and to decide when fine-tuning has stabilized.

\vspace{-5px}
\section{Preliminaries}
\vspace{-10px}

Let $f_\theta$ denote a pretrained large language model, parameterized by a collection of weight matrices $\theta = \{ W_i \}_{i=1}^{m}$, where $W_i \in \mathbb{R}^{d_i \times d_i'}$. We use $\mathcal{I} \subseteq \{1,\dots,m\}$ to denote the index set of matrices.

\xhdr{Task and fine-tuning}
Given a downstream task $\mathcal{T}$ with training data $\mathcal{D}_{\mathcal{T}} = \{(x_r, y_r)\}_{r=1}^{N}$, fine-tuning solves $\tilde{\theta} =
\arg\min_{\vartheta}\frac{1}{N}\sum_{r=1}^{N}\ell\big(y_r,\ f_{\vartheta}(x_r)\big),$
where $\ell$ is a supervised loss function. The resulting model is denoted by $f_{\tilde{\theta}}$ with parameters $\tilde{\theta} = \{ \tilde W_i \}_{i=1}^{m}$.

\xhdr{Topological tools} Persistent homology is a tool from topological data analysis that summarizes how topology evolves across a filtration~\citep{chazal2021introduction,hensel2021survey}. Let $\X={x_1,\dots,x_n}\subset \mathbb{R}^{d}$ be a point cloud equipped with a distance function $\delta$. We construct the Vietoris--Rips filtration $\{\mathrm{VR}_\varepsilon(\X)\}_{\varepsilon \ge 0}$, where simplices are added according to pairwise distance at scale $\varepsilon$~\cite{chazal2021introduction}. For simplicity, we write this filtration as $\K_1 \subset \K_2 \subset \dots \subset \K_n$.

\xhdr{Topological features and persistence diagrams} For each complex $\K_i$ in the filtration, let $H_k(\K_i)$ denote its $k$th homology group, whose nontrivial classes represent $k$-dimensional topological features such as connected components for $k=0$, loops for $k=1$, and voids for $k=2$. As the filtration grows, features appear and disappear;  $q_j = (\textcolor{teal}{b_j}, \textcolor{orange}{d_j}) \in \PD(\X)$ represents the \textcolor{teal}{birth} and \textcolor{orange}{death} times of a topological feature in $\X$ within a Persistence Diagram ($\PD$). The $k$th persistence diagram, denoted by $\mathrm{PD}_k(\X)$, is the multiset of \textcolor{teal}{birth}-\textcolor{orange}{death} pairs for all such features. Features with larger persistence $d_j-b_j$ are often interpreted as more salient. 

\xhdr{$p$-Wasserstein distance} Given two persistence diagrams $\PD_k(\X)$ and $\PD_k(\X')$ in the same homological dimension $k$, their $p$-Wasserstein distance is
$\W{ass}_p(\PD_k(\X),\PD_k(\X')) = \min_{\phi}\left( \sum_j \|q_j-\phi(q_j)\|_\infty^p\right)^{1/p},$
where $\phi$ ranges over partial matchings with diagonal matching allowed. 
This distance measures how much the underlying topological summaries are different between two point clouds.

\vspace{-7px}
\section{Topological Characteristics of LLM Fine-tuning}
\label{sec:sec4}

Topological data analysis is based on the principle that data has shape and shape has meaning \cite{carlsson2009topology}. Accordingly, we treat each projection (i.e., weight matrix) in an LLM as a geometric object and study how its internal structure evolves during fine-tuning across epochs.  Our topology-based hypothesis is that persistence diagrams of a projection provide a topological summary of this organization, and that the distance between pre- and post-fine-tuning diagrams yields a useful projection-level measure of structural change. Our goal is to use this measure to make fine-tuning more controlled: identify which projections can be frozen, reuse those freezing profiles across datasets, and stop training once the remaining trainable projections have stabilized.

Let $f_\theta$ be a pretrained transformer with layer set $\mathcal{L} = \{1,\dots,L\}$. For each layer $l \in \mathcal{L}$, let $W_l^{(q)}$, $W_l^{(k)}$, $W_l^{(v)}$, and $W_l^{(o)}$ denote the query, key, value, and output projections. Let $\mathcal{I}$ denote the collection of matrices considered for ranking and freezing.

\xhdr{Topological sensitivity pipeline}
For each matrix $W_i \in \mathbb{R}^{d_i \times d_i'}$, we form the row cloud $X_i = \{x_1^{(i)},\dots,x_{d_i}^{(i)}\} \subset \mathbb{R}^{d_i'}$, where $x_a^{(i)}$ is the $a$th row of $W_i$.  We equip $X_i$ with a distance metric $\delta_i$ and, for each threshold $\varepsilon \ge 0$, build the graph $\mathcal{G}_\varepsilon^{(i)}$ with vertex set $X_i$ and edges between two data points $(a,b)$ whenever their distance $\delta_i(a,b) \le \varepsilon$. We use rows because they correspond to the output channels produced by the projection, which are the units whose relations we want to compare. Columns describe how one input coordinate contributes across all output channels and therefore define a different object. Promoting every clique in $\mathcal{G}_\varepsilon^{(i)}$ to a simplex gives the Vietoris--Rips complex $\mathrm{VR}_\varepsilon^{(i)}$~\cite{hensel2021survey}, and the nested family $\{\mathrm{VR}_\varepsilon^{(i)}\}_{\varepsilon \ge 0}$ defines a filtration. From this filtration we compute persistence diagrams $PD_k(X_i)$ in homological dimensions $k \in \{0,1,\ldots\}$. 

\begin{definition}[Topological distance to pretrained model]
\label{def:distance}
Let $W^{(0)}_i$ be the pretrained version of matrix $i \in \mathcal{I}$, and let $\tilde{W}_i^{(t)}$ denote its version at the fine-tuning epoch $t\ge 1$. For a fixed homological dimension $k$, we define the topological distance of matrix $i$ at $t$ by $D_i^{(t)} := \W{ass}_p(\mathrm{PD}_k(X^{(0)}_i), \mathrm{PD}_k(X_i^{(t)}))$.
\end{definition}

Distances are always computed with respect to the pretrained model (denoted as epoch $t=0$). We next define drift, which uses the distance from the pretrained model to the first epoch as a reference.

\begin{definition}[Topological drift]
\label{def:topological-drift}
Let $D_i^{(t)}$ denote the fine-tuning induced topological distance for matrix $i$ at the fine-tuning epoch $t$. We define the  topological drift at $t \geq 2$ by
$\Delta_i^{(t)} := \frac{D_i^{(t)} - D_i^{(t-1)}}{D_i^{(1)} }$.
\end{definition}

The denominator (first epoch as a reference) provides a per-layer, epoch-to-epoch scaling of change.

\xhdr{Stability}
The usefulness of $\Delta_i$ rests on the stability principle of PH: small perturbations in the underlying metric induce small perturbations in the resulting diagrams. In our setting, this means that if fine-tuning changes pairwise distances between rows only slightly, then the corresponding topological drift must also remain small. The next proposition makes this statement explicit.

\begin{proposition}[Stability of topological sensitivity]
\label{def:sensitivity}
Let $X=\{x_1,\dots,x_d\}$ and $\tilde X=\{\tilde x_1,\dots,\tilde x_d\}$ be the row clouds of $W_i$ and $\tilde W_i$, indexed consistently. Let $d$ and $\tilde d$ be the metrics used to construct the Vietoris--Rips filtrations of $X$ and $\tilde X$. If $\|d-\tilde d\|_\infty := \max_{a,b}|d(a,b)-\tilde d(a,b)| \le \eta$, then for every $k \ge 0$ we have $d_B(PD_k(X),PD_k(\tilde X)) \le \eta$, and therefore $W_p(PD_k(X),PD_k(\tilde X)) \le N_k^{1/p}\eta$, where $d_B$ is the bottleneck distance and $N_k$ is the maximum number of off-diagonal points in the two diagrams. Due to space limitations, the proof is given in Appendix~\ref{proof:proposition}.
\end{proposition}

\section{Freezing Strategies Based on Topological Drift}
\label{sec:sec5}

We use topological drift to guide efficient adaptation and to study how matrix topology reorganizes during fine-tuning. Our framework is visualized in Figure~\ref{fig:system}. We consider two settings: prior-based freezing and dynamic stopping. The first uses a prior fine-tuning task to construct a fine-tuning profile, which is then reused for fast selective fine-tuning on new tasks and datasets. The second uses topological drift observed during the ongoing fine-tuning run to track structural change across epochs and terminate fine-tuning as early as possible. An algorithm for the freezing procedure is provided in Appendix ~\ref{app:freezing-pseudocode}. 

\subsection*{Scenario A: Fine-tuning with Freezing Priors}

We posit a subtle but important insight: reusable freezing profiles do not require identical fine-tuning trajectories across tasks; it suffices that the induced gradient fields remain close to the pretrained model, limiting drift (\ie model change) that could otherwise cause catastrophic forgetting.

Let $\theta^{(t)}$ denote the model obtained by fine-tuning the pretrained network $f_\theta$ on a single prior task $\mathcal T$ for $t =\{1,\ldots,t_{max}\}$ epochs. For each matrix  $W_i \in \mathcal I$, we first compute its prior-task topological distance 
$D_i := \W{ass}_p(\mathrm{PD}_k(X^{(0)}_i), \mathrm{PD}_k(X_i^{(t_{max})})),$
where $X_i^{(t_{max})}$ is the point cloud of the post-fine-tuning version of $W_i$. Next, we also measure the  magnitude change of $W_i$ by
$M_i := \frac{1}{|W_i|}\sum_{a,b} \left|\frac{W_{i,ab}^{(t_{max})}-W_{i,ab}^{(0)}}{|W_{i,ab}^{(0)}|}\right|$ element-wise. We use this magnitude-based change as a baseline to discuss the usefulness of a TDA-guided strategy. Using the prior-task model, we build distance-based matrix rankings that are later reused on a new task $\mathcal T^\ast$.

We now rank matrices from the prior task and reuse the resulting freezing set on the target task. For a budget $b$, let $\mathcal F_{\mathrm{prior}}^{\mathrm{TDA\text{-}H}}(b)$ be the $b$ matrices with the largest values of $D_i$, and let $\mathcal F_{\mathrm{prior}}^{\mathrm{TDA\text{-}L}}(b)$ be the $b$ matrices with the smallest values. The magnitude baseline uses the same ranking procedure with $M_i$. The selected matrices are frozen from the start of fine-tuning on $\mathcal T^\ast$. The theorem below gives a sufficient condition under which the prior-task ranking is preserved on $\mathcal T^\ast$, making the freezing profile reusable.

\begin{theorem}[Stability of reusable freezing profiles]
\label{thm:profile-stability}
Fix a pretrained transformer $f_{\theta^{(0)}}$ and a monitored matrix set $\mathcal{I}$. Let $\mathcal{T}$ be a prior task and let $\mathcal{T}^{\ast}$ be a target task. For each task, let $\theta_{\mathcal{T}}^{(t)}=\{W_{i,\mathcal{T}}^{(t)}\}_{i=1}^m$ and $\theta_{\mathcal{T}^{\ast}}^{(t)}=\{W_{i,\mathcal{T}^{\ast}}^{(t)}\}_{i=1}^m$ denote the fine-tuned parameters after $t$ optimization steps, both initialized at $\theta^{(0)}$. Let $L_{\mathcal{T}}(\theta)$ and $L_{\mathcal{T}^{\ast}}(\theta)$ be the corresponding population fine-tuning objectives, and write $G_{\mathcal{T}}(\theta)=\nabla_\theta L_{\mathcal{T}}(\theta)$ and $G_{\mathcal{T}^{\ast}}(\theta)=\nabla_\theta L_{\mathcal{T}^{\ast}}(\theta)$. Assume both trajectories remain in a ball $B(\theta^{(0)},R)$, both gradient fields are $L$-Lipschitz on this ball, and $\sup_{\theta\in B(\theta^{(0)},R)}\|G_{\mathcal{T}}(\theta)-G_{\mathcal{T}^{\ast}}(\theta)\|_2\leq\epsilon$. If both tasks are fine-tuned with the same gradient descent update and step size $\alpha$, then for every $t\leq t_{\max}$ we have $\|\theta_{\mathcal{T}}^{(t)}-\theta_{\mathcal{T}^{\ast}}^{(t)}\|_2\leq R_t$, where $R_t:=\alpha\epsilon\sum_{r=0}^{t-1}(1+\alpha L)^r$.

Assume that the row-cloud metric $\delta_i$ is fixed across tasks and stable under row perturbations. For each $i\in\mathcal{I}$ and fixed homological dimension $k$, define $D_{i,\mathcal{T}}^{(t)}:=\W{ass}_p(\mathrm{PD}_k(X_i^{(0)}),\mathrm{PD}_k(X_{i,\mathcal{T}}^{(t)}))$ and $D_{i,\mathcal{T}^{\ast}}^{(t)}:=\W{ass}_p(\mathrm{PD}_k(X_i^{(0)}),\mathrm{PD}_k(X_{i,\mathcal{T}^{\ast}}^{(t)}))$. Then $|D_{i,\mathcal{T}}^{(t)}-D_{i,\mathcal{T}^{\ast}}^{(t)}|\leq 2N_{i,k}^{1/p}R_t$, where $N_{i,k}$ is the maximum number of off-diagonal points in the two PDs for matrix $i$ in dimension $k$.

Take $k=0$ and define the prior and target ranking scores by $S_i^{\mathrm{prior}}:=D_{i,\mathcal{T}}^{(t_{\max})}$ and $S_i^{\ast}:=D_{i,\mathcal{T}^{\ast}}^{(t_{\max})}$. Let $S_{(1)}^{\mathrm{prior}}\geq \cdots \geq S_{(|\mathcal{I}|)}^{\mathrm{prior}}$ be the sorted prior scores, and let $g_H(b):=S_{(b)}^{\mathrm{prior}}-S_{(b+1)}^{\mathrm{prior}}$ be the high-drift margin for freezing budget $b$. Define $\mathcal F_{\ast}^{\mathrm{TDA\text{-}H}}(b)$ and $\mathcal F_{\ast}^{\mathrm{TDA\text{-}L}}(b)$ analogously from the target scores $S_i^{\ast}$.
 Let $\delta_{\max}:=2\max_{i\in\mathcal{I}}N_{i,0}^{1/p}R_{t_{\max}}$. If $g_H(b)>2\delta_{\max}$, then $F_{\mathrm{prior}}^{\mathrm{TDA-H}}(b)=F_{\ast}^{\mathrm{TDA-H}}(b)$.  The same statement holds for $F_{\mathrm{prior}}^{\mathrm{TDA-L}}(b)$ after sorting the scores in increasing order and using the corresponding low-drift margin. Due to space limitations, the proof is given in Appendix~\ref{proof:theorem}.
\end{theorem}

\subsection*{Scenario B: Fine-tuning with Dynamic Stopping}

We now consider the second setting, where no prior fine-tuning profile is available. Let $\{\theta^{(t)}\}_{t=0}^{t_{\max}}$ denote the checkpoint sequence of a single full fine-tuning run. In this setting, our goal is to stop fine-tuning once the monitored matrices have stabilized topologically. Prior work offers evidence that \textquote{a few epochs of fine-tuning are typically sufficient to achieve strong performance}~\cite{sun2026theoretical}, with later epochs often providing limited additional gains.

Using the topological distance $D_i^{(t)}$ from Definition~\ref{def:distance}, we compute the normalized epoch-to-epoch drift $\Delta_i^{(t)}$ from Definition~\ref{def:topological-drift}. For a monitored set of matrices $\mathcal W \subseteq \mathcal I$, we average this drift into a model-level stabilization score as follows.

\begin{definition}[Topology-based stopping score]
\label{def:stopping}
Let $\mathcal W \subseteq \mathcal I$ be the monitored set of matrices. For $t \geq 2$, we define
$r^{(t)} := \frac{1}{|\mathcal W|}\sum_{i \in \mathcal W} |\Delta_i^{(t)}|$.
This score measures the average normalized topological drift of the monitored matrices between consecutive checkpoints.
\end{definition}

Given a threshold $\eta \in (0,1)$, we choose the stopping epoch $t^\ast := \min\{t \geq 2 : r^{(t)} \leq \eta\}$. If no such epoch is found within the training budget, we stop at the final epoch. The rule stops only when the average structural drift of the monitored matrices has become small; if substantial drift continues at a later epoch, then $r^{(t)}$ remains large and training continues.

\xhdr{Computational complexity} \method is computationally efficient. For each monitored matrix $W_i \in \mathbb R^{d_i \times d^\prime_i}$, computing pairwise row distances costs $O(d_i^2 d^\prime_i)$. Constructing the Vietoris--Rips filtration and computing persistence has worst-case exponential complexity in $d_i$, since the number of simplices can grow exponentially. In practice, we compute only low-dimensional persistence, \ie $H_0$ and limit the analysis to selected attention projection matrices. Across $t$ checkpoints and monitored matrices $\mathcal W$, the post-hoc topological analysis costs
$O\!\left(t \sum_{i \in \mathcal W} (d_i^2 d^\prime_i + \mathrm{PH}(d_i))\right)$, where $\mathrm{PH}(d_i)$ denotes the persistence computation cost for the row cloud of matrix $i$.

\vspace{-5px}
\section{Experiments}
\label{sec:results}
\vspace{-10px}

\xhdr{Hardware and implementation} All experiments were conducted on NVIDIA H200 GPUs and AMD EPYC 9555 64-Core Processor CPUs with our PyTorch-based code. We compute persistence diagrams with Ripser~\cite{tralie2018ripser} and evaluate diagram distances using Gudhi’s Wasserstein implementation~\cite{maria2014gudhi}. Our implementation is available as open-source code on GitHub (https://anonymous.4open.science/r/TopoTuner/)



\xhdr{Datasets} We evaluate \method on a diverse set of downstream tasks spanning question answering, sentiment analysis, information retrieval, summarization, instruction following, and code generation. Table~\ref{tab:datasets_metrics} summarizes the datasets, splits, and evaluation metrics. We use \texttt{QA} and \texttt{SA} datasets for fine-tuning and evaluation, and all twelve datasets for catastrophic forgetting after fine-tuning. For large datasets, training is capped at 20,000 examples.

\begin{wraptable}{r}{0.56\textwidth}
\vspace{-1.2em}
\centering
\caption{Datasets and evaluation metrics used in our experiments. \texttt{QA}: question answering, \texttt{SA}: sentiment analysis, IR: information retrieval, \texttt{SU}: summarization, \texttt{IF}: instruction following, \texttt{CG}: code generation.}
\resizebox{0.56\textwidth}{!}{%
\setlength{\tabcolsep}{3pt}
\begin{tabular}{lllll}
\toprule
Task & Dataset & Train size & Eval size & Evaluation metric \\
\midrule
\texttt{QA} & \texttt{GSM8K}~\cite{cobbe2021gsm8k} & 7,473 & 1,319 & Accuracy \\
\texttt{QA} & \texttt{MMLU}~\cite{hendryckstest2021, hendrycks2021ethics} & 99,842 & 14,042 & Accuracy \\
\texttt{SA} & \texttt{IMDB}~\cite{maas-EtAl:2011:ACL-HLT2011}  & 25,000 & 25,000 & Accuracy \\
\texttt{SA} & \texttt{SST-2}~\cite{socher-etal-2013-recursive}  & 67,349 & 872 & Accuracy \\
\midrule
\multicolumn{5}{l}{\textit{Forgetting}} \\
\texttt{IR} & \texttt{HotpotQA}~\cite{yang2018hotpotqa}  & 90,447 & 7,405 & EM and F1 \\
\texttt{IR} & \texttt{SQuAD v1.1}~\cite{rajpurkar-etal-2016-squad}  & 87,599 & 10,570 & EM and F1 \\
\texttt{SU} & \texttt{XSum}~\cite{Narayan2018DontGM}  & 204,045 & 11,334 & ROUGE-1, ROUGE-2, ROUGE-L \\
\texttt{SU} & \texttt{CNN/DailyMail}~\cite{see-etal-2017-get}  & 287,113 & 11,490 & ROUGE-1, ROUGE-2, ROUGE-L \\
\texttt{IF} & \texttt{Databricks-dolly-15k}~\cite{DatabricksBlog2023DollyV2}  & 15,011 & 15,011 & ROUGE-1/2/L, BLEU, METEOR \\
\texttt{IF} & \texttt{Alpaca}~\cite{alpaca}  & 52,002 & 52,002 & ROUGE-1/2/L, BLEU, METEOR \\
\texttt{CG} & \texttt{HumanEval}~\cite{chen2021evaluating}  & 164 & 164 & Pass@1 \\
\texttt{CG} & \texttt{MBPP}~\cite{austin2021program}  & 257 & 257 & Pass@1 \\
\bottomrule
\end{tabular}%
}
\label{tab:datasets_metrics}
\vspace{-1.0em}
\end{wraptable}

\xhdr{Models and training procedure} We evaluate three open-source pretrained language models: \texttt{LLaMA-3.1-8B}~\cite{grattafiori2024llama} (LLaMA), \texttt{Mistral-7B-v0.3}~\cite{jiang2023mistral} (Mistral), and \texttt{Qwen3-8B-Base}~\cite{qwen3technicalreport} (Qwen). For each model, we run both full fine-tuning and LoRA-based fine-tuning on each downstream task. We also tune the prominent models Spectrum~\cite{spectrum2024} and  DropBP~\cite{woo2024dropbp}, as explained in Appendix~\ref{app:competitors}. We use cross-entropy loss and a fixed epoch budget of six for all runs to ensure consistency across models and tasks. Persistent homology uses the cosine distance and $H_0$. The experimental setup and hyperparameters are provided in the Appendix~\ref{app:hyperparameters}.  

\xhdr{Evaluation metrics}  Our goal is to achieve the highest accuracy with the fewest parameter changes and the shortest fine-tuning time. We evaluate each method along three axes: downstream task performance, trainable-parameter efficiency, and training time. Downstream performance is measured on held-out evaluation sets using the task-specific metrics listed in Table~\ref{tab:datasets_metrics}. Efficiency is measured by i) the percentage of trainable model weights (Train. \%) and ii) the percentage of updated model weights (Upd. \%). Computational cost is measured by wall-clock training time in seconds on fixed hardware. Accuracy (Acc.) is computed on the full held-out evaluation split, while Std is estimated from three fixed, non-overlapping random subsets of the same evaluation split.

\subsection{Fine-tuning with Freezing Priors}
We use the question-answering dataset \texttt{QA:GSM8K} as the prior because it provides a standard reasoning task with clean, exact-answer evaluation. However, as Figure~\ref{fig:distances_llama_V} shows, for a fixed LLM, we find that the topological change patterns observed on \texttt{QA:GSM8K} remain consistent across the other datasets in our benchmark (topological drifts of all models on datasets in Appendix~\ref{app:topological-drift}). 

\begin{wrapfigure}{r}{0.48\textwidth}
 \vspace{-15px}
    \centering
    \includegraphics[width=0.46\textwidth]{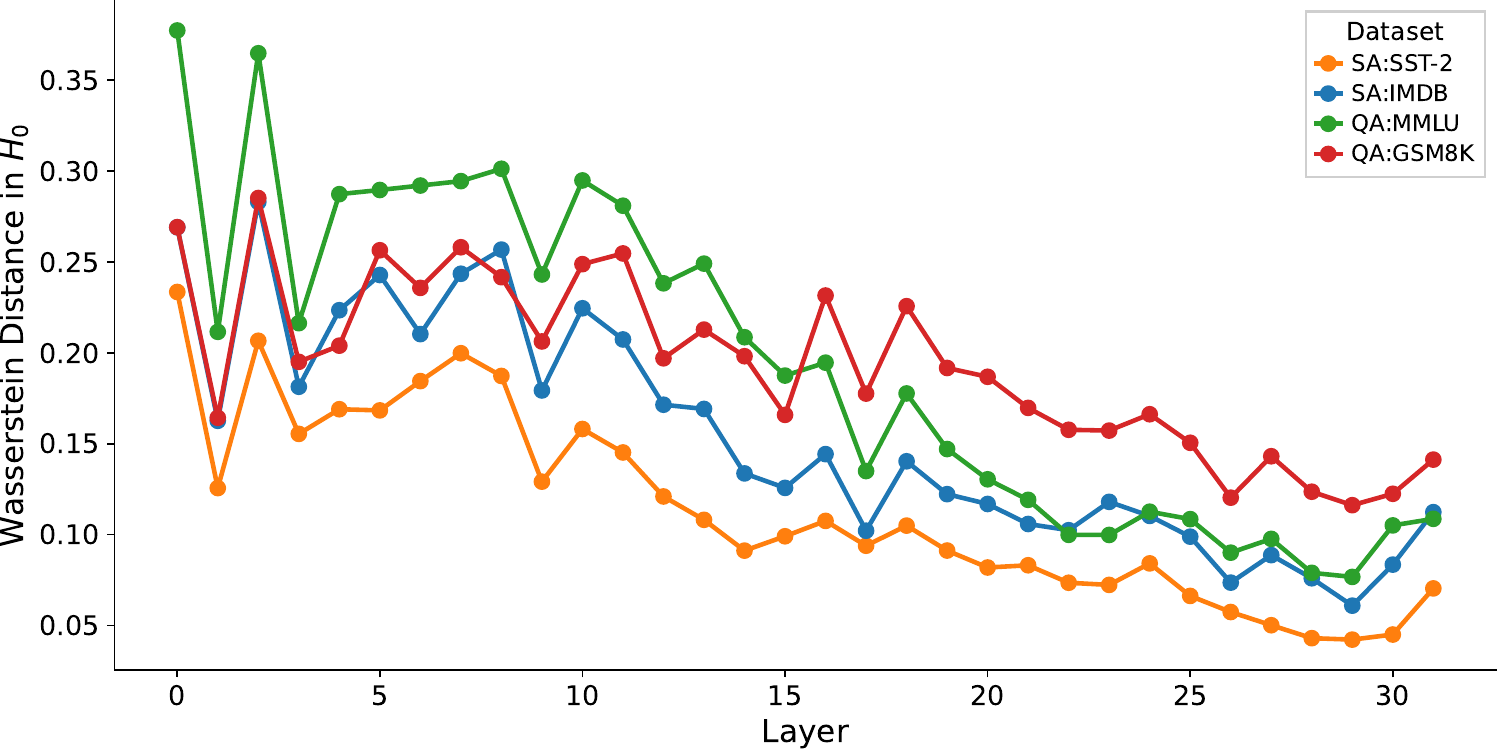}
    \caption{Epoch-6 Wasserstein $H_0$ distances for the $V$ projection of LLaMA across datasets. Full Wasserstein distance profiles across models are reported in Appendix~\ref{app:wasserstein-distances}.}
    \vspace{-15px}
\label{fig:distances_llama_V}
\end{wrapfigure}






We compare full fine-tuning, LoRA~\cite{hu2022lora}, Spectrum~\cite{spectrum2024}, DropBP~\cite{woo2024dropbp}, the magnitude-based selector, and our \method and report our metrics.

We evaluate low-drift and high-drift freezing over budgets $b\in\{3,6,9,12,15\}$. The low-drift strategy tests the hypothesis that stable matrices are less necessary for adaptation. The high-drift strategy  tests the hypothesis that blocking the dominant adaptation path can constrain task-specific movement, reduce overfitting, and force the remaining trainable matrices to absorb the update. In our experiments, \textbf{the high-drift strategy consistently yields stronger performance}. Accordingly, we omit the low-drift results from the main text for clarity and report them in Appendix~\ref{app:transferred-accuracy-curves}.

We freeze MLP weights and selectively freeze only the attention projections $K,Q,V,O$. This keeps the trainable budget comparable to PEFT methods, since MLP weights form most of the model parameters (\eg 62\% for LLaMA). Our projection-level ranks in Appendix Table~\ref{tab:eltwise-relative-avg-ranks-normalized} show that the dominant adaptation route is model-dependent: LLaMA and Mistral concentrate drift in $V/O$, while Qwen concentrates drift in $K/Q$; therefore, the candidate set includes all four attention projections.

We run both \textit{perfect-knowledge} and \textit{transferred-knowledge} experiments. The perfect-knowledge setting selects frozen projections from the same \texttt{QA:GSM8K} full-fine-tuning profile used for evaluation; this setup is mainly for diagnostics. As shown in Appendix~\ref{app:perfect-knowledge}, TDA-High3 outperforms LoRA by $2.43$ points on LLaMA and $3.59$ points on Mistral while changing only $2.64\%$--$2.82\%$ of model parameters.

\begin{table*}[!htbp]
\centering
\small
\setlength{\tabcolsep}{2pt}
\renewcommand{\arraystretch}{1.08}

\caption{\textbf{Transferred knowledge.} Fine-tuning results across \texttt{SA:SST-2}, \texttt{SA:IMDB}, and \texttt{QA:MMLU}. \method selective tuning freezes projections learned from \texttt{QA:GSM8K}. \tiny{*: LoRA trains only 0.17--0.19\% extra parameters, but after merging can affect up to 100\% of the targeted projection matrices}.}
\resizebox{0.98\textwidth}{!}{
\begin{tabular}{l|l|cccc|cccc|cccc}
\toprule
Dataset & Method
& \multicolumn{4}{c|}{\texttt{Qwen3-8B-Base}}
& \multicolumn{4}{c|}{\texttt{LLaMA-3.1-8B}}
& \multicolumn{4}{c}{\texttt{Mistral-7B-v0.3}} \\
& 
& Train.\% & Upd. \% & Acc.$\pm$Std. ($\uparrow$) & Time (min)
& Train.\% & Upd. \% &  Acc.$\pm$Std. ($\uparrow$) & Time (min)
& Train.\% & Upd. \% & Acc.$\pm$Std. ($\uparrow$) & Time (min) \\
\midrule

\multirow{5}{*}{\centering \texttt{SST-2}}
& \graycell{Full Fine-tuning}
& \graycell{100} & \graycell{13.01} & \graycell{$95.76 \pm 1.00$} & \graycell{32.7}
& \graycell{100} & \graycell{21.88} & \graycell{$96.10 \pm 0.53$} & \graycell{30.3}
& \graycell{100} & \graycell{25.00} & \graycell{$96.10 \pm 0.87$} & \graycell{32.9} \\

& LoRA
& 100* & 11.02 & $95.75 \pm 1.78$ & 24.8
& 100* & 16.62 & $95.64 \pm 2.22$ & 21.5
& 100* & 18.41 & $96.10 \pm 1.44$ & 24.7 \\

& DropBP
& 100* & 18.19 & $85.66 \pm 1.12$ & 34.0
& 100* & 16.61 & $95.53 \pm 0.34$ & 26.4
& 100* & 18.42 & $96.10 \pm 0.21$ & 34.2 \\

& Spectrum
& 17.55 & 0.47 & $92.78 \pm 0.68$ & 30.3
& 15.80 & 1.24 & $95.76 \pm 0.53$ & 26.9
& 6.71 & 1.23 & $95.53 \pm 0.92$ & 30.9 \\

& \method
& 8.45 & 1.15 & $\mathbf{95.99 \pm 0.41}$ & 23.9
& 7.57 & 2.02 & $\mathbf{95.98 \pm 1.44}$ & 21.2
& 8.39 & 2.44 & $\mathbf{96.21 \pm 0.70}$ & 24.0 \\
\midrule

\multirow{5}{*}{\centering \texttt{IMDB}}
& \graycell{Full Fine-tuning}
& \graycell{100} & \graycell{8.34} & \graycell{$82.83 \pm 0.31$} & \graycell{195.1}
& \graycell{100} & \graycell{24.95} & \graycell{$90.66 \pm 0.21$} & \graycell{172.4}
& \graycell{100} & \graycell{28.67} & \graycell{$83.10 \pm 0.18$} & \graycell{164.8} \\

& LoRA
& 100* & 18.38 & $82.89 \pm 0.35$ & 177.2
& 100* & 16.67 & $83.36 \pm 0.30$ & 153.2
& 100* & 18.49 & $82.72 \pm 0.16$ & 147.9 \\

& DropBP
& 100* & 18.29 & $70.79 \pm 0.30$ & 214.6
& 100* & 16.63 & $\mathbf{95.30 \pm 0.20}$ & 158.4
& 100* & 18.40 & $93.64 \pm 0.22$ & 164.2 \\

& Spectrum
& 17.55 & 0.82 & $81.62 \pm 0.28$ & 172.0
& 15.80 & 2.56 & $94.98 \pm 0.22$ & 150.9
& 6.71 & 1.82 & $\mathbf{94.82 \pm 0.20}$ & 143.6 \\

& \method
& 8.45 & 0.78 & $\mathbf{82.97 \pm 0.34}$ & 163.6
& 7.57 & 2.36 & $83.55 \pm 0.26$ & 143.0
& 8.39 & 2.92 & $79.76 \pm 0.30$ & 137.7 \\
\midrule

\multirow{5}{*}{\centering \texttt{MMLU}}
& \graycell{Full Fine-tuning}
& \graycell{100} & \graycell{9.39} & \graycell{$73.22 \pm 2.15$} & \graycell{157.6}
& \graycell{100} & \graycell{31.29} & \graycell{$60.03 \pm 0.05$} & \graycell{139.6}
& \graycell{100} & \graycell{35.38} & \graycell{$61.85 \pm 2.62$} & \graycell{132.2} \\

& LoRA
& 100* & 18.41 & $\mathbf{74.13 \pm 2.96}$ & 141.9
& 100* & 16.67 & $61.86 \pm 1.67$ & 122.9
& 100* & 18.49 & $56.82 \pm 2.85$ & 118.4 \\

& DropBP
& 100* & 18.19 & $37.54 \pm 0.36$ & 165.1
& 100* & 16.61 & $54.19 \pm 0.91$ & 122.3
& 100* & 18.42 & $50.01 \pm 0.51$ & 126.1 \\

& Spectrum
& 17.55 & 0.66 & $66.88 \pm 0.62$ & 143.3
& 15.80 & 2.81 & $54.79 \pm 0.94$ & 126.2
& 6.71 & 2.00 & $52.31 \pm 0.63$ & 119.0 \\

& \method
& 8.45 & 0.68 & $72.83 \pm 1.97$ & 131.7
& 7.57 & 2.49 & $\mathbf{63.75 \pm 1.79}$ & 115.4
& 8.39 & 2.98 & $\mathbf{60.68 \pm 3.06}$ & 110.3 \\

\bottomrule
\end{tabular}}
\label{tab:mainPerformance}
\end{table*}

In the second and main setting, Table~\ref{tab:mainPerformance} evaluates whether a freezing profile learned from \texttt{QA:GSM8K} transfers to new datasets (competitors are not transferred but are fine-tuned on the same dataset). In this setting, we report TDA-High3 as \method because it is the smallest high-drift budget, already improves over LoRA in $7$ out of $9$ settings, reduces training time, shows a smaller train-test generalization gap, indicating reduced overfitting (see Appendix~\ref{app:overfitting}) and keeps parameter change below $3\%$; larger budgets are reported in Appendix Tables~\ref{tab:llama38GSM8KFull}, \ref{tab:qwenGSM8KFull} and \ref{tab:mistralGSM8KFull}.  The gains are clearest for LLaMA, where \method outperforms LoRA across all three datasets, with the largest improvement on \texttt{QA:MMLU}. For Mistral, \method improves over LoRA on \texttt{SA:SST-2} and \texttt{QA:MMLU}. For Qwen, \method is strongest on \texttt{SA:SST-2} and \texttt{SA:IMDB}. Across all settings, \method is also faster than competitors and trains on (Train. \%) substantially fewer parameters, indicating that the transferred topological profile can reduce training time and parameter updates while preserving or improving accuracy in most cases.  

The forgetting results in Appendix Tables~\ref{tab:forgetting-llama-11datasets}, \ref{tab:forgetting-qwen-11datasets}, and \ref{tab:forgetting-mistral-11datasets} show the same pattern: \method often preserves stronger out-of-domain performance on \texttt{QA}, \texttt{SA}, and \texttt{IR} evaluations, with more mixed effects on generation-style tasks.




%


\subsection{Fine-tuning with Dynamic Stopping}

We next test whether topological drift can decide when fine-tuning should stop. For each fine-tuning epoch, we compute the topological drift of the monitored projection matrices from initialization and aggregate the displacement into the stopping score from Definition~\ref{def:stopping}. Training stops once the normalized epoch-to-epoch change falls below a threshold based on the model's stabilization.

\begin{wrapfigure}{r}{0.48\textwidth}
    \centering
    \includegraphics[width=0.46\textwidth]{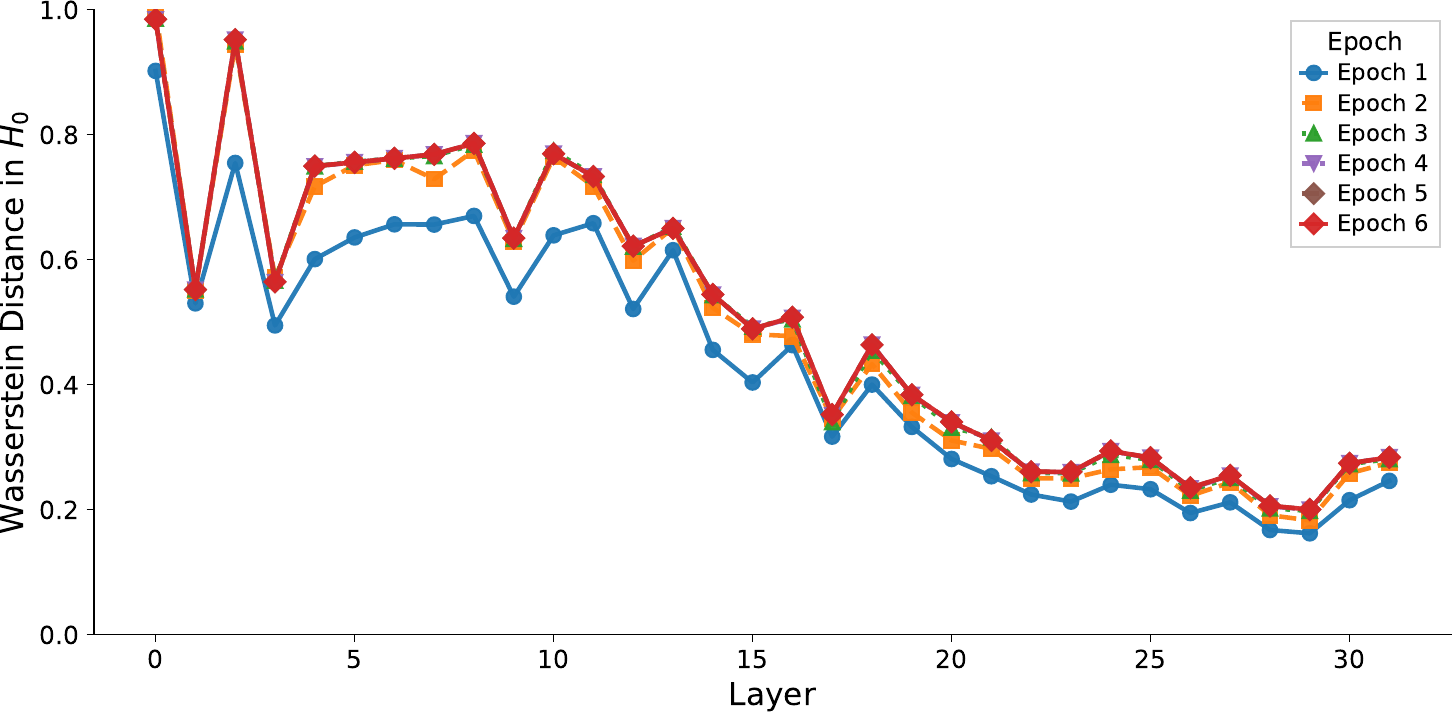}
    \caption{Evolution of Wasserstein $H_0$ distances $D^{(t)}_i$ over layers for the $V$ projection of LLaMA fully fine-tuned on \texttt{QA:MMLU}. Each line is a distance to the pretrained model.}
    \vspace{-10px}
    \label{fig:dynamic}
\end{wrapfigure}

Figure~\ref{fig:dynamic} shows the Wasserstein distance $D^{(t)}_i$ across epochs for LLaMA fully fine-tuned on \texttt{QA:MMLU}. The first epoch produces the largest change in early layers, while subsequent epochs contribute little additional change there.   Figure~\ref{fig:llama31-main-driftbars} shows the corresponding epoch-to-epoch drift bars for LLaMA on \texttt{QA:MMLU}, \texttt{SA:SST-2}, and \texttt{SA:IMDB} (see other models in the Appendix~\ref{app:topological-drift}). Drift values in Figure~\ref{fig:llama31-main-driftbars} show that fine-tuning on \texttt{SA:SST-2}, and \texttt{SA:IMDB} can stop after the second epoch. The accuracy curves in Appendix Figure~\ref{fig:transferred-accuracy-curves} yield the same conclusion regarding performance. Accuracy rises early and then moves within a narrow range, while the topological drift has already decayed. Training beyond this point mostly involves small residual updates.

\begin{figure}[!htbp]
\centering
\setlength{\tabcolsep}{2pt}

\begin{tabular}{ccc}
\includegraphics[width=0.32\linewidth]{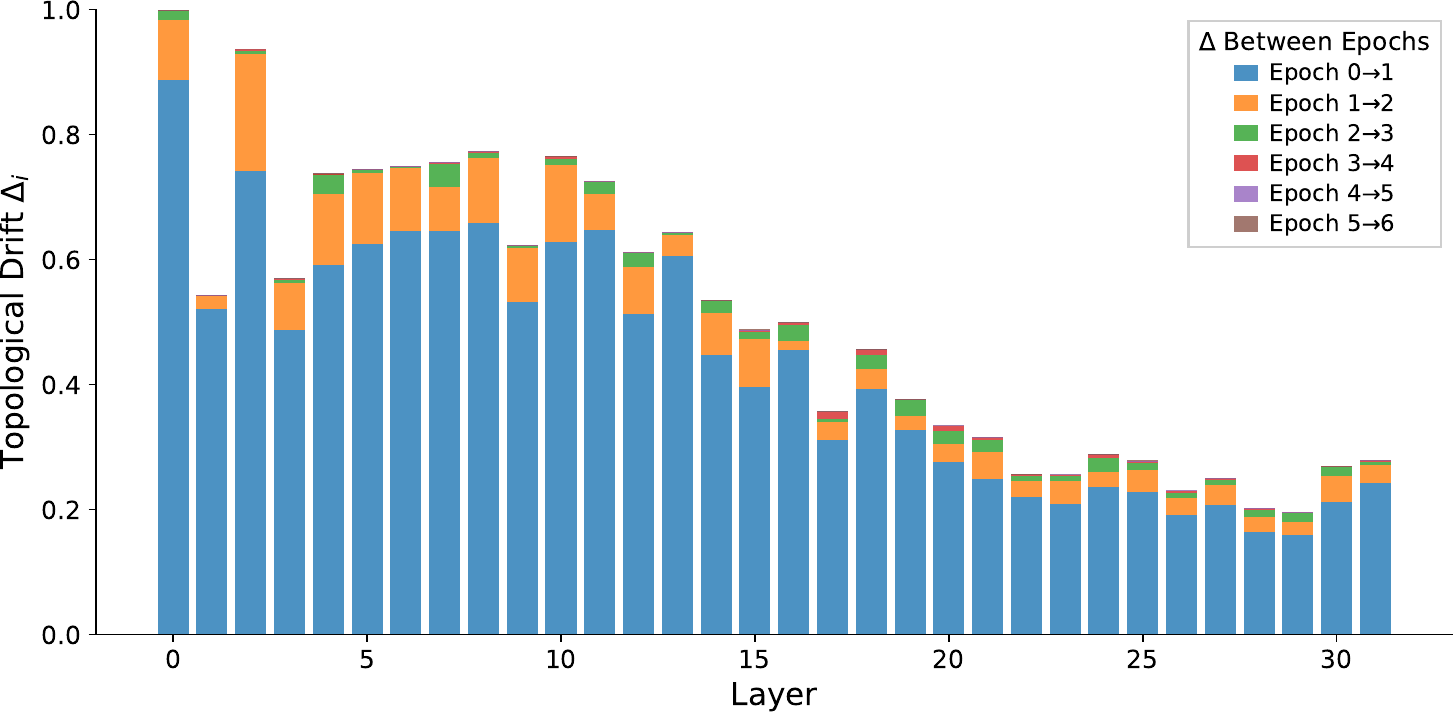} &
\includegraphics[width=0.32\linewidth]{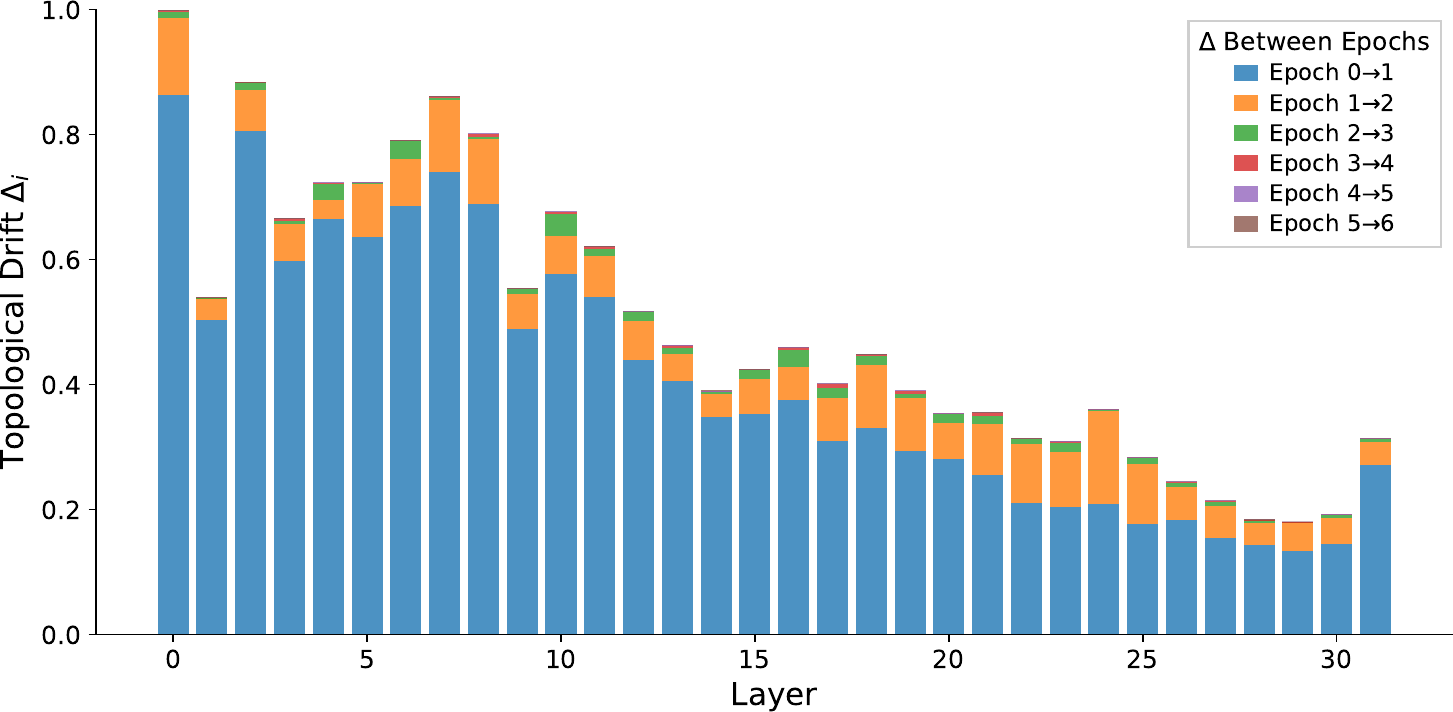} &
\includegraphics[width=0.32\linewidth]{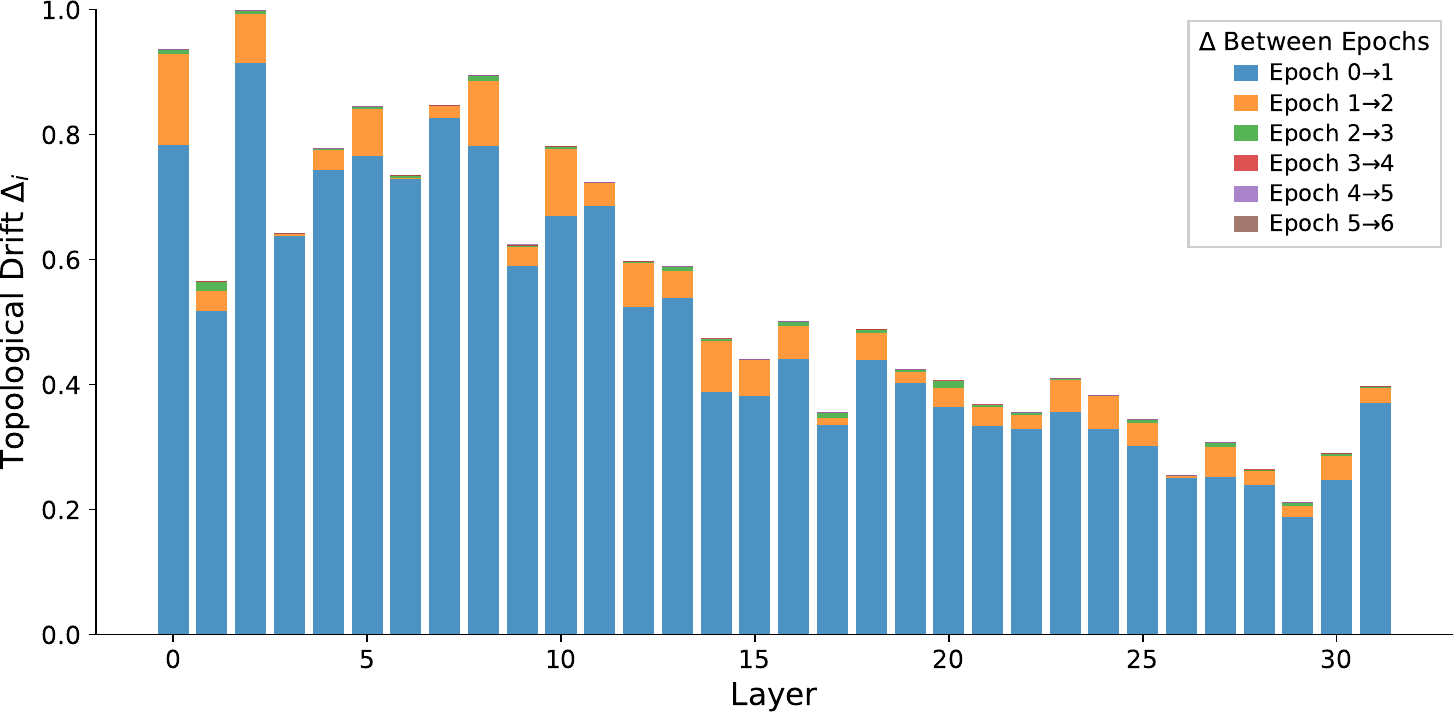} \\

\small \texttt{QA:MMLU} & \small \texttt{SA:SST-2} & \small \texttt{SA:IMDB} \\[0.4em]
\end{tabular}
\caption{Epoch-to-epoch topological drift bars for the $V$ projections of LLaMA under $H_0$ during full fine-tuning. Most layers stabilize by epoch $2$, with subsequent epochs showing only small residual changes. See all datasets and models in Appendix~\ref{app:topological-drift}.}
\label{fig:llama31-main-driftbars}
\end{figure}

Figure~\ref{fig:eta-stop-sst2-only} shows how the selected stopping epoch changes with $\eta$ on \texttt{SA:SST-2}. The rule first stops at epoch 3 for LLaMA and Qwen at $\eta=0.05$, and reaches epoch-2 stopping for Qwen at $\eta=0.15$. Appendix~\ref{app:dynamic-stopping} provides the stopping-epoch curves for \texttt{SA:IMDB} and \texttt{QA:MMLU} across the three model families.

\begin{wrapfigure}{r}{0.4\textwidth}
    \centering
    \vspace{-10px}
    \includegraphics[width=0.36\textwidth]{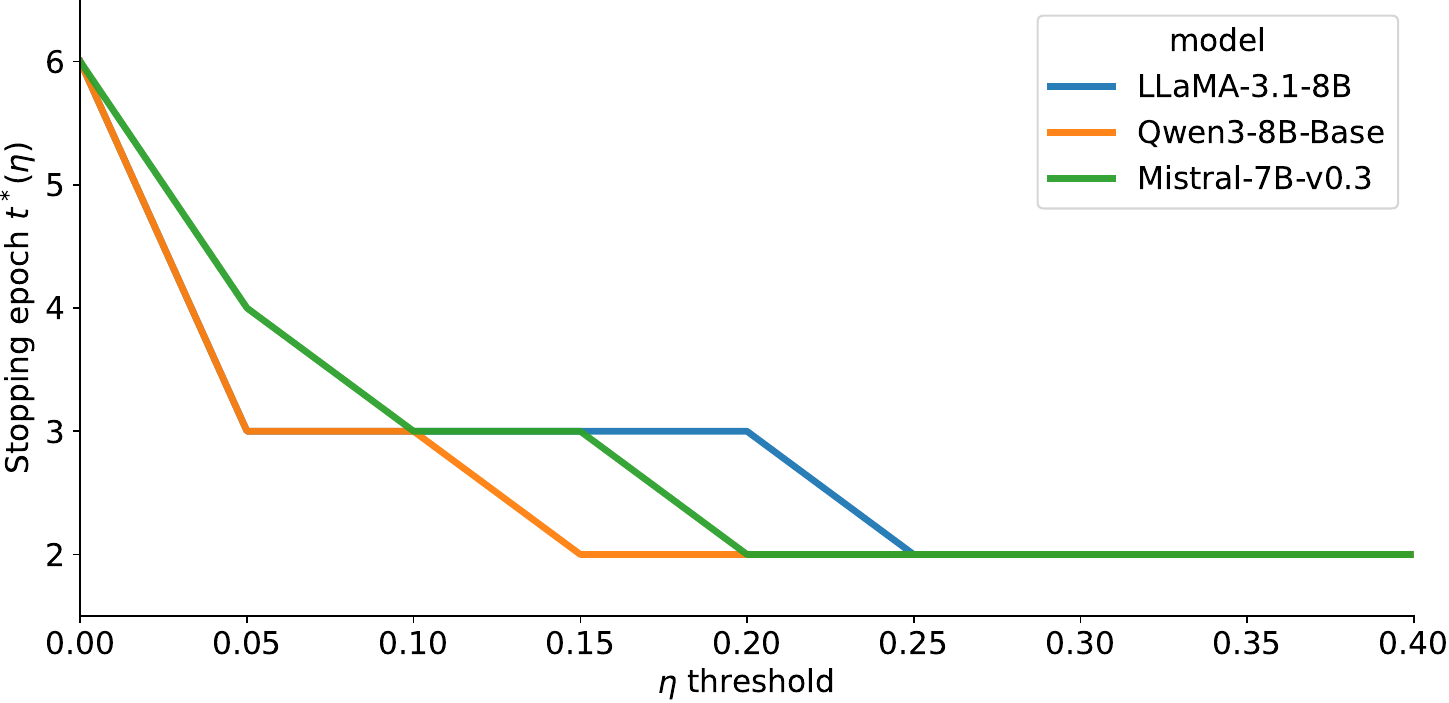}
    \caption{Topology-based stopping epochs under full fine-tuning. The plots show selected stopping epochs across $\eta$ thresholds for \texttt{SA:SST-2} across the three model families.}
    \vspace{-15px}
\label{fig:eta-stop-sst2-only}
\end{wrapfigure}

\xhdr{Time efficiency}
On \texttt{QA:GSM8K}, persistence computation takes approximately 2.1 min for LLaMA, 2.4 min for Qwen, and 2.2 min for Mistral per epoch; the subsequent Wasserstein computation is shorter, taking approximately 1.8 min, 2.1 min, and 2.0 min. These costs are CPU-side, highly parallelizable, and can be amortized because the resulting profiles are reused for freezing and transfer experiments. This cost is paid offline after checkpoints are saved and does not change the asymptotic cost of gradient-based fine-tuning; the online savings come from freezing selected projection matrices and stopping training early. \begin{wrapfigure}{r}{0.31\textwidth}
\vspace{-2.5em}
\centering

\definecolor{profileblue}{RGB}{238,245,255}
\definecolor{profileborder}{RGB}{65,105,160}

\tcbset{
  modelprofile/.style={
    enhanced,
    colback=profileblue,
    colframe=profileborder,
    boxrule=0.7pt,
    arc=2mm,
    left=1.2mm,
    right=1.2mm,
    top=1mm,
    bottom=1mm,
    fonttitle=\ttfamily\scriptsize,
    fontupper=\ttfamily\tiny,
    coltitle=profileborder,
    attach boxed title to top center={yshift=-1.2mm},
    boxed title style={
      colback=white,
      colframe=profileborder,
      boxrule=0.5pt,
      arc=1mm,
      left=0.8mm,
      right=0.8mm,
      top=0.4mm,
      bottom=0.4mm
    }
  }
}

\newcommand{\cardrule}{%
  \\ \multicolumn{2}{@{}c@{}}{\smash{{\rule{0.75\linewidth}{0.5pt}}}} \\
}

\begin{tcolorbox}[modelprofile,title={TDA-High3-LLaMA}]
\renewcommand{\arraystretch}{1}
\begin{tabularx}{\linewidth}{@{}p{1cm}X@{}}

\textbf{Model} 
& LLaMA-3.1-8B
\cardrule

\textbf{Prior task} & {\href{https://huggingface.co/datasets/openai/gsm8k}{QA:GSM8K}}
\cardrule

\textbf{Freezing strategy} 
&
\begin{tabular}[t]{@{}l@{}}
Cosine-Wasserstein \\
$H_0$ High-$3$
\end{tabular}
\cardrule

\textbf{Frozen IDs} 
&
\begin{tabular}[t]{@{}l@{}}
$V$:$\{V_{8}, V_{2}, V_{0}\}$ \\
$O$:$\{O_{11}, O_{10}, O_{13}\}$
\end{tabular}
\cardrule

\multicolumn{2}{c}{\textbf{Transferred to}}\\
\multicolumn{2}{c}{
\tiny
\begin{tabular}[t]{rl}
\multicolumn{2}{c}{\href{https://huggingface.co/datasets/stanfordnlp/sst2}{SA:SST-2}} \\
Acc\({\pm}\)Std & \(95.98{\pm}1.44\)\\
ETA ({\href{https://www.nvidia.com/en-us/data-center/h200/}{H200}}) & \(21.2\) min\\
Upd.\% & \(2.02\%\)\\

\multicolumn{2}{c}{\href{https://huggingface.co/datasets/stanfordnlp/imdb}{SA:IMDB}} \\
Acc\({\pm}\)Std & \(83.55{\pm}0.26\)\\
ETA ({\href{https://www.nvidia.com/en-us/data-center/h200/}{H200}}) & \(143.0\) min\\
Upd.\% & \(2.36\%\)\\

\multicolumn{2}{c}{\href{https://huggingface.co/datasets/cais/mmlu}{QA:MMLU}} \\
Acc\({\pm}\)Std & \(63.75{\pm}1.79\)\\
ETA ({\href{https://www.nvidia.com/en-us/data-center/h200/}{H200}}) & \(115.4\) min\\
Upd.\% & \(2.49\%\) \\ \\
Forgetting & 
    \begin{tabular}[t]{@{}l@{}}
    \cmark {\href{https://huggingface.co/datasets/cais/mmlu}{MMLU}}, \\ \cmark  {\href{https://huggingface.co/datasets/stanfordnlp/sst2}{SST-2}}, \cmark {\href{https://huggingface.co/datasets/stanfordnlp/imdb}{IMDB}}, \\
    \cmark {\href{https://huggingface.co/datasets/rajpurkar/squad}{SQuAD}}, 
    \cmark {\href{https://huggingface.co/datasets/hotpotqa/hotpot_qa}{HotpotQA}}, \\ 
    \cmark {\href{https://huggingface.co/datasets/EdinburghNLP/xsum}{XSum}}, \cmark {\href{https://huggingface.co/datasets/databricks/databricks-dolly-15k}{Dolly-15k}}, \\
    \xmark {\href{https://huggingface.co/datasets/abisee/cnn_dailymail}{CNN/DM}},
    \xmark
    {\href{https://huggingface.co/datasets/tatsu-lab/alpaca}{Alpaca}}, \\
    \xmark {\href{https://huggingface.co/datasets/openai/openai_humaneval}{HumanEval}},
    \xmark {\href{https://huggingface.co/datasets/Muennighoff/mbpp}{MBPP}}
    \end{tabular}
\end{tabular}
}
\cardrule

\textbf{Stopping $\eta$}
& {SA:SST-2} $\eta$=0.25, {SA:IMDB} $\eta$=0.20, {QA:MMLU} $\eta$=0.25

\end{tabularx}
\end{tcolorbox}

\vspace{-0.8em}
\caption{Reusable freezing profile for LLaMA-3.1-8B learned from \texttt{QA:GSM8K}.}
\label{fig:llama-freezing-profile}
\vspace{-1em}
\end{wrapfigure}  

As a result, Table~\ref{tab:mainPerformance} shows that \method is more time-efficient than the competitors in transferred fine-tuning, and the time gains are significant on large datasets such as \texttt{QA:MMLU}.

\subsection{Reusable Fine-Tuning}
\label{sec:reusable-fine-tuning}

Our results and descriptive analysis in the Appendix  Section~\ref{sec:weightChangeStatistics} suggest that fine-tuning changes LLMs in a structured but highly constrained manner. Entrywise weight movements are small, and most parameters remain close to their pretrained values. This makes purely magnitude-based descriptions incomplete: the important question is not only how much a matrix moves but also whether the topology changes coherently. Wasserstein distances on persistence diagrams provide such a signal.

A first result is that projection-level drift is architecture-dependent. As Appendix Tables~\ref{tab:eltwise-relative-avg-ranks-normalized}  and \ref{tab:eltwise-relative-topn-counts} show, LLaMA and Mistral usually place the largest changes in $V$ and $O$, while Qwen often shows the opposite pattern, with $K$ and $Q$ receiving stronger drift. This rules out a universal projection rule such as always tuning $V/O$ or always freezing $K/Q$.

This model dependence is paired with a second (useful and surprising) observation: for a fixed model, the drift profile is reusable across tasks. This does not imply that all tasks induce the same geometry (we analyze layer change similarity across models for tasks in Appendix~\ref{app:cross-task-geometry}). However, each pretrained model appears to have a stable set of adaptation channels. 

We propose a reusable profile card that summarizes which attention projections consistently absorb topological change during a prior fine-tuning run. The card records the source task, the topological ranking strategy, the selected frozen projection IDs, the stopping threshold, transfer performance on downstream tasks, parameter update rates, and forgetting behavior across held-out evaluations. Figure~\ref{fig:llama-freezing-profile} shows an example profile learned from \texttt{QA:GSM8K} for LLaMA-3.1-8B using a cosine distance-Wasserstein $H_0$ High-3 strategy. The same frozen projection set transfers to sentiment analysis, question answering, and summarization tasks without recomputing a new freezing search. This suggests that pretrained models contain reusable adaptation routes that persist across downstream fine-tuning runs. The profile card therefore acts as a compact transferable prior over how a model changes during adaptation.

\xhdr{Limitations} \method requires access to checkpoints and attention projection weights, so closed-weight API models cannot be studied without provider-side access to weights and freezing controls. Future work should test larger open-weight models, broader task families, and profile reuse across longer fine-tuning sequences.

\section{Conclusion}

We have proposed \method, a topology-guided approach for controlling LLM fine-tuning through matrix-level structural drift. \method measures how attention projection matrices reorganize relative to the pretrained model and uses this signal for selective freezing and early stopping. Across our experiments, \method outperforms LoRA in $7$ out of $9$ transferred-knowledge settings, reduces training time by $20.4\%$ on average, and improves over LoRA by $2.43\%$ on LLaMA and $3.59\%$ on Mistral. Compared with baselines, \method is faster and changes fewer projection matrices, only about 2.64\%--2.82\%, which gives a more constrained adaptation process with lower risk of unintended changes. More importantly, \method produces reusable freezing profiles that can instruct future users which projection matrices to train when fine-tuning the same model on new datasets.
\bibliography{topotuner}
\bibliographystyle{unsrtnat}

\newpage
\appendix
\onecolumn
 
\section*{Appendix}
\label{Appendix}

\section{Broader Impacts}
\label{app:broader-impacts}

This work improves the efficiency and control of fine-tuning for large language models by identifying which components need to change and when training can stop. A primary benefit is reduced computational cost. By training a small fraction of parameters and avoiding unnecessary epochs, the method lowers energy consumption and shortens development cycles. This makes fine-tuning more accessible to smaller research groups and organizations with limited compute resources.

The method also contributes to better understanding of model adaptation. By measuring structural change at the level of individual weight matrices, it provides insight into how different tasks affect pretrained models. This can support more reliable deployment by helping practitioners identify stable and unstable components during adaptation. In settings where consistency is important, such as finance or data-driven decision systems, controlled updates can reduce unintended deviations from pretrained behavior.

Another positive aspect is reusability. Freezing profiles learned on one task can be applied to new tasks, which reduces repeated exploration and simplifies the fine-tuning process. This can accelerate experimentation and enable more systematic development of downstream applications without requiring full retraining.

\section{Reusable Freezing Profiles}
\label{app:freezing-profiles}

The following cards summarize the reusable freezing profiles produced by \method. Each profile is learned from a prior fine-tuning task, ranks attention projection matrices by prior-task topological drift, and freezes the selected drift matrices during target-task fine-tuning.

\definecolor{profileblue}{RGB}{238,245,255}
\definecolor{profileborder}{RGB}{65,105,160}
\definecolor{profilegray}{RGB}{248,248,248}

\tcbset{
  freezingprofile/.style={
    enhanced,
    colback=profileblue,
    colframe=profileborder,
    boxrule=0.7pt,
    arc=2mm,
    left=1.5mm,
    right=1.5mm,
    top=1.2mm,
    bottom=1.2mm,
    fonttitle=\ttfamily\small,
    fontupper=\ttfamily\scriptsize, 
    coltitle=profileborder,
    attach boxed title to top center={yshift=-1.5mm},
    boxed title style={
      colback=white,
      colframe=profileborder,
      boxrule=0.5pt,
      arc=1mm,
      left=1mm,
      right=1mm,
      top=0.5mm,
      bottom=0.5mm
    }
  },
  profiledefinition/.style={
    freezingprofile,
    colback=profilegray,
    colframe=black!80
  },
  modelprofile/.style={
    freezingprofile,
    valign=top
  }
}

\newcommand{\cardrule}{%
  \\ \multicolumn{2}{@{}c@{}}{\smash{{\rule{0.8\linewidth}{0.6pt}}}} \\
}

\begin{figure*}[!htbp]
\centering

\begin{minipage}[t]{0.9\textwidth}
\begin{tcolorbox}[profiledefinition,title={Profile Definition}]
\renewcommand{\arraystretch}{1.1}
\begin{tabularx}{\linewidth}{@{}p{1.4cm}X@{}}

\textbf{Profile ID} 
& Identifier of the reusable freezing profile, combining the freezing strategy and model.
\cardrule

\textbf{Model} 
& Pretrained model $f_{\theta}$ for which the freezing profile is learned. 
\cardrule

\textbf{Prior task} 
& Prior task $\mathcal{T}$ and prior dataset $\mathcal{D}_{\mathcal{T}}$ used to compute the matrix-level ranking scores. 
\cardrule

\textbf{Freezing strategy} 
& Ranking score and selection rule used to choose frozen matrices from the candidate set, e.g., Cosine-Wasserstein for ranking score and High-$3$ for selection rule. 
\cardrule

\textbf{Frozen IDs} 
& Selected matrices or layers frozen from the start of target-task fine-tuning. 
\cardrule

\textbf{Transferred to} 
& Acc\({\pm}\)Std, ETA (with GPU Used), Upd.\%, and $\Delta_f$, denoting accuracy with standard deviation, total training time, updated-parameter percentage, and forgetting datasets. \cmark denotes our profile card is forgetting less than full fine-tuning and \xmark denotes the opposite.
\cardrule

\textbf{Stopping $\eta$} 
& Topology-based stopping threshold. Fine-tuning stops at the first epoch where the normalized epoch-to-epoch topological drift falls below $\eta$.

\end{tabularx}
\end{tcolorbox}
\end{minipage}

\vspace{1em}

\begin{tabular}{@{}p{0.31\textwidth} p{0.31\textwidth} p{0.31\textwidth}@{}}

\begin{tcolorbox}[modelprofile,title={TDA-High3-LLaMA}]
\renewcommand{\arraystretch}{1}
\begin{tabularx}{\linewidth}{@{}p{1cm}X@{}}

\textbf{Model} 
& LLaMA-3.1-8B
\cardrule

\textbf{Prior task} & {\href{https://huggingface.co/datasets/openai/gsm8k}{QA:GSM8K}}
\cardrule

\textbf{Freezing strategy} 
&
\begin{tabular}[t]{@{}l@{}}
Cosine-Wasserstein \\
$H_0$ High-$3$
\end{tabular}
\cardrule

\textbf{Frozen IDs} 
&
\begin{tabular}[t]{@{}l@{}}
$V$:$\{V_{8}, V_{2}, V_{0}\}$ \\
$O$:$\{O_{11}, O_{10}, O_{13}\}$
\end{tabular}
\cardrule

\multicolumn{2}{c}{\textbf{Transferred to}}\\
\multicolumn{2}{c}{
\tiny
\begin{tabular}[t]{rl}
\multicolumn{2}{c}{\href{https://huggingface.co/datasets/stanfordnlp/sst2}{SA:SST-2}} \\
Acc\({\pm}\)Std & \(95.98{\pm}1.44\)\\
ETA ({\href{https://www.nvidia.com/en-us/data-center/h200/}{H200}}) & \(21.2\) min\\
Upd.\% & \(2.02\%\)\\

\multicolumn{2}{c}{\href{https://huggingface.co/datasets/stanfordnlp/imdb}{SA:IMDB}} \\
Acc\({\pm}\)Std & \(83.55{\pm}0.26\)\\
ETA ({\href{https://www.nvidia.com/en-us/data-center/h200/}{H200}}) & \(143.0\) min\\
Upd.\% & \(2.36\%\)\\

\multicolumn{2}{c}{\href{https://huggingface.co/datasets/cais/mmlu}{QA:MMLU}} \\
Acc\({\pm}\)Std & \(63.75{\pm}1.79\)\\
ETA ({\href{https://www.nvidia.com/en-us/data-center/h200/}{H200}}) & \(115.4\) min\\
Upd.\% & \(2.49\%\) \\ \\
Forgetting & 
    \begin{tabular}[t]{@{}l@{}}
    \cmark {\href{https://huggingface.co/datasets/cais/mmlu}{MMLU}}, \\
    \cmark  {\href{https://huggingface.co/datasets/stanfordnlp/sst2}{SST-2}},
    \cmark {\href{https://huggingface.co/datasets/stanfordnlp/imdb}{IMDB}}, \\
    \cmark {\href{https://huggingface.co/datasets/rajpurkar/squad}{SQuAD}}, 
    \cmark {\href{https://huggingface.co/datasets/hotpotqa/hotpot_qa}{HotpotQA}}, \\ 
    \cmark {\href{https://huggingface.co/datasets/EdinburghNLP/xsum}{XSum}},
    \cmark {\href{https://huggingface.co/datasets/databricks/databricks-dolly-15k}{Dolly-15k}}, \\
    \xmark {\href{https://huggingface.co/datasets/abisee/cnn_dailymail}{CNN/DM}},
    \xmark {\href{https://huggingface.co/datasets/tatsu-lab/alpaca}{Alpaca}}, \\
    \xmark {\href{https://huggingface.co/datasets/openai/openai_humaneval}{HumanEval}},
    \xmark {\href{https://huggingface.co/datasets/Muennighoff/mbpp}{MBPP}}
    \end{tabular}
\end{tabular}
}
\cardrule

\textbf{Stopping $\eta$}
& {SA:SST-2} $\eta$=0.25, {SA:IMDB} $\eta$=0.20, {QA:MMLU} $\eta$=0.25

\end{tabularx}
\end{tcolorbox}

&

\begin{tcolorbox}[modelprofile,title={TDA-High3-Qwen}]
\renewcommand{\arraystretch}{1}
\begin{tabularx}{\linewidth}{@{}p{1cm}X@{}}

\textbf{Model} 
& Qwen3-8B-Base
\cardrule

\textbf{Prior task} & {\href{https://huggingface.co/datasets/openai/gsm8k}{QA:GSM8K}}
\cardrule

\textbf{Freezing strategy} 
&
\begin{tabular}[t]{@{}l@{}}
Cosine-Wasserstein \\
$H_0$ High-$3$
\end{tabular}
\cardrule

\textbf{Frozen IDs} 
&
\begin{tabular}[t]{@{}l@{}}
$K$:$\{K_{29}, K_{34}, K_{33}\}$ \\
$Q$:$\{Q_{21}, Q_{32}, Q_{33}\}$
\end{tabular}
\cardrule

\multicolumn{2}{c}{\textbf{Transferred to}}\\[-1pt]
\multicolumn{2}{c}{
\centering
\tiny
\begin{tabular}[t]{@{}rl@{}}
\multicolumn{2}{c}{\href{https://huggingface.co/datasets/stanfordnlp/sst2}{SA:SST-2}} \\
Acc\({\pm}\)Std & \(95.99{\pm}0.41\)\\
ETA ({\href{https://www.nvidia.com/en-us/data-center/h200/}{H200}}) & \(23.9\) min\\
Upd.\% & \(1.15\%\)\\

\multicolumn{2}{c}{\href{https://huggingface.co/datasets/stanfordnlp/imdb}{SA:IMDB}} \\
Acc\({\pm}\)Std & \(82.97{\pm}0.34\)\\
ETA ({\href{https://www.nvidia.com/en-us/data-center/h200/}{H200}}) & \(163.6\) min\\
Upd.\% & \(0.78\%\)\\

\multicolumn{2}{c}
{\href{https://huggingface.co/datasets/cais/mmlu}{QA:MMLU}} \\
Acc\({\pm}\)Std & \(72.83{\pm}1.97\)\\
ETA ({\href{https://www.nvidia.com/en-us/data-center/h200/}{H200}}) & \(131.7\) min\\
Upd.\% & \(0.68\%\) \\ \\
Forgetting & 
    \begin{tabular}[t]{@{}l@{}}
    \cmark {\href{https://huggingface.co/datasets/stanfordnlp/imdb}{IMDB}},
    \cmark {\href{https://huggingface.co/datasets/hotpotqa/hotpot_qa}{HotpotQA}}, \\ 
    \cmark {\href{https://huggingface.co/datasets/EdinburghNLP/xsum}{XSum}},
    \cmark {\href{https://huggingface.co/datasets/databricks/databricks-dolly-15k}{Dolly-15k}}, \\
    \cmark {\href{https://huggingface.co/datasets/abisee/cnn_dailymail}{CNN/DM}},
    \cmark {\href{https://huggingface.co/datasets/tatsu-lab/alpaca}{Alpaca}}, \\
    \cmark {\href{https://huggingface.co/datasets/openai/openai_humaneval}{HumanEval}},
    \cmark {\href{https://huggingface.co/datasets/Muennighoff/mbpp}{MBPP}}, \\
    \xmark {\href{https://huggingface.co/datasets/cais/mmlu}{MMLU}},
    \xmark {\href{https://huggingface.co/datasets/stanfordnlp/sst2}{SST-2}}, \\
    \xmark {\href{https://huggingface.co/datasets/rajpurkar/squad}{SQuAD}}, 
    \end{tabular}

\end{tabular}
}
\cardrule

\textbf{Stopping $\eta$}
& {SA:SST-2} $\eta$=0.15, {SA:IMDB} $\eta$=0.2, {QA:MMLU} $\eta$=0.25

\end{tabularx}
\end{tcolorbox}

&

\begin{tcolorbox}[modelprofile,title={TDA-High3-Mistral}]
\renewcommand{\arraystretch}{1}
\begin{tabularx}{\linewidth}{@{}p{1cm}X@{}}

\textbf{Model} 
& Mistral-7B-v0.3
\cardrule

\textbf{Prior task} & {\href{https://huggingface.co/datasets/openai/gsm8k}{QA:GSM8K}}
\cardrule

\textbf{Freezing strategy} 
&
\begin{tabular}[t]{@{}l@{}}
Cosine-Wasserstein \\
$H_0$ High-$3$
\end{tabular}
\cardrule

\textbf{Frozen IDs} 
&
\begin{tabular}[t]{@{}l@{}}
$V$:$\{V_{0}, V_{10}, V_{12}\}$ \\
$O$:$\{O_{0}, O_{1}, O_{2}\}$
\end{tabular}
\cardrule

\multicolumn{2}{c}{\textbf{Transferred to}}\\[-1pt]
\multicolumn{2}{c}{
\centering
\tiny
\begin{tabular}[t]{@{}rl@{}}

\multicolumn{2}{c}{\href{https://huggingface.co/datasets/stanfordnlp/sst2}{SA:SST-2}} \\
Acc\({\pm}\)Std & \(96.21{\pm}0.70\)\\
ETA ({\href{https://www.nvidia.com/en-us/data-center/h200/}{H200}}) & \(24.0\) min\\
Upd.\% & \(2.44\%\)\\

\multicolumn{2}{c}{\href{https://huggingface.co/datasets/stanfordnlp/imdb}{SA:IMDB}} \\
Acc\({\pm}\)Std & \(79.76{\pm}0.30\)\\
ETA ({\href{https://www.nvidia.com/en-us/data-center/h200/}{H200}}) & \(137.7\) min\\
Upd.\% & \(2.92\%\)\\

\multicolumn{2}{c}
{\href{https://huggingface.co/datasets/cais/mmlu}{QA:MMLU}} \\
Acc\({\pm}\)Std & \(60.68{\pm}3.06\)\\
ETA ({\href{https://www.nvidia.com/en-us/data-center/h200/}{H200}}) & \(110.3\) min\\
Upd.\% & \(2.98\%\)\\ \\
Forgetting & 
    \begin{tabular}[t]{@{}l@{}}
    \cmark {\href{https://huggingface.co/datasets/cais/mmlu}{MMLU}}, \\
    \cmark {\href{https://huggingface.co/datasets/stanfordnlp/sst2}{SST-2}},
    \cmark {\href{https://huggingface.co/datasets/stanfordnlp/imdb}{IMDB}}, \\
    \cmark {\href{https://huggingface.co/datasets/hotpotqa/hotpot_qa}{HotpotQA}}, 
    \cmark {\href{https://huggingface.co/datasets/EdinburghNLP/xsum}{XSum}}, \\
    \cmark {\href{https://huggingface.co/datasets/abisee/cnn_dailymail}{CNN/DM}},
    \cmark {\href{https://huggingface.co/datasets/openai/openai_humaneval}{HumanEval}}, \\
    \cmark {\href{https://huggingface.co/datasets/Muennighoff/mbpp}{MBPP}}
    \xmark {\href{https://huggingface.co/datasets/rajpurkar/squad}{SQuAD}}, \\
    \xmark {\href{https://huggingface.co/datasets/databricks/databricks-dolly-15k}{Dolly-15k}},
    \xmark {\href{https://huggingface.co/datasets/tatsu-lab/alpaca}{Alpaca}}, \\
    \end{tabular}
\end{tabular}
}
\cardrule

\textbf{Stopping $\eta$}
& {SA:SST-2} $\eta$=0.2, {SA:IMDB} $\eta$=0.15, {QA:MMLU} $\eta$=0.2

\end{tabularx}
\end{tcolorbox}

\end{tabular}
\caption{\textbf{Reusable Freezing Profiles.} We define a reusable freezing profile learned from a prior-task fine-tuning run. Each model card summarizes the model, prior task, candidate matrix set, freezing strategy, selected frozen IDs, and outcome on target datasets. We report three profiles used in our experiments: TDA-High3-LLaMA, TDA-High3-Qwen, and TDA-High3-Mistral.}
\label{fig:freezing-profile-cards}
\end{figure*}
\FloatBarrier
\section{Experimental Setup and Hyperparameters}
\label{app:hyperparameters}

Tables~\ref{tab:training-hyperparameters} and~\ref{tab:evaluation-settings} summarize the training and evaluation setup used in our experiments. Table~\ref{tab:training-hyperparameters} reports task-level training settings, model-specific learning rates, and LoRA configurations, while Table~\ref{tab:evaluation-settings} reports the decoding and checkpoint-scoring settings.

\begin{table*}[!htbp]
\centering
\scriptsize
\caption{Training hyperparameters used in our fine-tuning experiments. Freeze denotes the selective tuning setting used by the Eltwise- and TDA-based methods.}
\label{tab:training-hyperparameters}
\resizebox{\textwidth}{!}{
\begin{tabular}{l|llll}
\toprule
Setting & \texttt{QA:GSM8K} & \texttt{SA:SST-2} & \texttt{SA:IMDB} & \texttt{QA:MMLU} \\
\midrule

\multicolumn{5}{l}{\textit{Task-level training settings}} \\
Epochs
& $6$ & $6$ & $6$ & $6$ \\

Train subset
& $7{,}473$ & $20{,}000$ & $25{,}000$ & $20{,}000$ \\

Batch size
& $16$ & $64$ & $32$ & $32$ \\

Gradient accumulation
& $1$ & $1$ & $1$ & $1$ \\

Max sequence length
& $512$ & $1024$ & $1024$ & $512$ \\

Precision
& bf16 & bf16 & bf16 & bf16 \\

Scheduler
& cosine & cosine & cosine & cosine \\

Optimizer
& AdamW & AdamW & AdamW & AdamW \\

Warmup Ratio
& 0.03 & 0.03 & 0.03 & 0.03 \\

Max grad norm
& $1.0$ & $1.0$ & $1.0$ & $1.0$ \\

Gradient checkpointing
& enabled & enabled & enabled & enabled \\

Seed
& $42 / 123 / 456$ & $42 / 123 / 456$ & $42 / 123 / 456$ & $42 / 123 / 456$ \\

\midrule
\multicolumn{5}{l}{\textit{Model-specific learning rates}} \\

Qwen3-8B-Base full FT LR
& $2{\times}10^{-5}$ & $1{\times}10^{-5}$ & $1{\times}10^{-5}$ & $1{\times}10^{-5}$ \\

LLaMA-3.1-8B full FT LR
& $2{\times}10^{-5}$ & $2{\times}10^{-5}$ & $2{\times}10^{-5}$ & $2{\times}10^{-5}$ \\

Mistral-7B-v0.3 full FT LR
& $5{\times}10^{-6}$ & $5{\times}10^{-6}$ & $5{\times}10^{-6}$ & $5{\times}10^{-6}$ \\

LoRA LR
& $2{\times}10^{-4}$ & $2{\times}10^{-4}$ & $2{\times}10^{-4}$ & $2{\times}10^{-4}$ \\

Spectrum
& same as full FT LR & same as full FT LR & same as full FT LR & same as full FT LR \\

DropBP
& same as LoRA LR & same as LoRA FT LR & same as LoRA FT LR & same as LoRA FT LR \\

Freeze LR
& same as full FT LR & same as full FT LR & same as full FT LR & same as full FT LR \\

\midrule
\multicolumn{5}{l}{\textit{LoRA configuration}} \\

Rank / alpha / dropout
& $16/32/0.05$ & $16/32/0.05$ & $16/32/0.05$ & $16/32/0.05$ \\

Target modules
& $q,k,v,o$ & $q,k,v,o$ & $q,k,v,o$ & $q,k,v,o$ \\

\bottomrule
\end{tabular}}
\end{table*}

\begin{table}[!htbp]
\centering
\scriptsize
\caption{Evaluation settings used for \texttt{SA:SST-2}, \texttt{SA:IMDB}, \texttt{QA:MMLU} checkpoint scoring in the transferred-knowledge experiments.}
\label{tab:evaluation-settings}
\resizebox{0.72\textwidth}{!}{
\begin{tabular}{llll}
\toprule
Setting & \texttt{SA:SST-2} & \texttt{SA:IMDB} & \texttt{QA:MMLU} \\
\midrule
Split & validation & test & validation \\
Decoding & greedy & greedy & greedy \\
Sampling & disabled & disabled & disabled \\
Max new tokens & $4$ & $4$ & $2$ \\
Prompt max length & $256$ & $512$ & $1024$ \\
Metric & Accuracy & Accuracy & Accuracy \\
Subset reporting & three fixed folds & three fixed folds & three fixed folds \\
\bottomrule
\end{tabular}}
\end{table}
\FloatBarrier
\section{Matrix Types, Weight Changes, and Freezing Profiles}
\label{sec:selectionJustification}


To quantify projection-level changes, we compare each epoch-6 fine-tuned matrix with its pretrained counterpart. For each projection matrix $W \in \{K,Q,V,O\}$, we compute mean element-wise relative drift:
$$
d(W)=\frac{1}{n}\sum_{i=1}^{n}\frac{|W_{6,i}-W_{0,i}|}{|W_{0,i}|},
$$
where $W_0$ and $W_6$ denote the pretrained and epoch-6 weights, $n$ is the number of matrix entries.




We justify our matrix selection design by ranking all candidate matrices by their average distance to the pretrained model. We compute the average rank of each matrix family ($K$, $Q$, $V$, and $O$) and count how many $K$ or $Q$ matrices appear among the top distance matrices. As Table~\ref{tab:eltwise-relative-avg-ranks-normalized} and Table~\ref{tab:eltwise-relative-topn-counts} show, empirically, $K$ and $Q$ occupy low positions in the ranking and are largely absent from the top-ranked set, which supports restricting selective tuning to $V$ and $O$.



\begin{table}[!htbp]
\centering
\caption{Normalized average projection ranks (rank divided by total count). Lower values indicate larger drift. \texttt{LLaMA-3.1-8B}/ \texttt{Mistral-7B-v0.3} are normalized by $128$, \texttt{Qwen3-8B-Base} by $144$.}
\label{tab:eltwise-relative-avg-ranks-normalized}
\scriptsize
\setlength{\tabcolsep}{4pt}
\renewcommand{\arraystretch}{1.08}
\begin{tabular}{llccc}
\toprule
Dataset & Statistic & \texttt{LLaMA-3.1-8B} & \texttt{Mistral-7B-v0.3} & \texttt{Qwen3-8B-Base} \\
\midrule
\multirow{4}{*}{\texttt{QA:GSM8K}}
& Avg K rank & 0.85 & 0.71 & 0.29 \\
& Avg Q rank & 0.58 & 0.48 & 0.33 \\
& Avg V rank & 0.26 & 0.41 & 0.75 \\
& Avg O rank & 0.32 & 0.41 & 0.65 \\
\midrule
\multirow{4}{*}{\texttt{SA:IMDB}}
& Avg K rank & 0.79 & 0.61 & 0.41 \\
& Avg Q rank & 0.51 & 0.53 & 0.44 \\
& Avg V rank & 0.35 & 0.44 & 0.62 \\
& Avg O rank & 0.36 & 0.43 & 0.55 \\
\midrule
\multirow{4}{*}{\texttt{QA:MMLU}}
& Avg K rank & 0.72 & 0.66 & 0.43 \\
& Avg Q rank & 0.54 & 0.52 & 0.42 \\
& Avg V rank & 0.37 & 0.43 & 0.64 \\
& Avg O rank & 0.39 & 0.41 & 0.53 \\
\midrule
\multirow{4}{*}{\texttt{SA:SST2}}
& Avg K rank & 0.76 & 0.61 & 0.39 \\
& Avg Q rank & 0.51 & 0.51 & 0.44 \\
& Avg V rank & 0.36 & 0.44 & 0.60 \\
& Avg O rank & 0.39 & 0.45 & 0.59 \\
\bottomrule
\end{tabular}
\end{table}

\begin{table*}[!htbp]
\centering
\caption{Top-$N$ projection counts under mean element-wise relative drift. Each cell reports $K\ Q\ (K{+}Q)\ |\ V\ O\ (V{+}O)$ among the top-$N$ most changed projections. \texttt{LLaMA-3.1-8B}/\texttt{Mistral-7B-v0.3} runs contain 128 projection matrices from 32 layers, while \texttt{Qwen3-8B-Base} runs contain 144 projection matrices from 36 layers.}
\label{tab:eltwise-relative-topn-counts}
\scriptsize
\setlength{\tabcolsep}{3.5pt}
\renewcommand{\arraystretch}{1.08}
\begin{tabular}{llccc}
\toprule
Dataset & Top-$N$ & \texttt{LLaMA-3.1-8B} & \texttt{Mistral-7B-v0.3} & \texttt{Qwen3-8B-Base} \\
\midrule
\multirow{3}{*}{\texttt{QA:GSM8K}}
& Top-10 & 1 1 (2) $|$ 6 2 (8) & 1 4 (5) $|$ 3 2 (5) & 8 2 (10) $|$ 0 0 (0) \\
& Top-20 & 1 1 (2) $|$ 15 3 (18) & 3 7 (10) $|$ 7 3 (10) & 14 6 (20) $|$ 0 0 (0) \\
& Top-30 & 1 1 (2) $|$ 21 7 (28) & 3 10 (13) $|$ 12 5 (17) & 16 14 (30) $|$ 0 0 (0) \\
\midrule
\multirow{3}{*}{\texttt{SA:IMDB}}
& Top-10 & 1 1 (2) $|$ 6 2 (8) & 0 0 (0) $|$ 7 3 (10) & 6 4 (10) $|$ 0 0 (0) \\
& Top-20 & 1 1 (2) $|$ 12 6 (18) & 1 0 (1) $|$ 13 6 (19) & 9 11 (20) $|$ 0 0 (0) \\
& Top-30 & 1 2 (3) $|$ 14 13 (27) & 1 1 (2) $|$ 14 14 (28) & 15 13 (28) $|$ 0 2 (2) \\
\midrule
\multirow{3}{*}{\texttt{QA:MMLU}}
& Top-10 & 1 1 (2) $|$ 6 2 (8) & 1 0 (1) $|$ 6 3 (9) & 6 4 (10) $|$ 0 0 (0) \\
& Top-20 & 1 1 (2) $|$ 13 5 (18) & 1 0 (1) $|$ 12 7 (19) & 10 7 (17) $|$ 1 2 (3) \\
& Top-30 & 1 1 (2) $|$ 15 13 (28) & 1 0 (1) $|$ 14 15 (29) & 10 12 (22) $|$ 4 4 (8) \\
\midrule
\multirow{3}{*}{\texttt{SA:SST2}}
& Top-10 & 1 1 (2) $|$ 6 2 (8) & 0 0 (0) $|$ 7 3 (10) & 8 2 (10) $|$ 0 0 (0) \\
& Top-20 & 1 1 (2) $|$ 11 7 (18) & 0 0 (0) $|$ 12 8 (20) & 13 7 (20) $|$ 0 0 (0) \\
& Top-30 & 1 2 (3) $|$ 14 13 (27) & 0 1 (1) $|$ 15 14 (29) & 17 11 (28) $|$ 1 1 (2) \\
\bottomrule
\end{tabular}
\end{table*}

\FloatBarrier
\section{Model Weight-Change Statistics}
\label{sec:weightChangeStatistics}

We report full-model weight-change statistics by comparing each fine-tuned checkpoint against its pretrained initialization. Changed Params. quantifies how widely the checkpoint differs from the pretrained model by counting parameters whose absolute update exceeds a fixed threshold:
\[
|w_{\mathrm{final}} - w_{\mathrm{base}}| > 10^{-6}.
\]
Mean Rel. $\Delta$ measures the average relative magnitude of the parameter update:
\[
\mathrm{Mean\ Rel.}\ \Delta
=
\mathrm{mean}\left(
\frac{|w_{\mathrm{final}} - w_{\mathrm{base}}|}{|w_{\mathrm{base}}|}
\right).
\]
Together, these statistics capture both the spread of parameter changes across the model and the average size of those changes relative to the pretrained weights.



\begin{table*}[!htbp]
\centering
\scriptsize
\caption{Model weight-change statistics relative to the pretrained checkpoint. TDA-High3 uses $86.25\%$ less parameter updates than LoRA across the 12 LLM-dataset settings.}
\label{tab:full-model-weight-change}
\setlength{\tabcolsep}{3pt}
\renewcommand{\arraystretch}{1.08}
\resizebox{\textwidth}{!}{%
\begin{tabular}{l|rrrr|rrrr|rrrr}
\toprule
\multirow{2}{*}{Method}
& \multicolumn{4}{c|}{\texttt{Qwen3-8B-Base}}
& \multicolumn{4}{c|}{\texttt{LLaMA-3.1-8B}}
& \multicolumn{4}{c}{\texttt{Mistral-7B-v0.3}} \\
\cmidrule(lr){2-5}\cmidrule(lr){6-9}\cmidrule(lr){10-13}
& \shortstack[c]{Total\\Params.} 
& \shortstack[c]{Changed\\Params.} 
& \shortstack[c]{Upd.\\(\%)} 
& \shortstack[c]{Mean Rel.\\$\Delta$ (\%)} 
& \shortstack[c]{Total\\Params.} 
& \shortstack[c]{Changed\\Params.} 
& \shortstack[c]{Upd.\\(\%)} 
& \shortstack[c]{Mean Rel.\\$\Delta$ (\%)} 
& \shortstack[c]{Total\\Params.} 
& \shortstack[c]{Changed\\Params.} 
& \shortstack[c]{Upd.\\(\%)} 
& \shortstack[c]{Mean Rel.\\$\Delta$ (\%)} \\
\midrule

\multicolumn{13}{l}{\texttt{QA:GSM8K}} \\
Full
& 8.19B & 1.198B & 14.63 & 12.57
& 8.03B & 2.211B & 27.54 & 26.12
& 7.25B & 2.280B & 31.46 & 25.57 \\
LoRA
& 8.19B & 1.508B & 18.41 & 3.89
& 8.03B & 1.341B & 16.70 & 9.13
& 7.25B & 1.340B & 18.49 & 36.57 \\
TDA-High3
& 8.19B & 0.100B & 1.22 & 1.65
& 8.03B & 0.212B & 2.64 & 3.17
& 7.25B & 0.204B & 2.82 & 3.02 \\
TDA-High6
& 8.19B & 0.091B & 1.12 & 1.55
& 8.03B & 0.187B & 2.33 & 2.89
& 7.25B & 0.183B & 2.52 & 2.79 \\
TDA-High9
& 8.19B & 0.082B & 1.00 & 1.45
& 8.03B & 0.165B & 2.05 & 2.60
& 7.25B & 0.162B & 2.23 & 2.55 \\
Eltwise-High3
& 8.19B & 0.099B & 1.21 & 1.62
& 8.03B & 0.220B & 2.74 & 3.17
& 7.25B & 0.213B & 2.94 & 3.11 \\
Eltwise-High6
& 8.19B & 0.091B & 1.11 & 1.52
& 8.03B & 0.196B & 2.44 & 2.88
& 7.25B & 0.194B & 2.67 & 2.85 \\
Eltwise-High9
& 8.19B & 0.082B & 1.01 & 1.41
& 8.03B & 0.175B & 2.18 & 2.63
& 7.25B & 0.167B & 2.30 & 2.53 \\
\midrule

\multicolumn{13}{l}{\texttt{SA:SST-2}} \\
Full
& 8.19B & 1.066B & 13.01 & 8.78
& 8.03B & 1.757B & 21.88 & 11.07
& 7.25B & 1.812B & 25.00 & 12.56 \\
LoRA
& 8.19B & 0.903B & 11.02 & 1.01
& 8.03B & 1.335B & 16.62 & 2.03
& 7.25B & 1.334B & 18.41 & 15.87 \\
TDA-High3
& 8.19B & 0.094B & 1.15 & 0.93
& 8.03B & 0.162B & 2.02 & 1.27
& 7.25B & 0.177B & 2.44 & 1.44 \\
\midrule

\multicolumn{13}{l}{\texttt{SA:IMDB}} \\
Full
& 8.19B & 0.683B & 8.34 & 6.69
& 8.03B & 2.004B & 24.95 & 16.87
& 7.25B & 2.078B & 28.67 & 16.88 \\
LoRA
& 8.19B & 1.505B & 18.38 & 3.11
& 8.03B & 1.338B & 16.67 & 5.04
& 7.25B & 1.340B & 18.49 & 39.57 \\
TDA-High3
& 8.19B & 0.064B & 0.78 & 0.78
& 8.03B & 0.189B & 2.36 & 1.88
& 7.25B & 0.211B & 2.92 & 1.96 \\
\midrule

\multicolumn{13}{l}{\texttt{QA:MMLU}} \\
Full
& 8.19B & 0.770B & 9.39 & 7.16
& 8.03B & 2.513B & 31.29 & 21.72
& 7.25B & 2.564B & 35.38 & 21.22 \\
LoRA
& 8.19B & 1.508B & 18.41 & 4.02
& 8.03B & 1.339B & 16.67 & 5.33
& 7.25B & 1.340B & 18.49 & 34.39 \\
TDA-High3
& 8.19B & 0.056B & 0.68 & 0.86
& 8.03B & 0.200B & 2.49 & 2.06
& 7.25B & 0.216B & 2.98 & 2.13 \\

\bottomrule
\end{tabular}%
}
\end{table*}

\FloatBarrier
\section{Magnitude Changes in Models}
\label{app:magnitude-changes}

\xhdr{Entrywise weight change}
We first report entrywise weight movement as a magnitude-based baseline. For each fine-tuned matrix, we compare the epoch-6 weights against the pretrained weights and summarize the raw mean absolute change for the \(V\) and \(O\) projections. The entrywise plots in Figure~\ref{fig:eltwise_distances} provide a direct view of parameter movement across datasets and model families. This magnitude-based view is later compared with the topology-based distance used for matrix ranking and freezing.

\begin{figure}[!htbp]
\centering
\setlength{\tabcolsep}{2pt}

\begin{tabular}{ccc}
\multicolumn{3}{c}{\small \textit{\(V\) Projection}} \\
\includegraphics[width=0.32\linewidth]{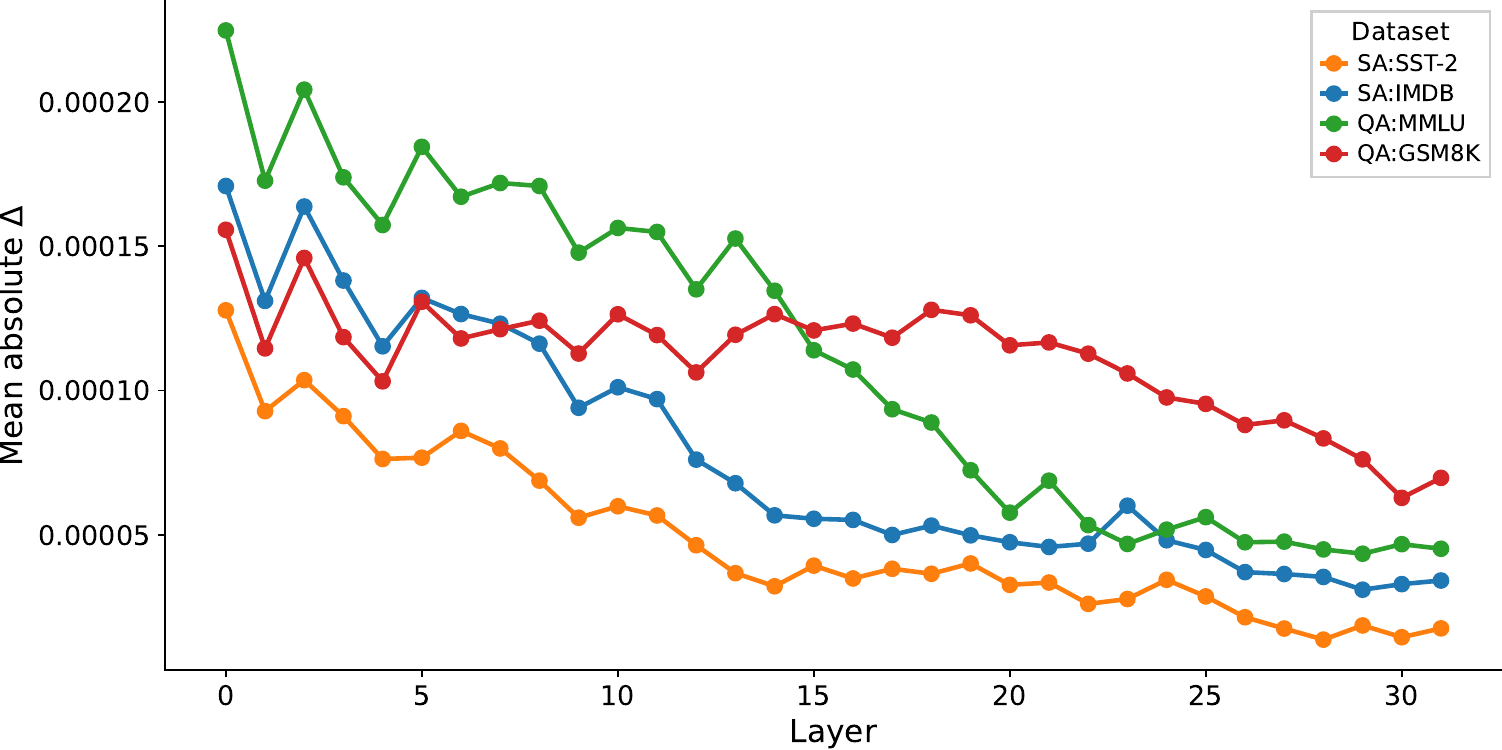} &
\includegraphics[width=0.32\linewidth]{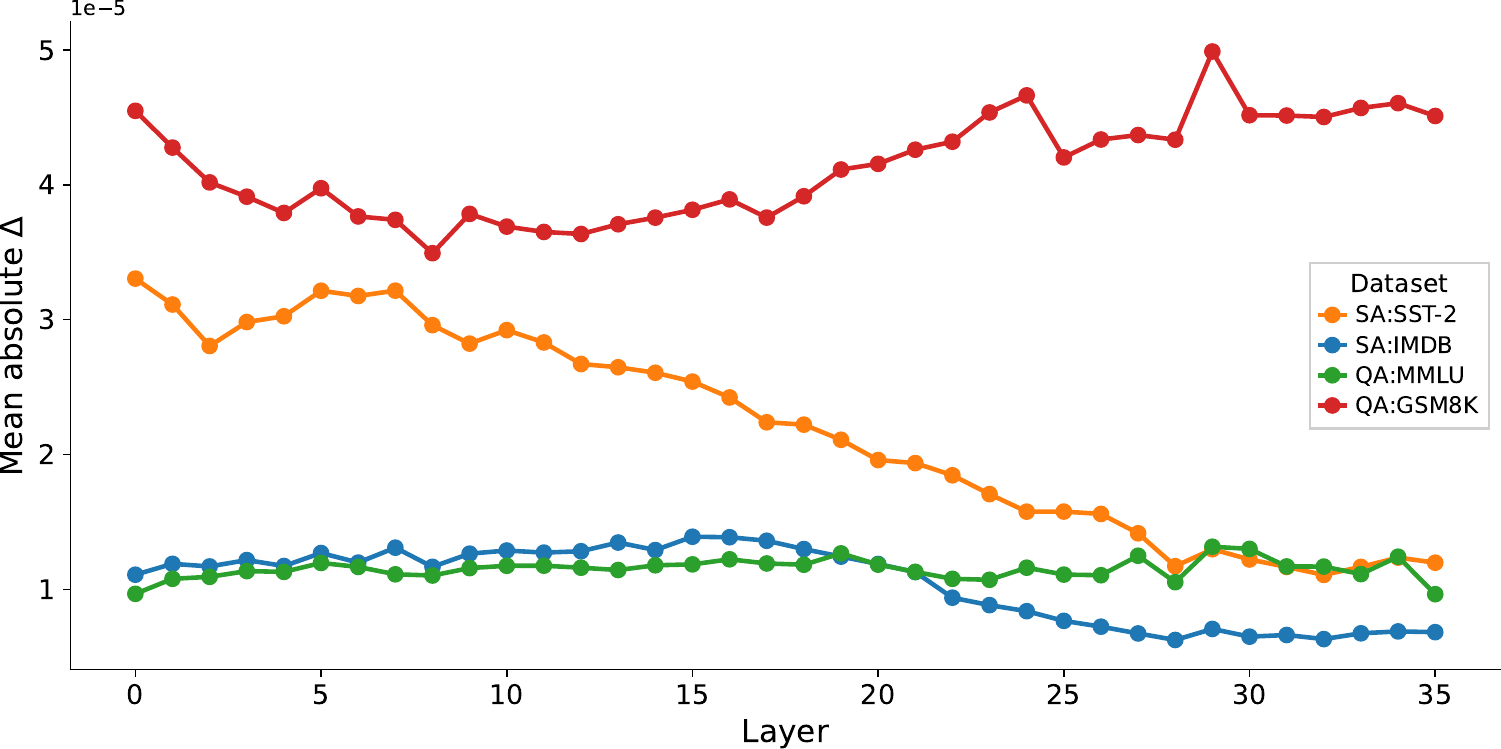} &
\includegraphics[width=0.32\linewidth]{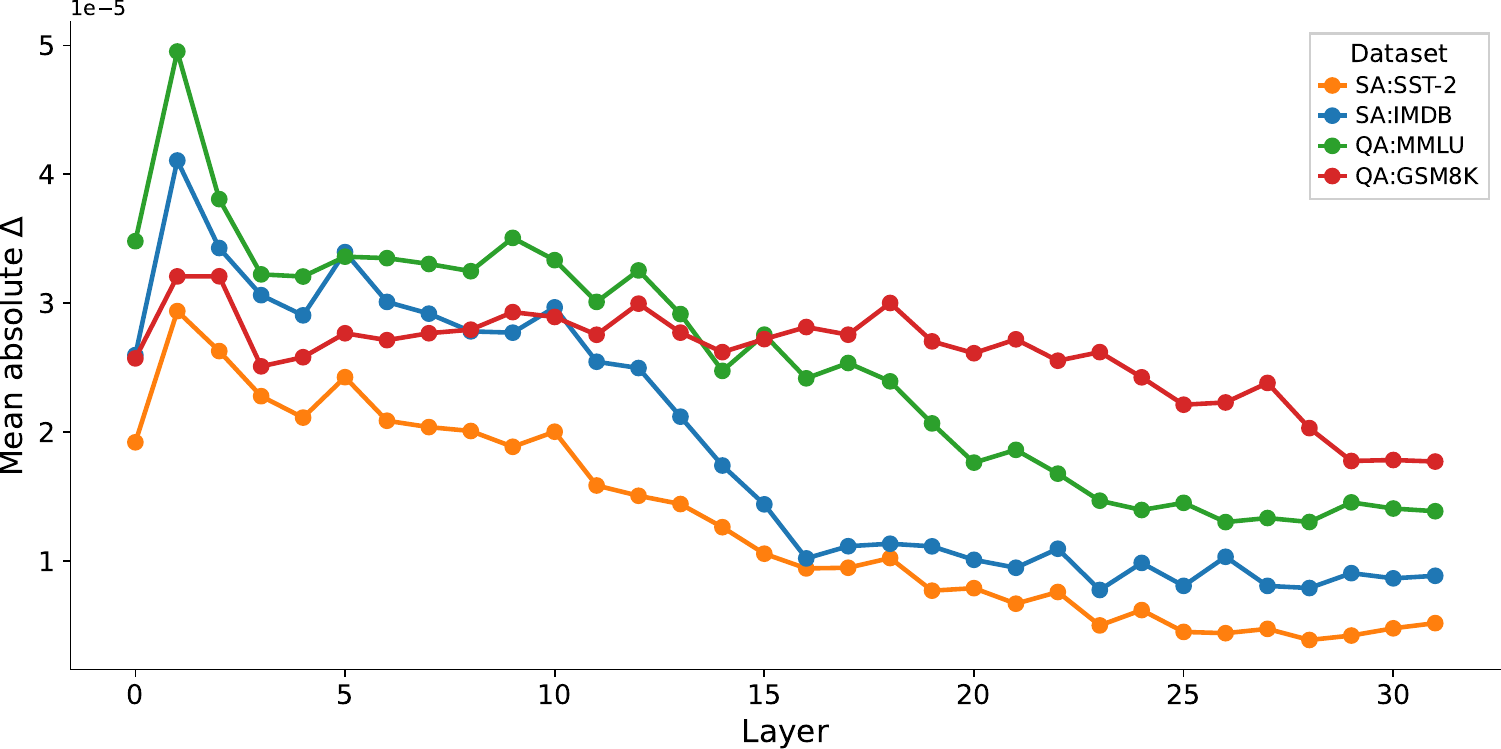} \\
[2em]
\multicolumn{3}{c}{\small \textit{\(O\) Projection}} \\
\includegraphics[width=0.32\linewidth]
{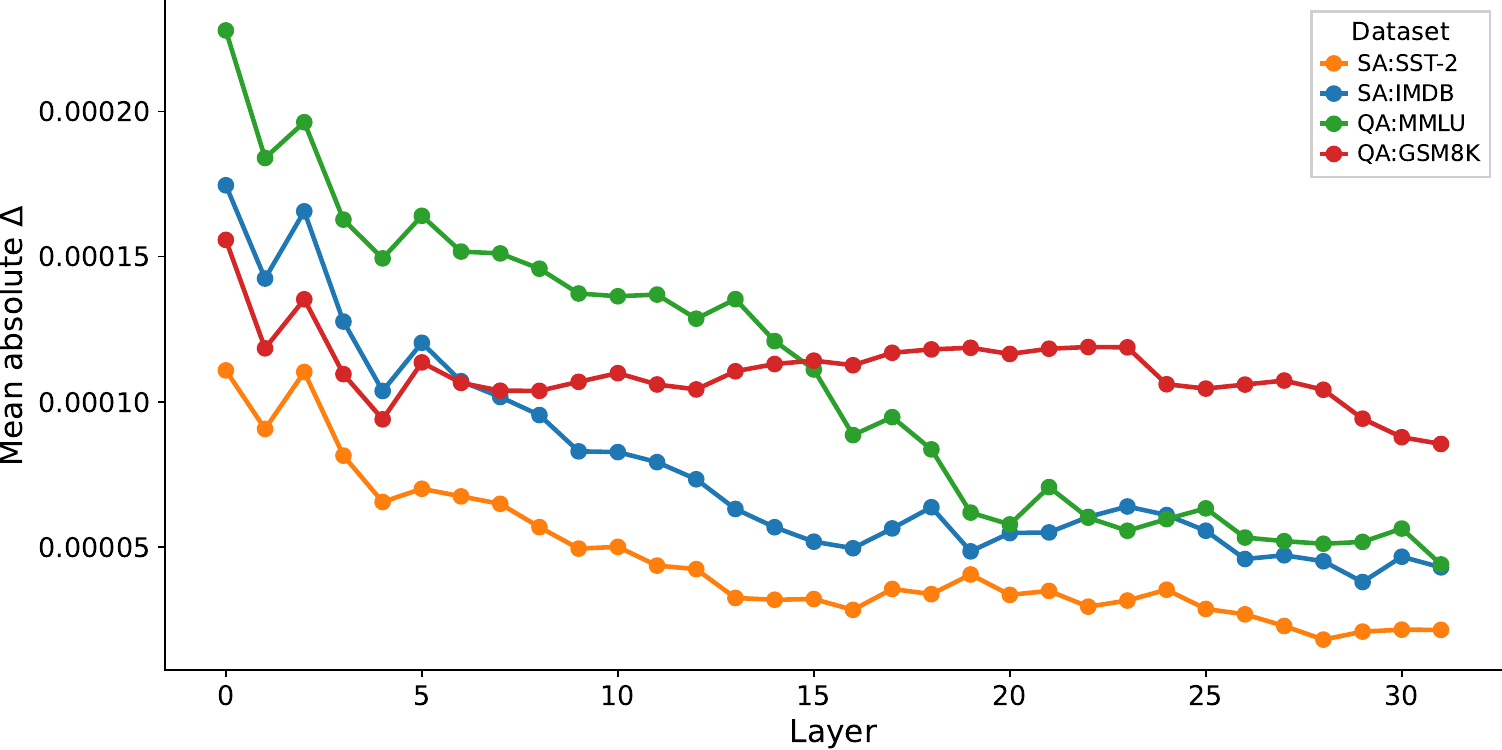} &
\includegraphics[width=0.32\linewidth]{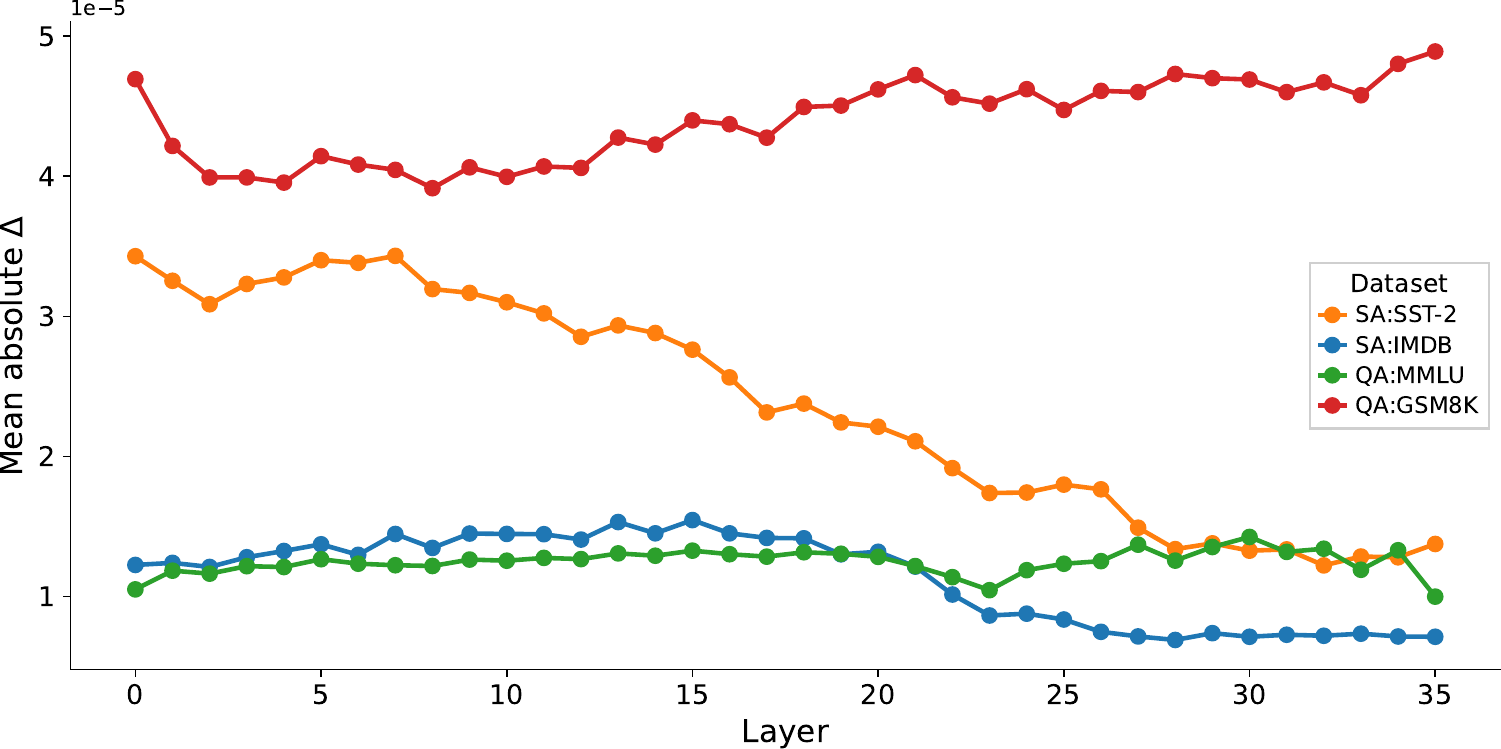} &
\includegraphics[width=0.32\linewidth]{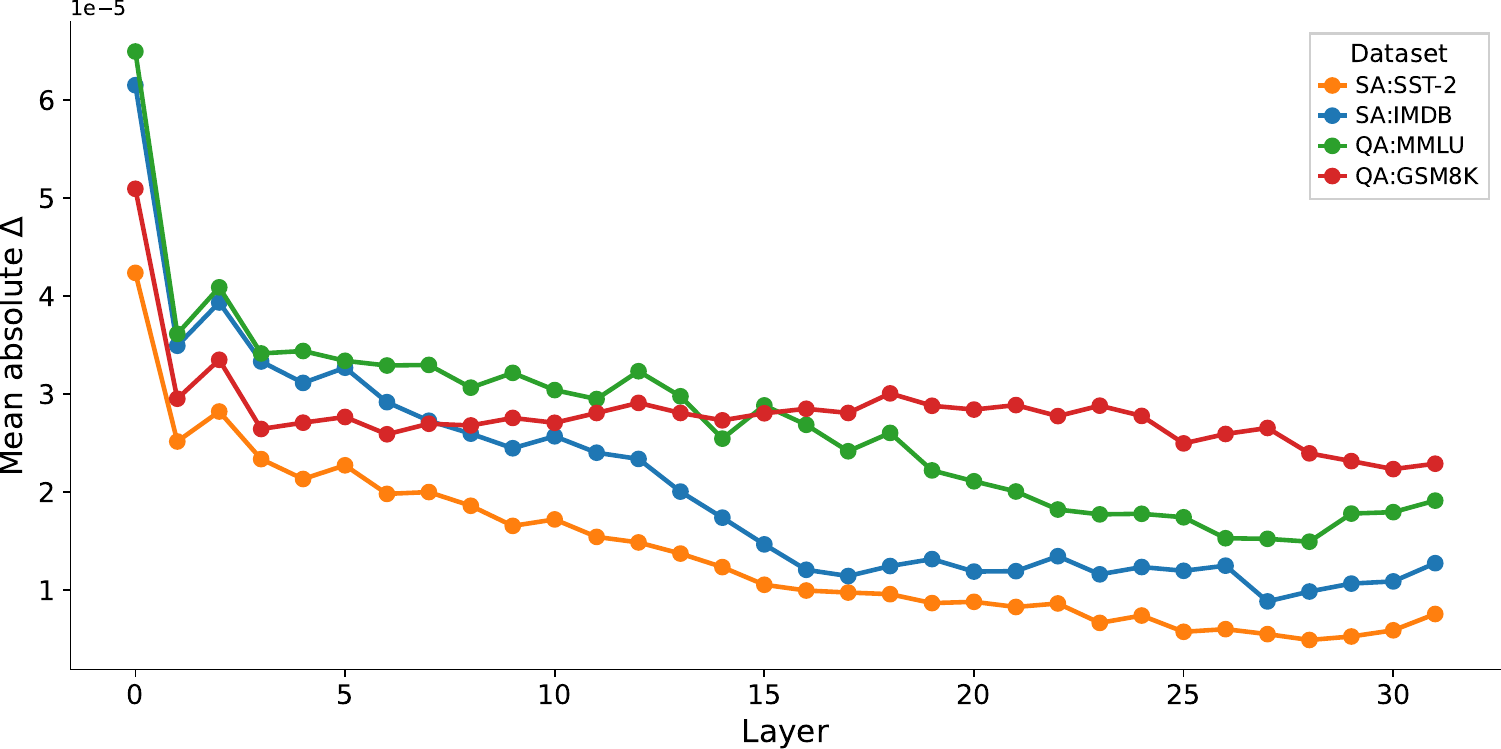}\\
\small \texttt{LLaMA-3.1-8B} & \small \texttt{Qwen3-8B-Base} & \small \texttt{Mistral-7B-v0.3} \\
\end{tabular}

\caption{Epoch-6 Eltwise distances for the \(V\) and \(O\) projection matrices under $H_0$.}
\label{fig:eltwise_distances}
\end{figure}

\newpage
\FloatBarrier
\section{Wasserstein Distances in Models}
\label{app:wasserstein-distances}

\xhdr{Wasserstein distance}
We also report the Wasserstein-based structural distance used by our topological analysis. As defined later in Definition~\ref{def:distance}, each projection matrix is treated as a row cloud and compared to its pretrained version through the Wasserstein distance between persistence diagrams. The Wasserstein plots in Figure~\ref{fig:wass_distances} show this epoch-6 structural change for the \(V\) and \(O\) projections under the same model and dataset ordering as the entrywise plots.

\begin{figure}[!htbp]
\centering
\setlength{\tabcolsep}{2pt}

\begin{tabular}{ccc}
\multicolumn{3}{c}{\small \textit{\(V\) Projection}} \\
\includegraphics[width=0.32\linewidth]{figs/figs_wass_full_by_task/wasserstein_distance_dataset_comparison/llama_snapshot_epoch6_v.pdf} &
\includegraphics[width=0.32\linewidth]{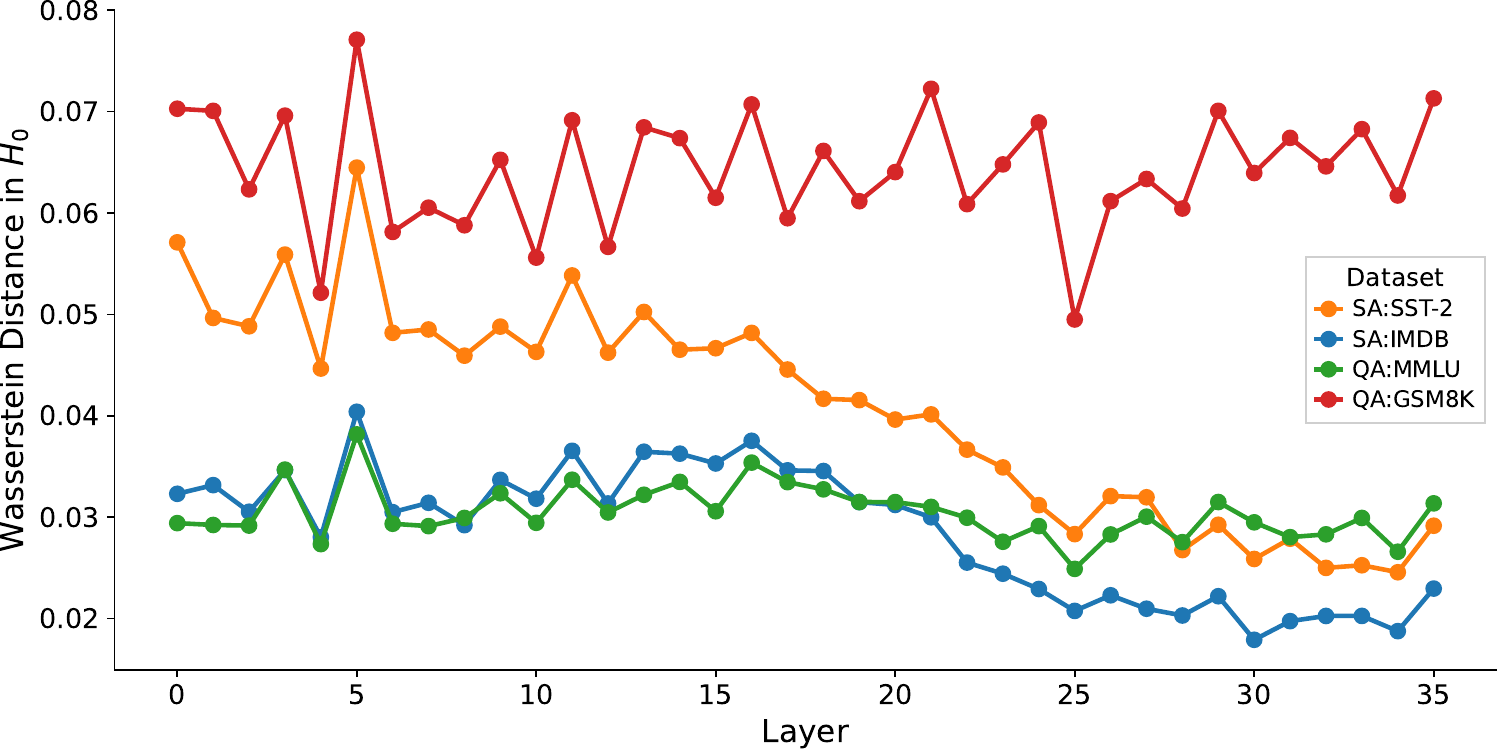} &
\includegraphics[width=0.32\linewidth]{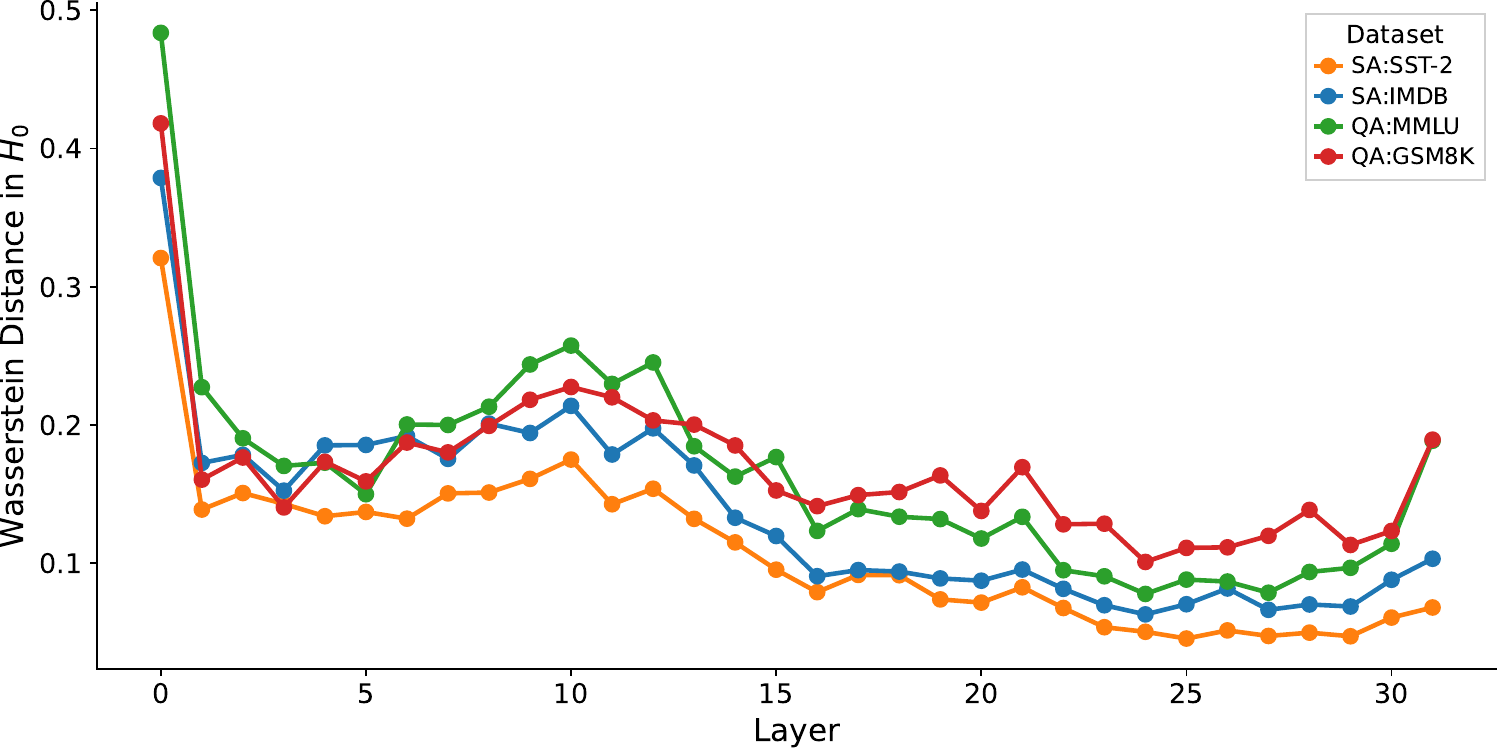} \\
[2em]
\multicolumn{3}{c}{\small \textit{\(O\) Projection}} \\
\includegraphics[width=0.32\linewidth]{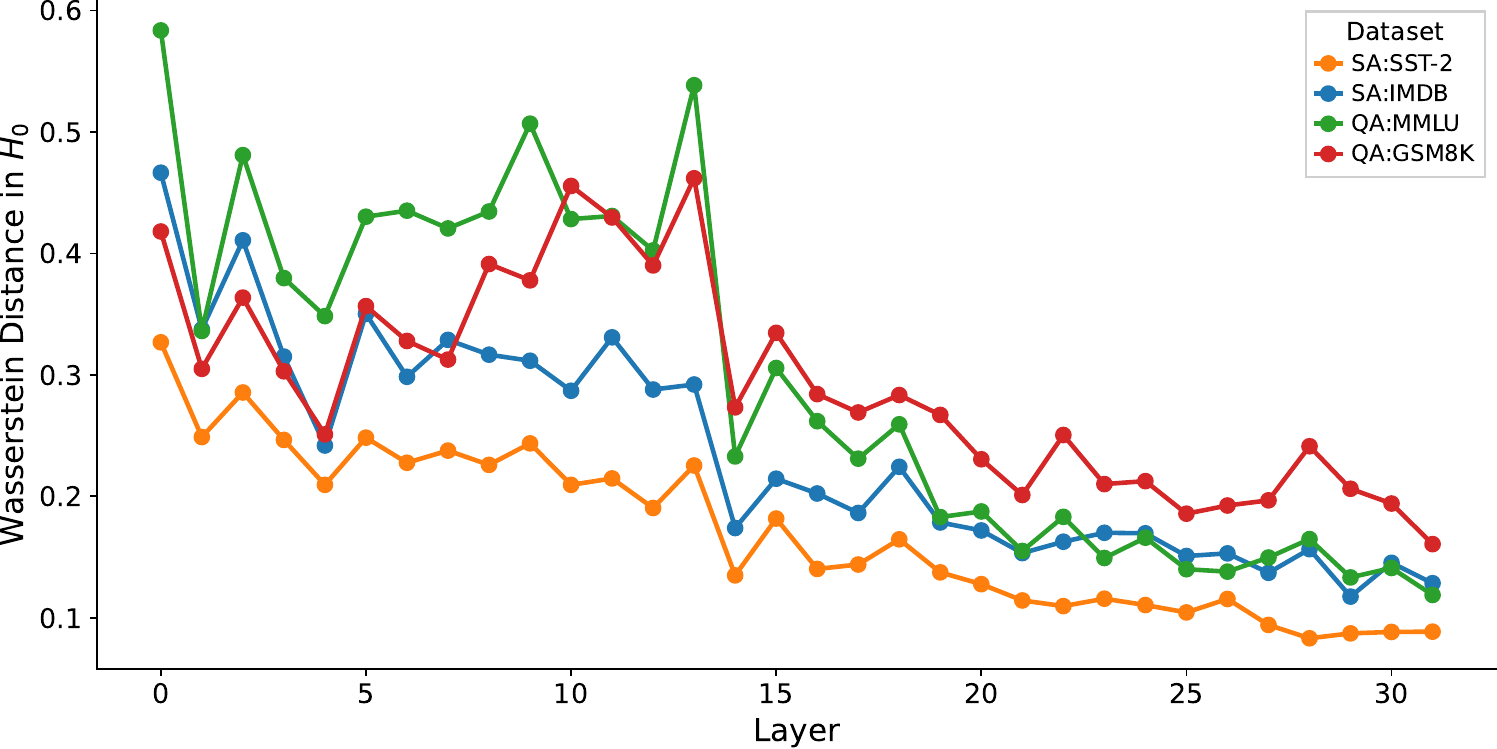} &
\includegraphics[width=0.32\linewidth]{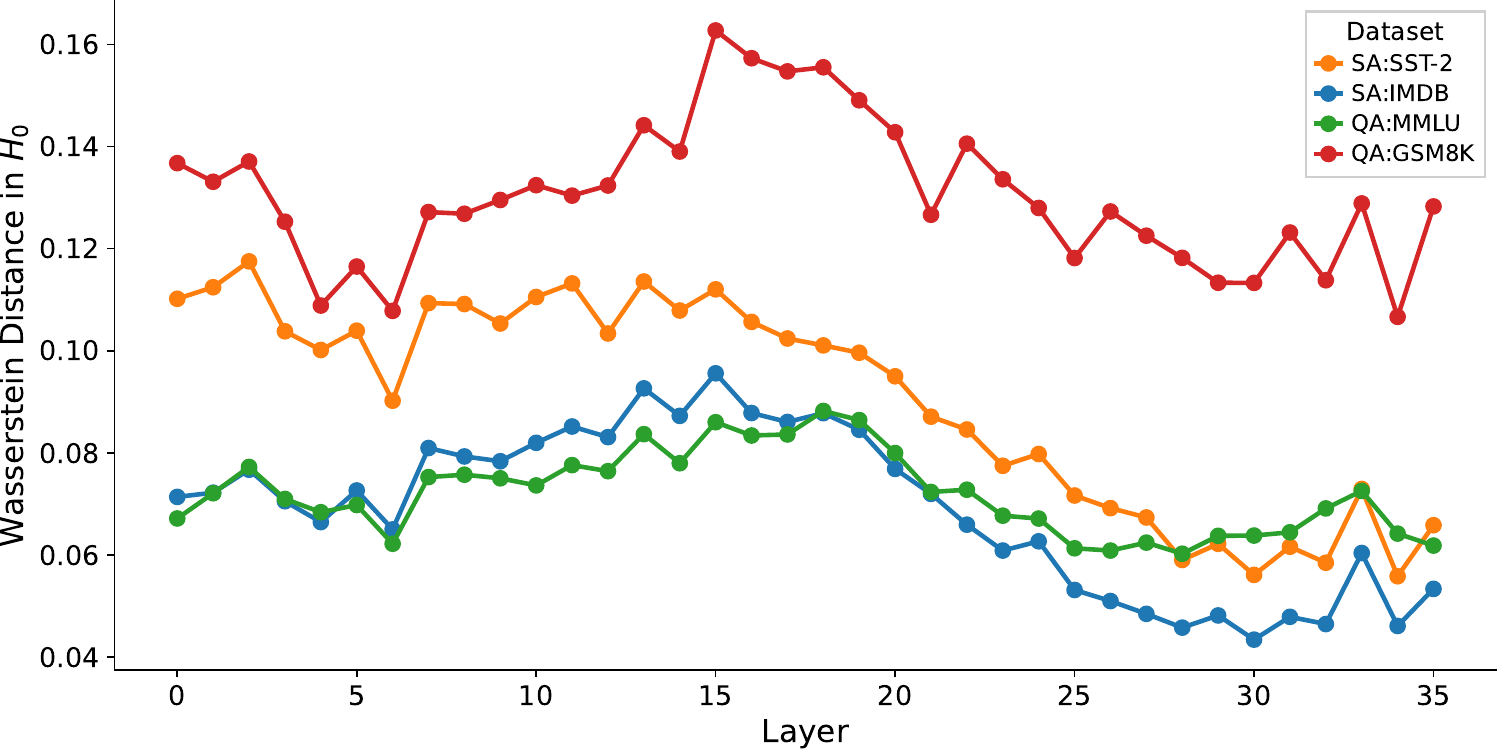} &
\includegraphics[width=0.32\linewidth]{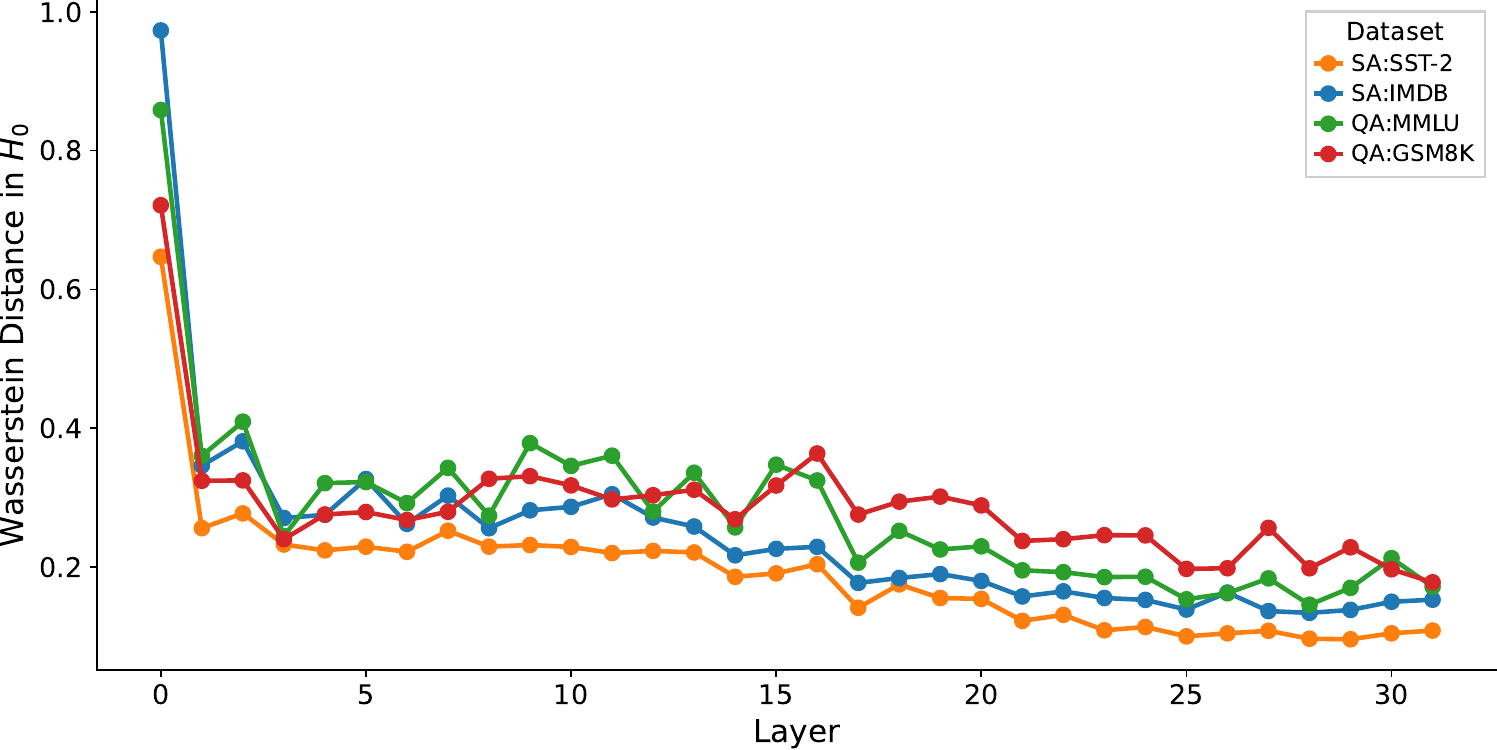}\\

\small \texttt{LLaMA-3.1-8B} & \small \texttt{Qwen3-8B-Base} & \small \texttt{Mistral-7B-v0.3} \\


\end{tabular}

\caption{Epoch-6 Wasserstein distances for the \(V\) and \(O\) projection matrices under $H_0$.}
\label{fig:wass_distances}
\end{figure}
\FloatBarrier
\section{Freezing Pseudocode}
\label{app:freezing-pseudocode}

Algorithms~\ref{alg:layer-orderings}--\ref{alg:selective-kqvo-freezing} summarize the full \method selection pipeline. Algorithm~\ref{alg:layer-orderings} computes projection-specific layer orderings using Wasserstein topological drift, Algorithm~\ref{alg:frozen-sets} converts these orderings into frozen \(K/Q/V/O\) layer sets under the Low or High regime, and Algorithm~\ref{alg:selective-kqvo-freezing} applies the resulting sets by freezing all parameters and re-enabling gradients only for the selected unfrozen \(K/Q/V/O\) projections.

\begin{algorithm}[t]
\caption{Computing TDA-Based Wasserstein Layer Orderings}
\label{alg:layer-orderings}
\begin{algorithmic}[1]
\Require Pretrained model \(f_{\theta^{(0)}}\); fully fine-tuned model \(f_{\theta^{(T)}}\); projection set \(\mathcal{P}=\{K,Q,V,O\}\); Wasserstein order \(p\)
\Ensure Ordered layer lists \(\pi_K,\pi_Q,\pi_V,\pi_O\)

\For{$P \in \mathcal{P}$}
    \For{each transformer layer \(\ell\)}
        \State Extract pretrained matrix \(W_{\ell,P}^{(0)}\) and fine-tuned matrix \(W_{\ell,P}^{(T)}\)
        \State \(X_{\ell,P}^{(0)} \gets \mathrm{rows}(W_{\ell,P}^{(0)})\), \quad \(X_{\ell,P}^{(T)} \gets \mathrm{rows}(W_{\ell,P}^{(T)})\)
        \State \(D_{\ell,P}^{(0)} \gets \mathrm{PD}(X_{\ell,P}^{(0)})\), \quad \(D_{\ell,P}^{(T)} \gets \mathrm{PD}(X_{\ell,P}^{(T)})\)
        \State \(s_{\ell,P} \gets W_p\left(D_{\ell,P}^{(0)},D_{\ell,P}^{(T)}\right)\) \Comment{Wasserstein distance between persistence diagrams}
    \EndFor
    \State \(\pi_P \gets \mathrm{argsort}_{\ell}(s_{\ell,P})\) \Comment{low to high Wasserstein score}
\EndFor
\State \Return \(\pi_K,\pi_Q,\pi_V,\pi_O\)
\end{algorithmic}
\end{algorithm}

\begin{algorithm}[!htbp]
\caption{Constructing Frozen Layer Sets from Wasserstein Layer Orderings}
\label{alg:frozen-sets}
\begin{algorithmic}[1]
\Require Ordered layer lists $\pi_K,\pi_Q,\pi_V,\pi_O$ sorted from low to high Wasserstein score; number of layers $L$; freezing budget $B$; regime $r\in\{\textsc{Low},\textsc{High}\}$
\Ensure Frozen layer sets $\mathcal{F}_K,\mathcal{F}_Q,\mathcal{F}_V,\mathcal{F}_O$

\For{$P \in \{K,Q,V,O\}$}
    \If{$r=\textsc{Low}$}
        \State $\mathcal{F}_P \gets \{\pi_P[0],\ldots,\pi_P[B-1]\}$
    \ElsIf{$r=\textsc{High}$}
        \State $\mathcal{F}_P \gets \{\pi_P[L-B],\ldots,\pi_P[L-1]\}$
    \EndIf
\EndFor

\State \Return $\mathcal{F}_K,\mathcal{F}_Q,\mathcal{F}_V,\mathcal{F}_O$
\end{algorithmic}
\end{algorithm}

\begin{algorithm}[!htbp]
\caption{Selective \(K/Q/V/O\) Tuning from Frozen Layer Sets}
\label{alg:selective-kqvo-freezing}
\begin{algorithmic}[1]
\Require Model $f_\theta$; frozen layer sets $\mathcal{F}_K,\mathcal{F}_Q,\mathcal{F}_V,\mathcal{F}_O$
\Ensure Model with selected \(K/Q/V/O\) parameters trainable

\ForAll{parameters $p \in \theta$}
    \State $p.\mathrm{requires\_grad} \gets \mathrm{False}$
\EndFor

\ForAll{named parameters $(n,p)$ in $f_\theta$}
    \State Parse transformer layer index $\ell$ from parameter name $n$
    \If{$n$ contains \texttt{k\_proj} and $\ell \notin \mathcal{F}_K$}
        \State $p.\mathrm{requires\_grad} \gets \mathrm{True}$
    \ElsIf{$n$ contains \texttt{q\_proj} and $\ell \notin \mathcal{F}_Q$}
        \State $p.\mathrm{requires\_grad} \gets \mathrm{True}$
    \ElsIf{$n$ contains \texttt{v\_proj} and $\ell \notin \mathcal{F}_V$}
        \State $p.\mathrm{requires\_grad} \gets \mathrm{True}$
    \ElsIf{$n$ contains \texttt{o\_proj} and $\ell \notin \mathcal{F}_O$}
        \State $p.\mathrm{requires\_grad} \gets \mathrm{True}$
    \EndIf
\EndFor

\State Keep all non-\(K/Q/V/O\) parameters frozen
\State \Return $f_\theta$
\end{algorithmic}
\end{algorithm}
\FloatBarrier
\section{Topological Drift Across Models and Datasets}
\label{app:topological-drift}

We use the topological distance to initialization and topological drift defined in Definition~\ref{def:distance} and Definition~\ref{def:topological-drift} to visualize how attention projection matrices change during fine-tuning. Figures~\ref{fig:topodrift-gsm8k-v-h0}-\ref{fig:topodrift-sst2-v-h0} show this quantity for the \(V\) and \(O\) projections across datasets and models. The drift-bar plots summarize epoch-to-epoch changes, while the line plots show the Wasserstein distance from initialization across layers for each epoch.


\begin{figure*}[!htbp]
\centering
\scriptsize
\setlength{\tabcolsep}{2pt}
\begin{tabular}{ccc}
\multicolumn{3}{c}{\small \textit{\(V\) Projection}} \\
\includegraphics[width=0.32\textwidth]{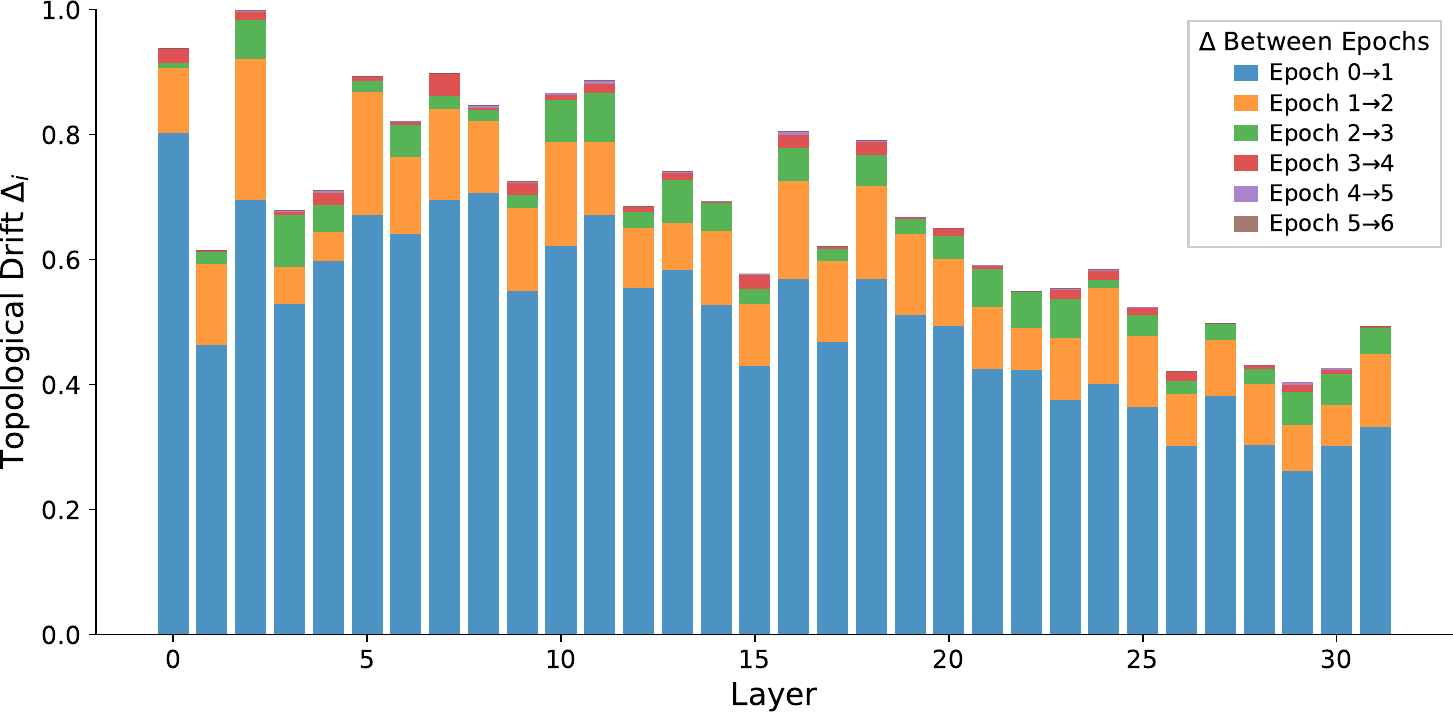} &
\includegraphics[width=0.32\textwidth]{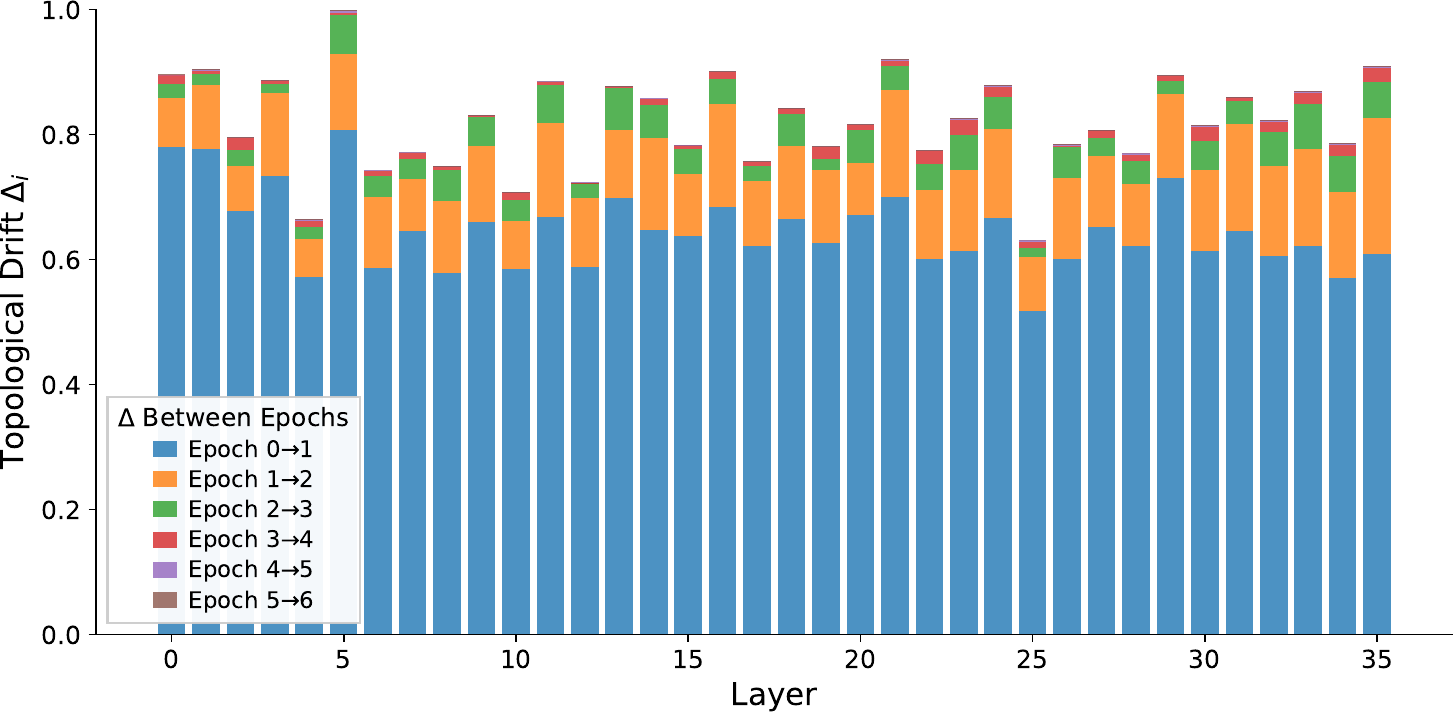} &
\includegraphics[width=0.32\textwidth]{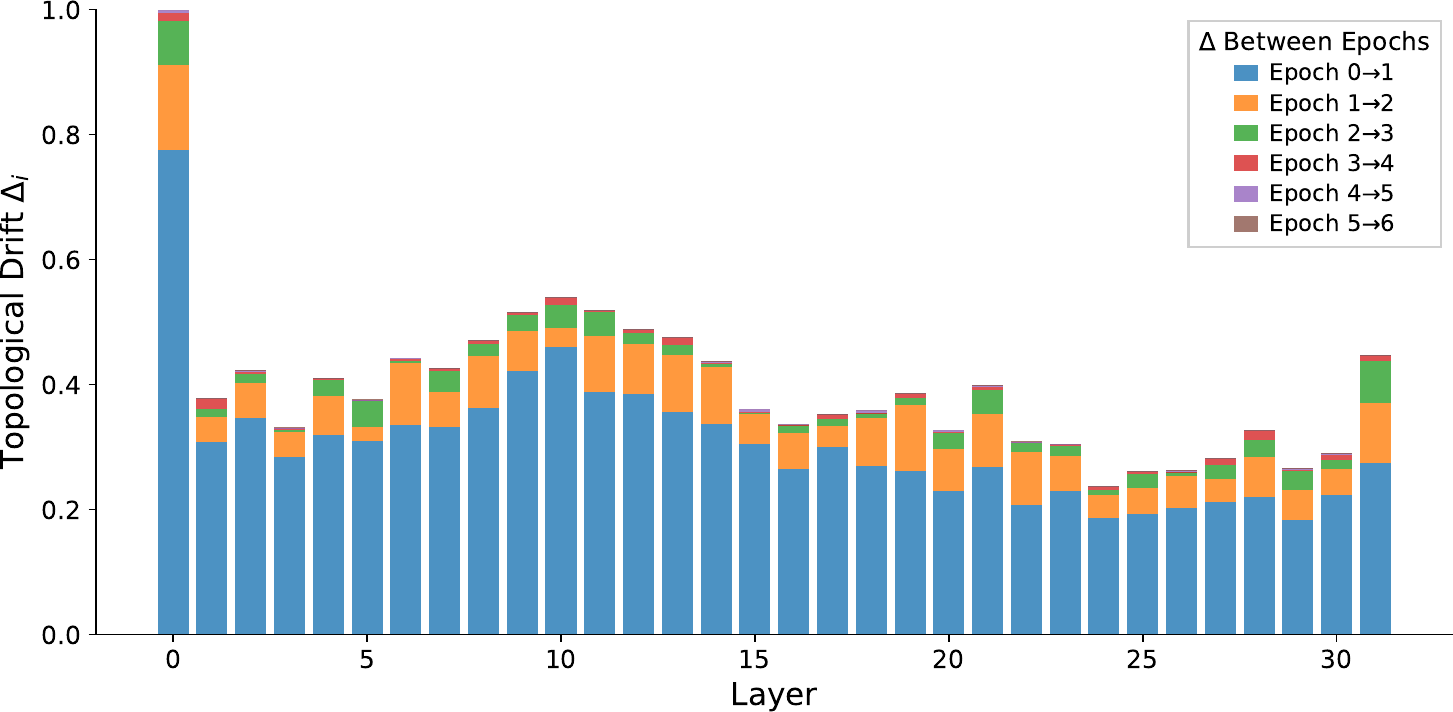} \\
[1em]
\includegraphics[width=0.32\textwidth]{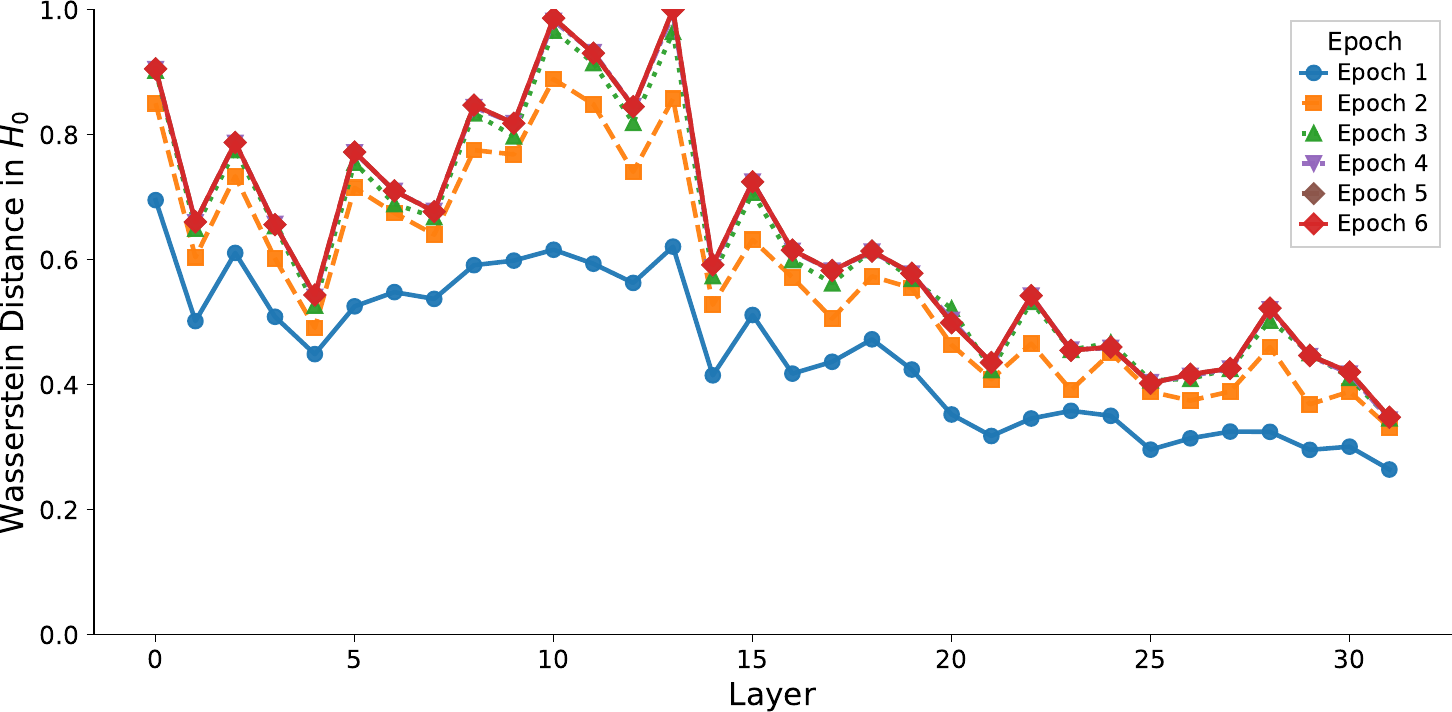} &
\includegraphics[width=0.32\textwidth]{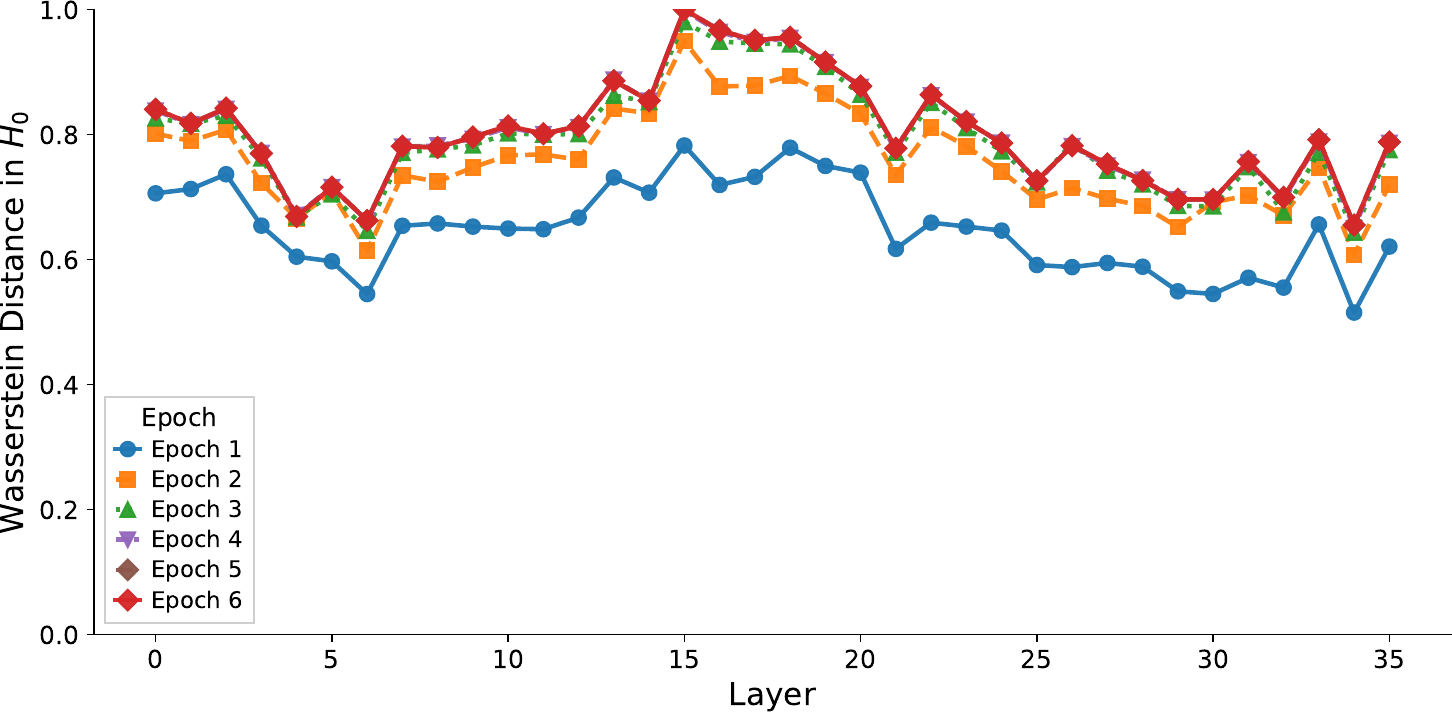} &
\includegraphics[width=0.32\textwidth]{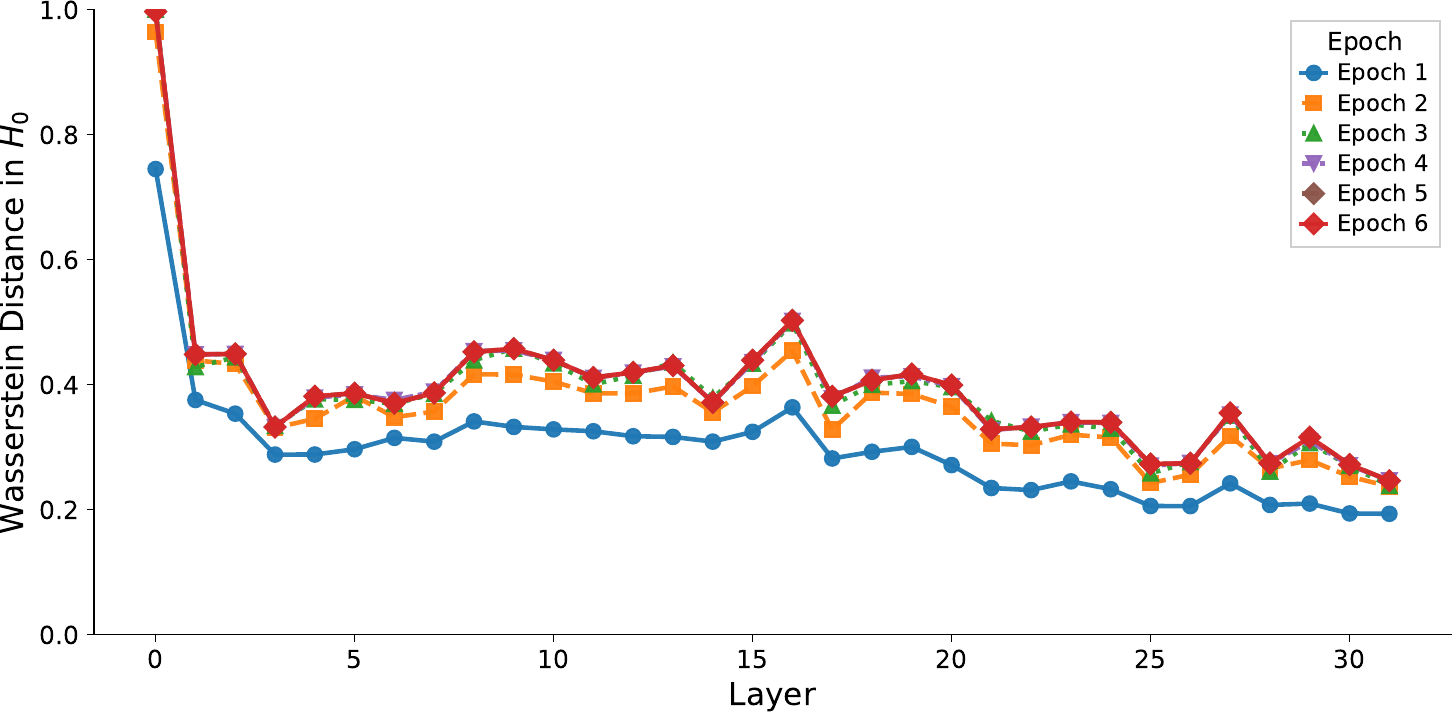} \\
[2em]
\multicolumn{3}{c}{\small \textit{\(O\) Projection}} \\
\includegraphics[width=0.32\textwidth]{figs/figs_wass_full_by_task/topological_drift/llama31_8b_gsm8k_v_h0_driftbars.pdf} &
\includegraphics[width=0.32\textwidth]{figs/figs_wass_full_by_task/topological_drift/qwen_8b_gsm8k_v_h0_driftbars.pdf} &
\includegraphics[width=0.32\textwidth]{figs/figs_wass_full_by_task/topological_drift/mistral7b_gsm8k_v_h0_driftbars.pdf} \\
[1em]
\includegraphics[width=0.32\textwidth]{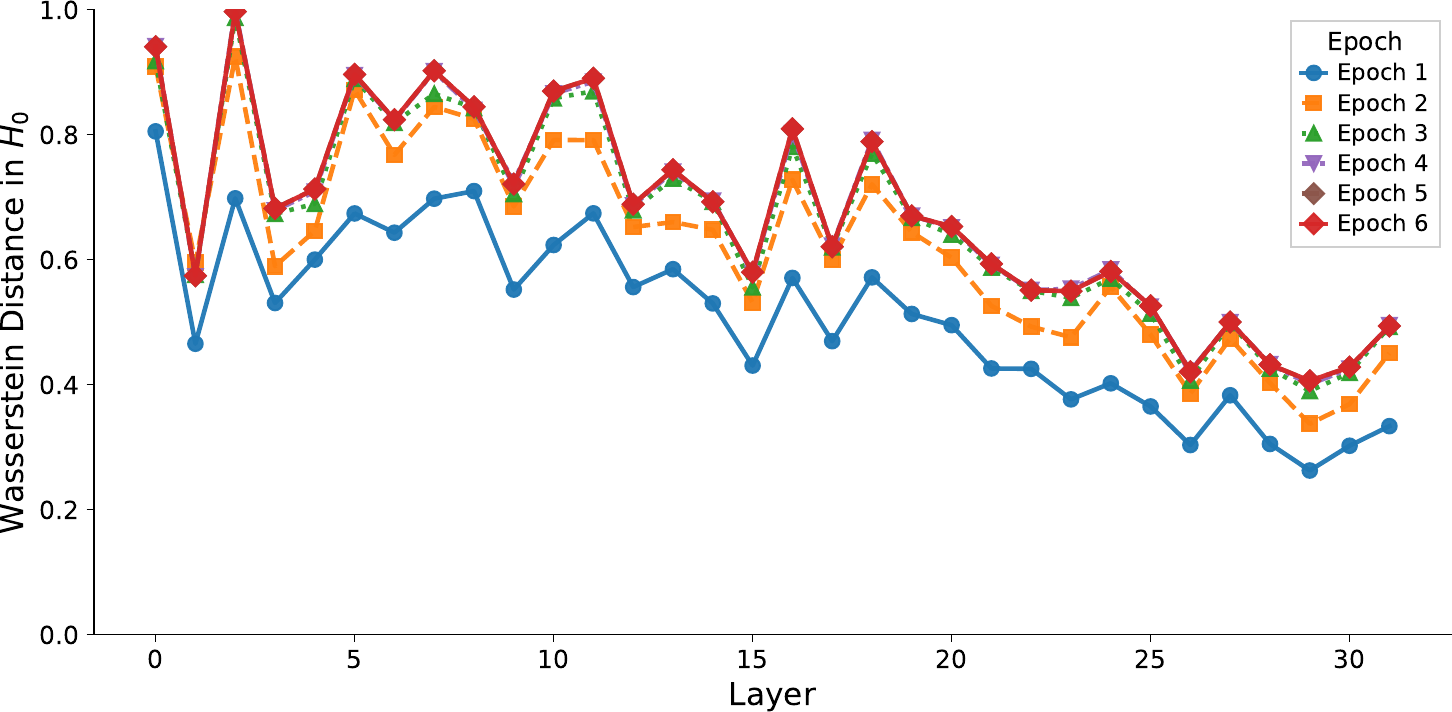} &
\includegraphics[width=0.32\textwidth]{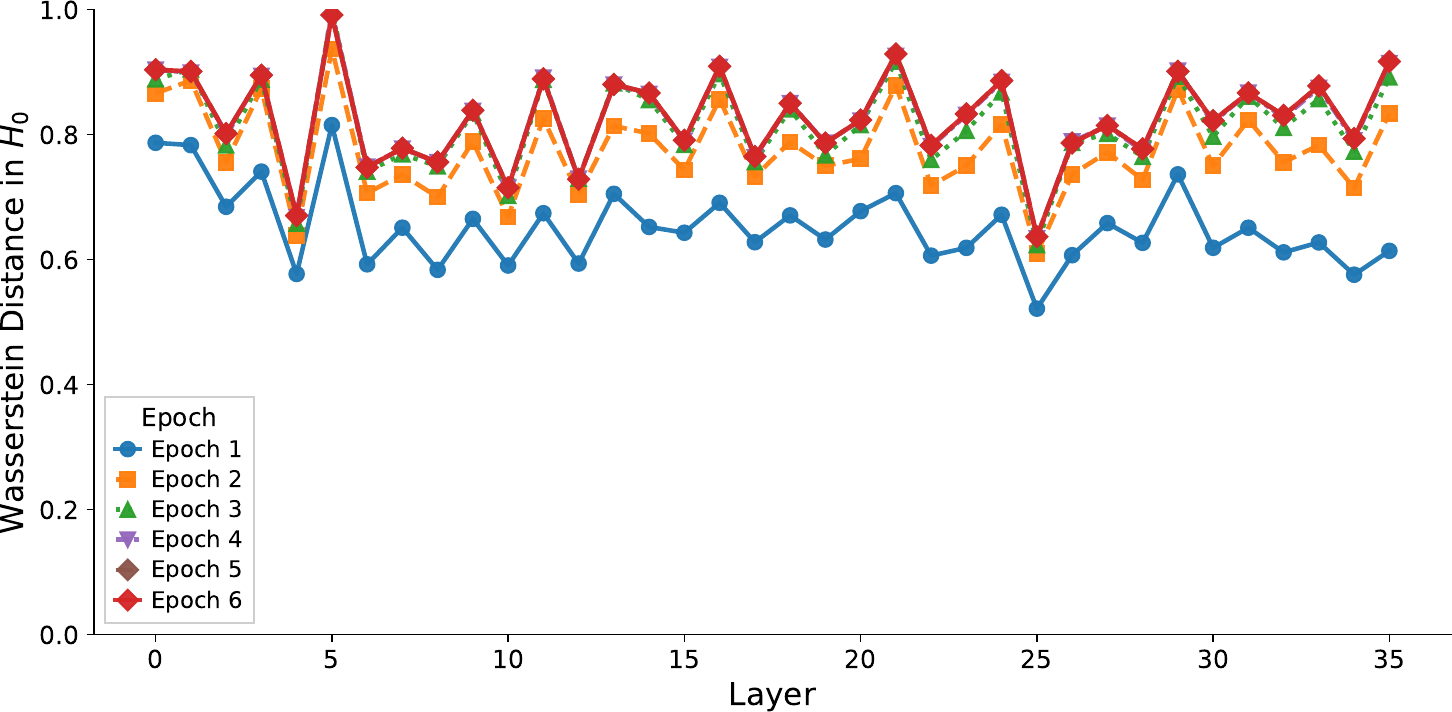} &
\includegraphics[width=0.32\textwidth]{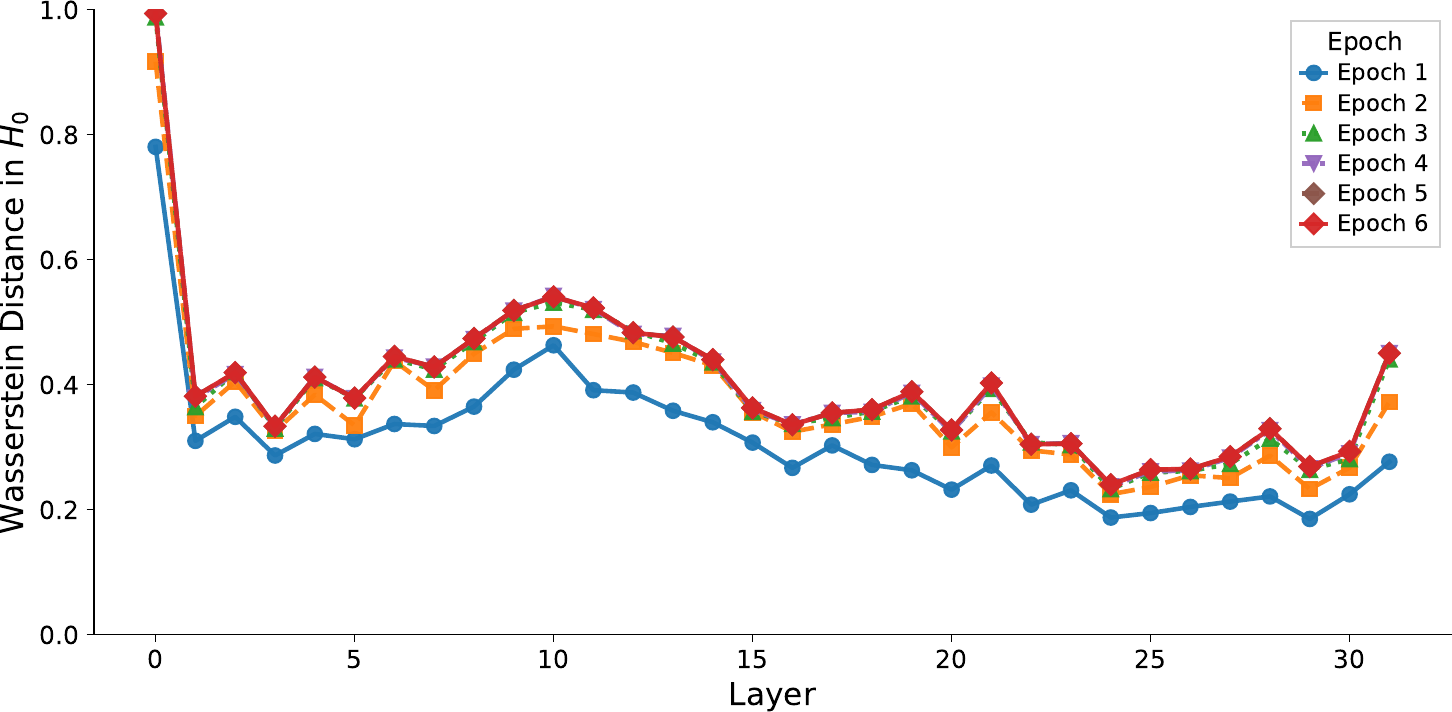} \\

\small \texttt{LLaMA-3.1-8B} & \small \texttt{Qwen3-8B-Base} & \small \texttt{Mistral-7B-v0.3}
\end{tabular}
\caption{Topological drift of the $V$ and $O$ projection on \texttt{QA:GSM8K} under $H_0$. The top row shows epoch-to-epoch drift bars, and the bottom row shows Wasserstein distance curves across layers.}
\label{fig:topodrift-gsm8k-v-h0}
\end{figure*}

\begin{figure*}[!htbp]
\centering
\scriptsize
\setlength{\tabcolsep}{2pt}
\begin{tabular}{ccc}

\multicolumn{3}{c}{\small \textit{\(V\) Projection}} \\
\includegraphics[width=0.32\textwidth]{figs/figs_wass_full_by_task/topological_drift/llama31_8b_mmlu_v_h0_driftbars.pdf} &
\includegraphics[width=0.32\textwidth]{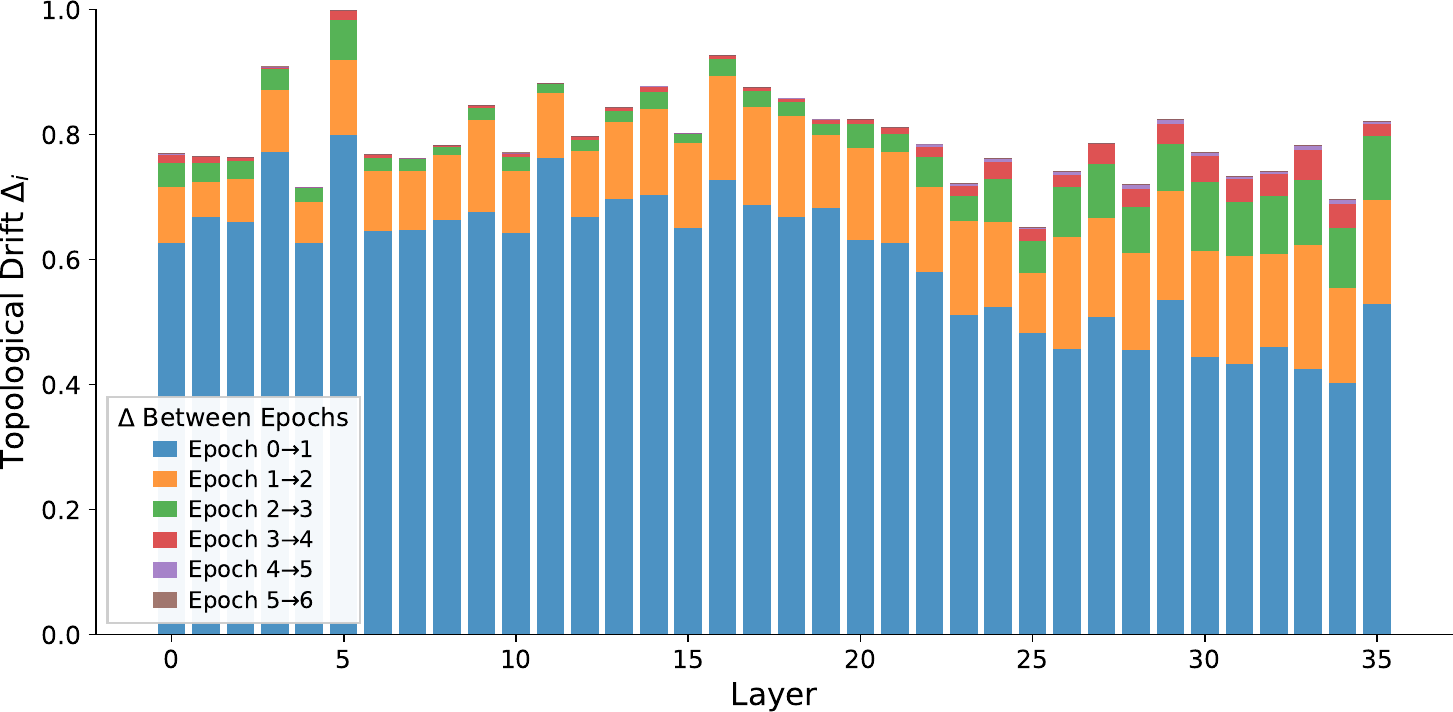} &
\includegraphics[width=0.32\textwidth]{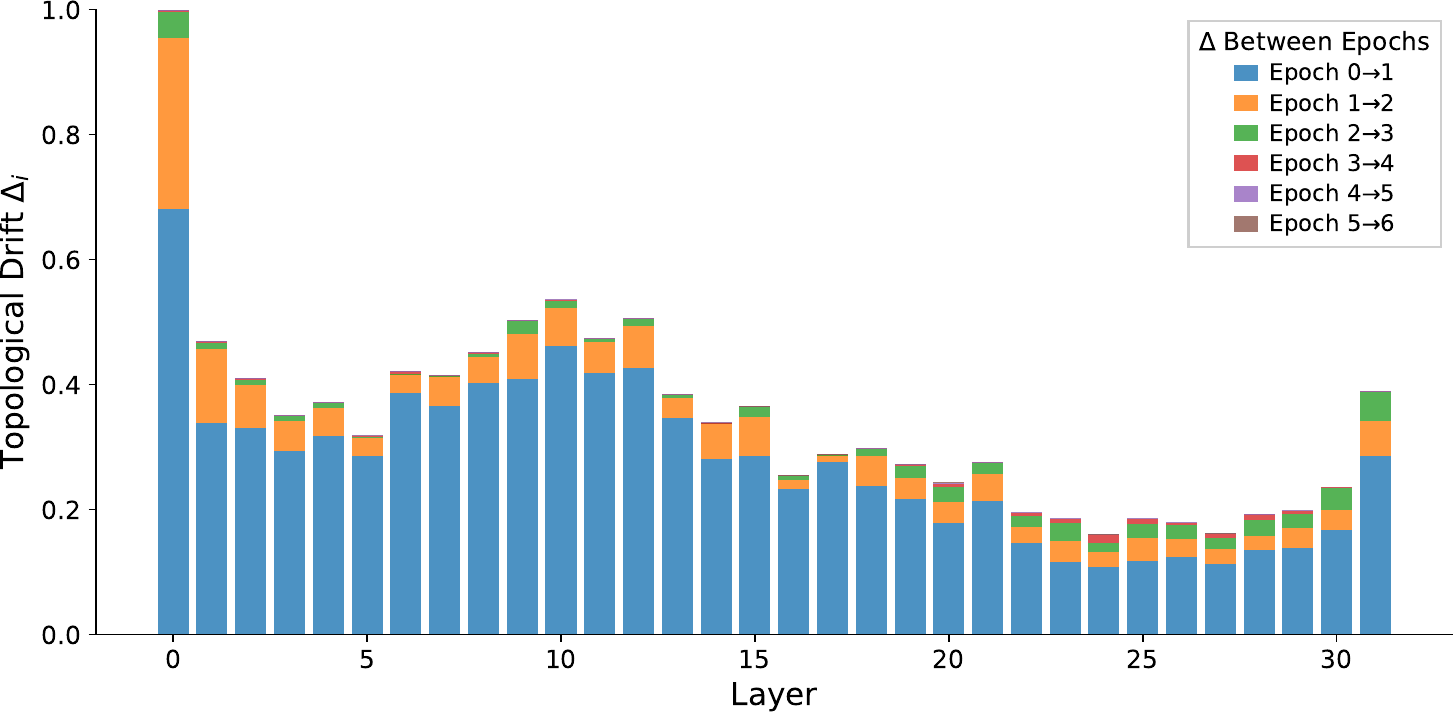} \\
[1em]
\includegraphics[width=0.32\textwidth]{figs/figs_wass_full_by_task/topological_drift/llama31_8b_mmlu_v_h0_lines.pdf} &
\includegraphics[width=0.32\textwidth]{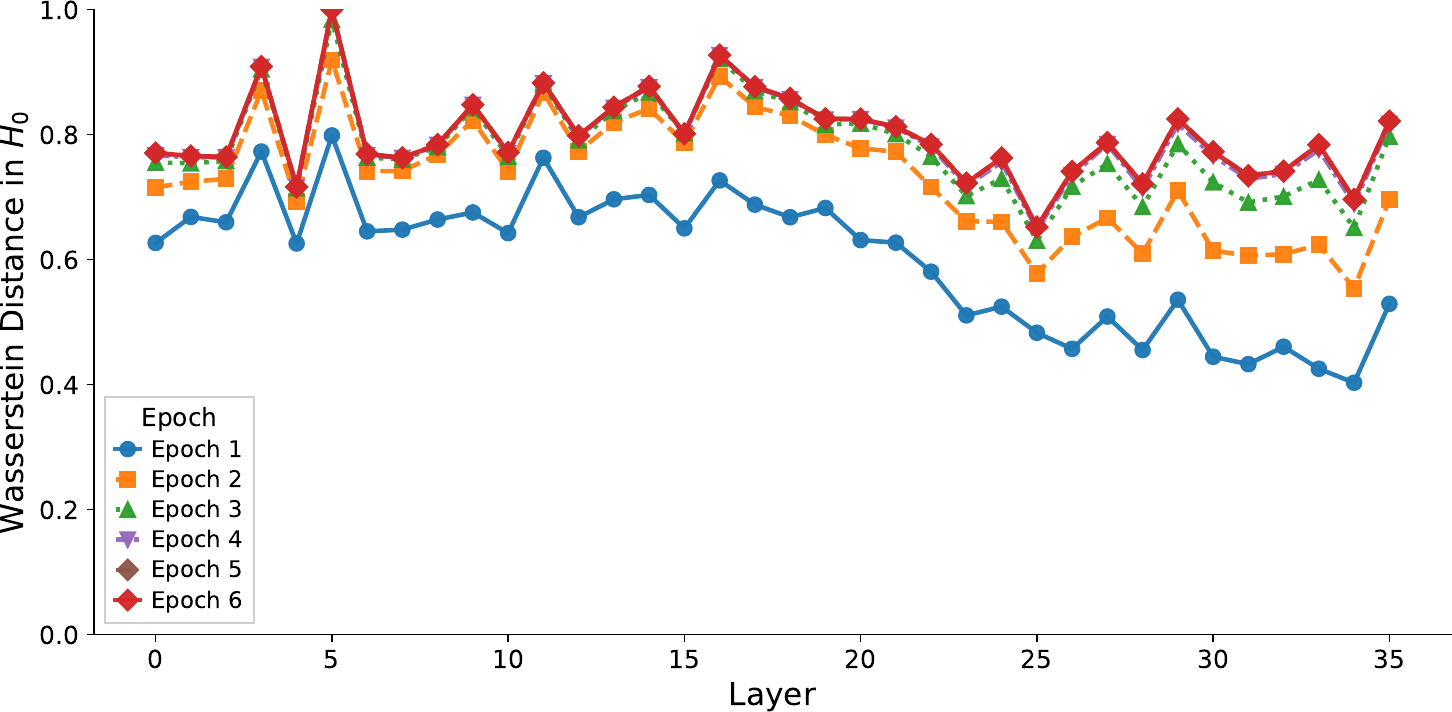} &
\includegraphics[width=0.32\textwidth]{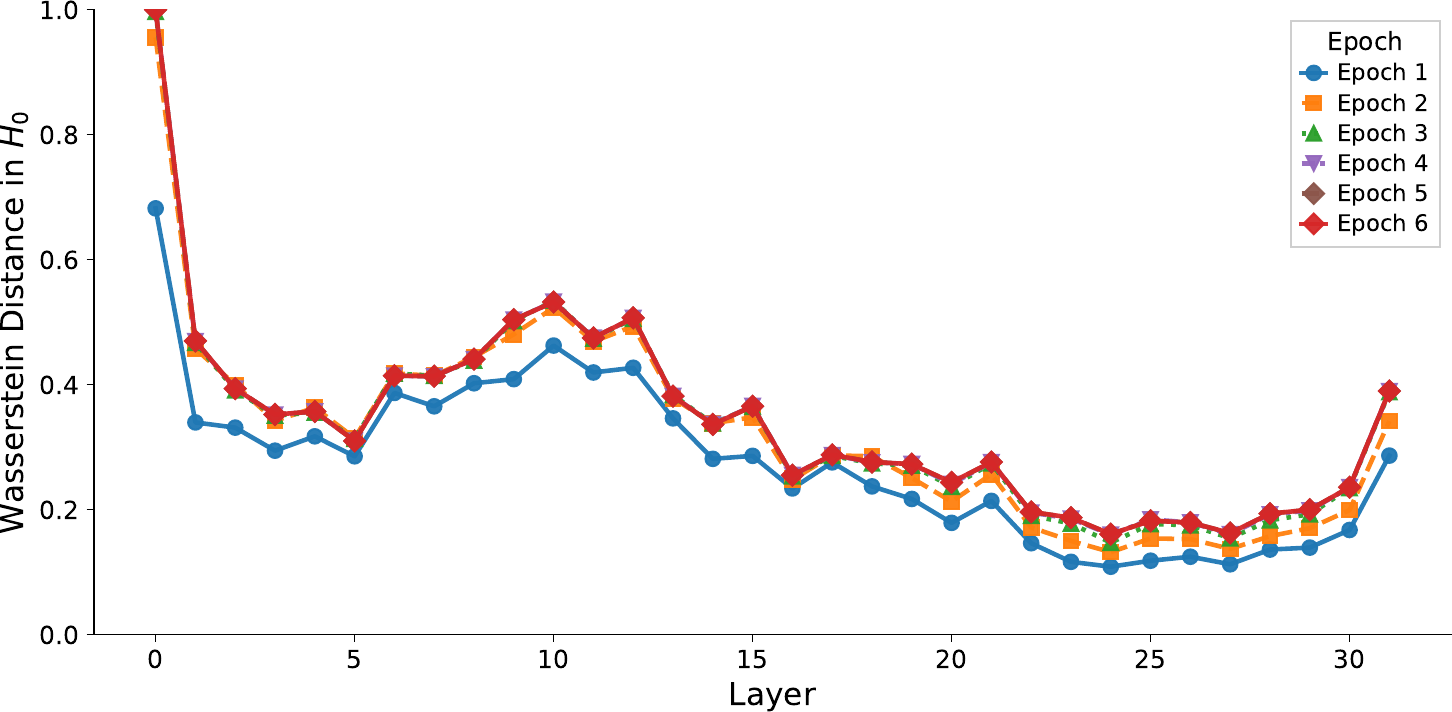} \\
[2em]
\multicolumn{3}{c}{\small \textit{\(O\) Projection}} \\
\includegraphics[width=0.32\textwidth]{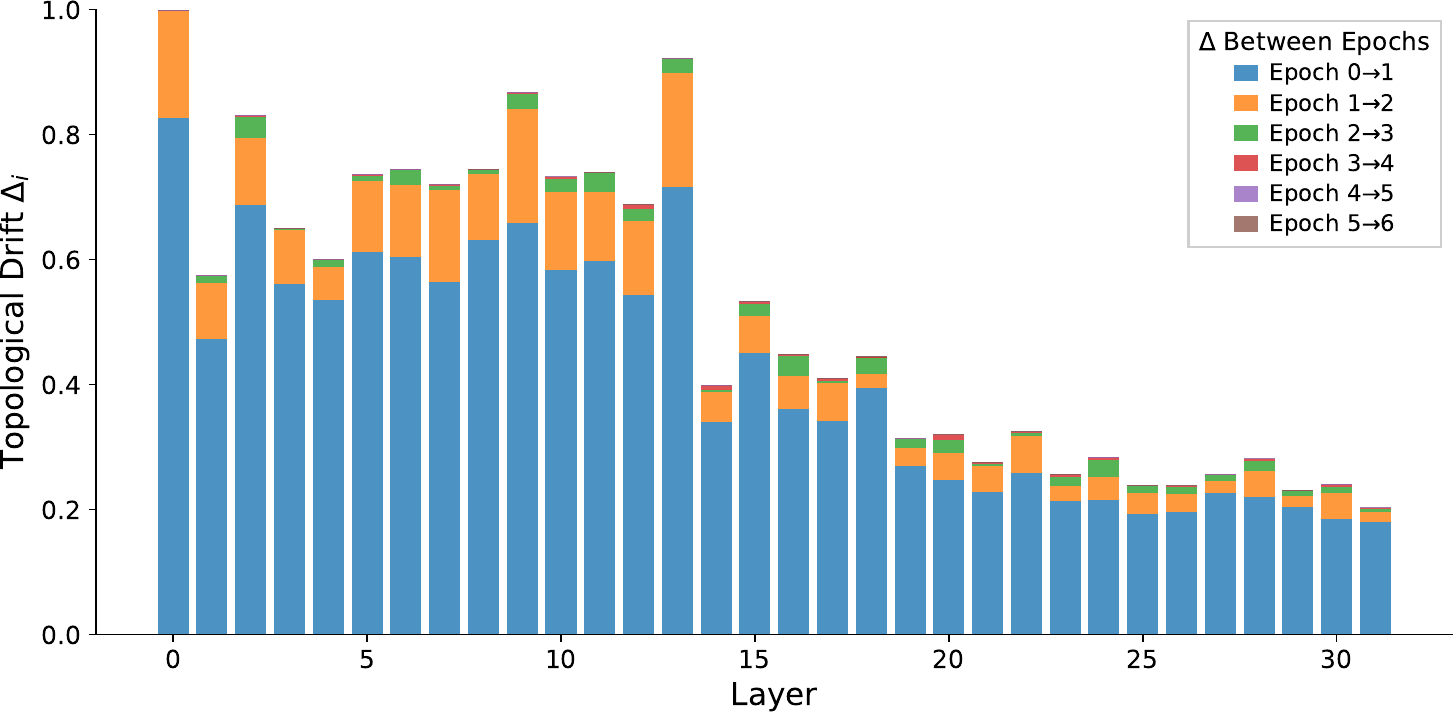} &
\includegraphics[width=0.32\textwidth]{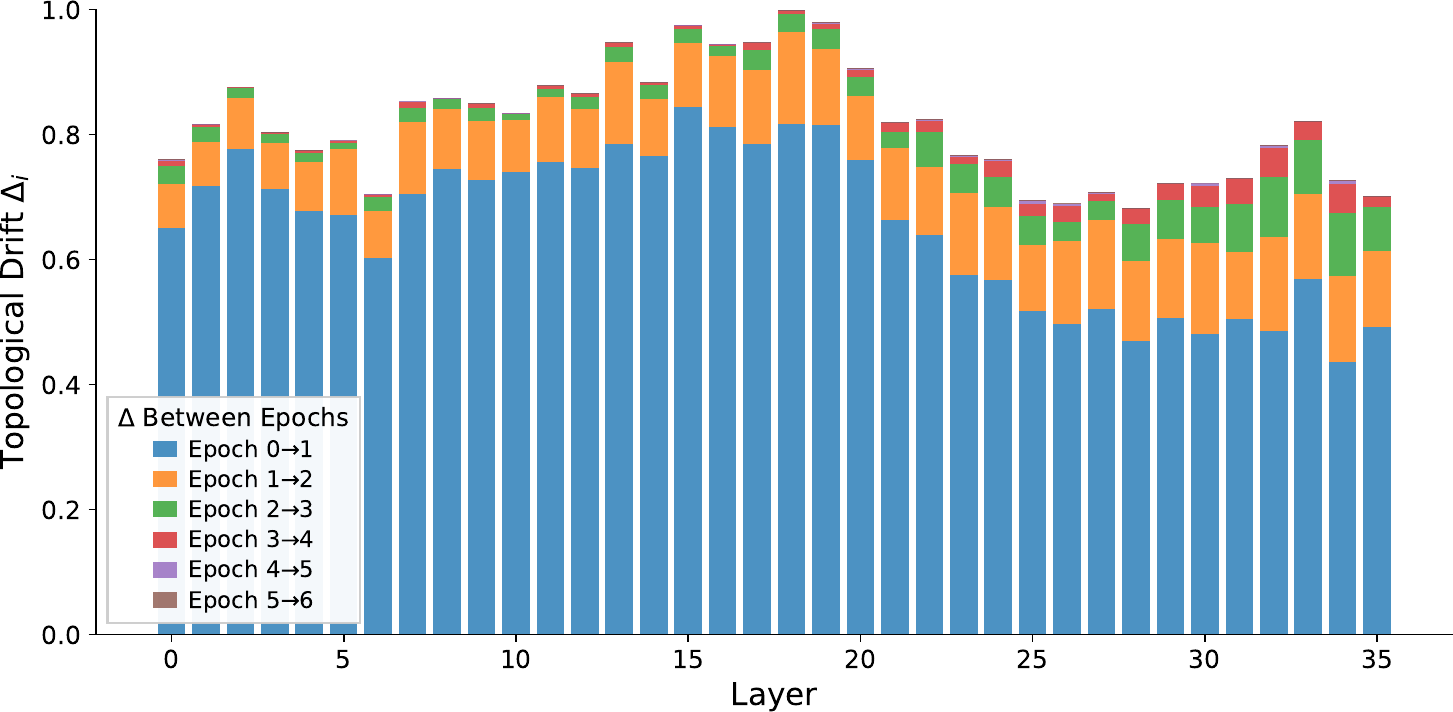} &
\includegraphics[width=0.32\textwidth]{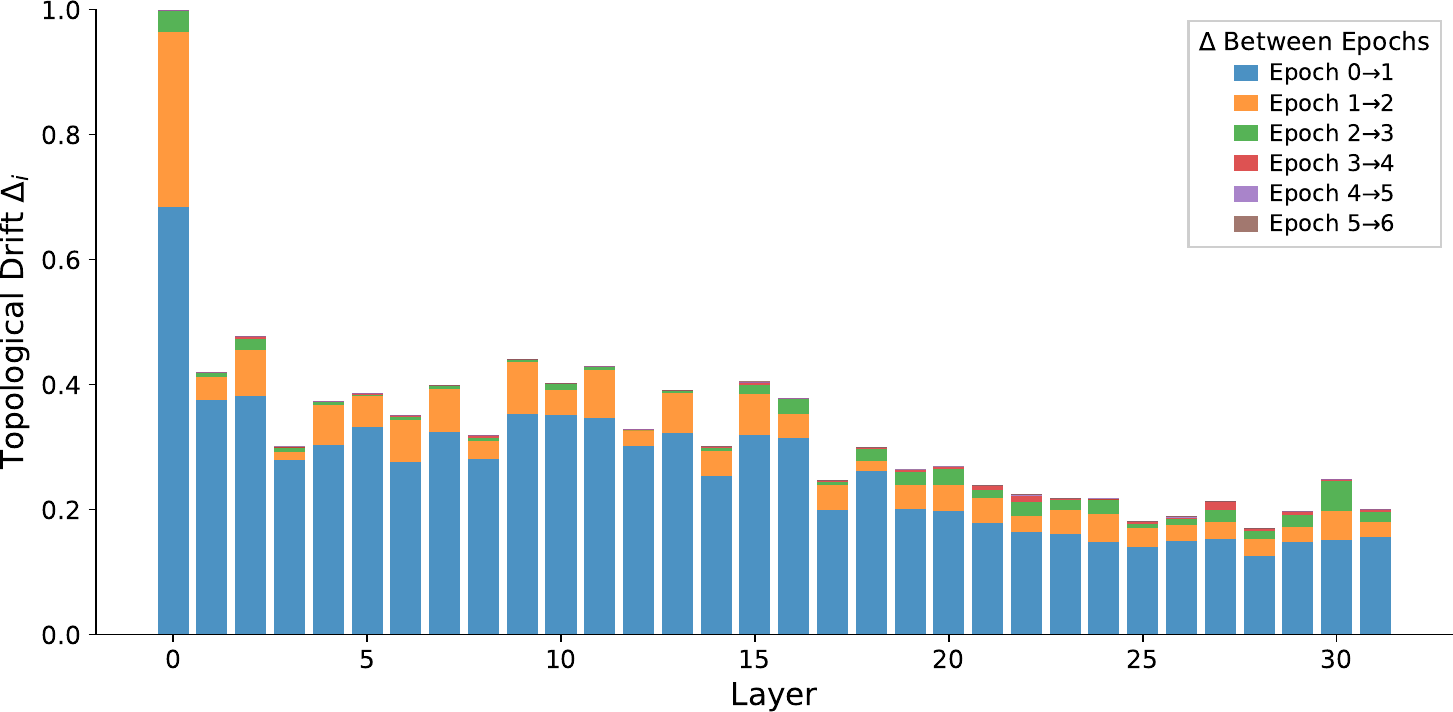} \\
[1em]
\includegraphics[width=0.32\textwidth]{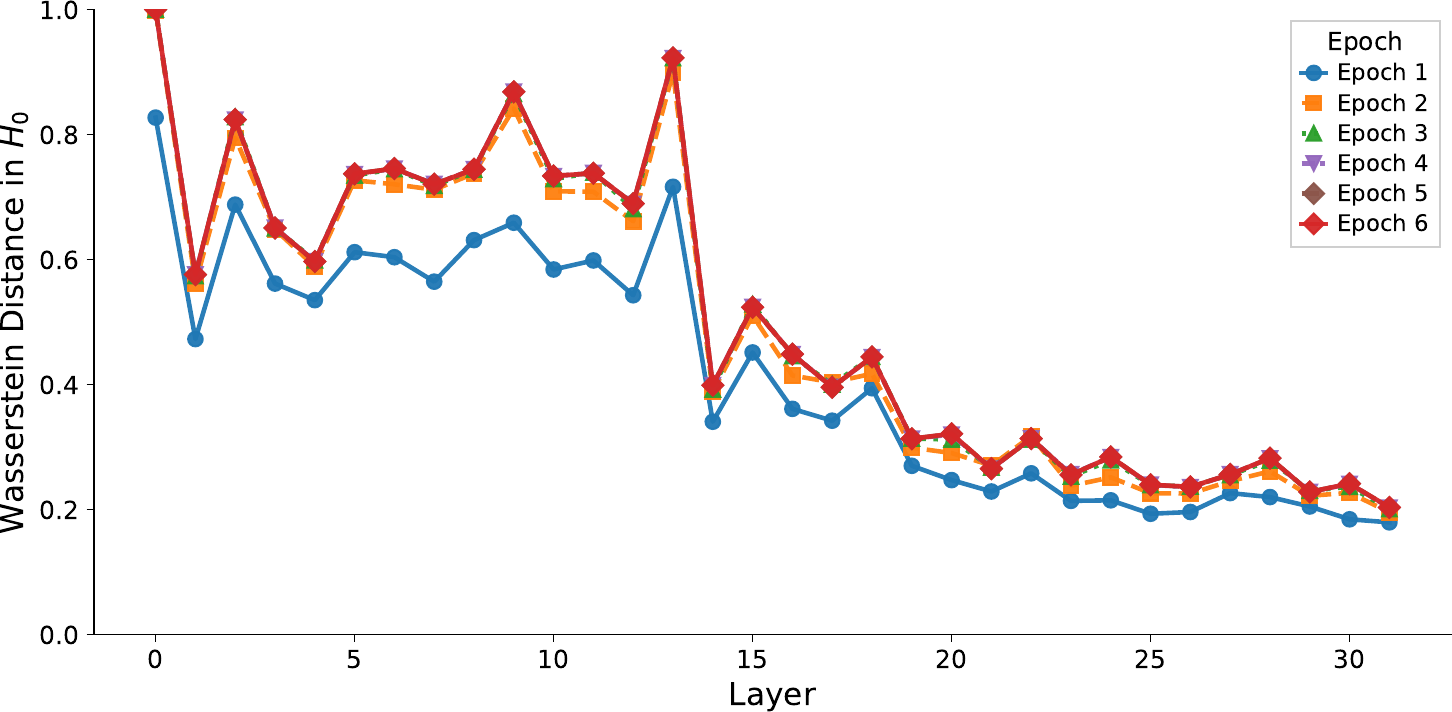} &
\includegraphics[width=0.32\textwidth]{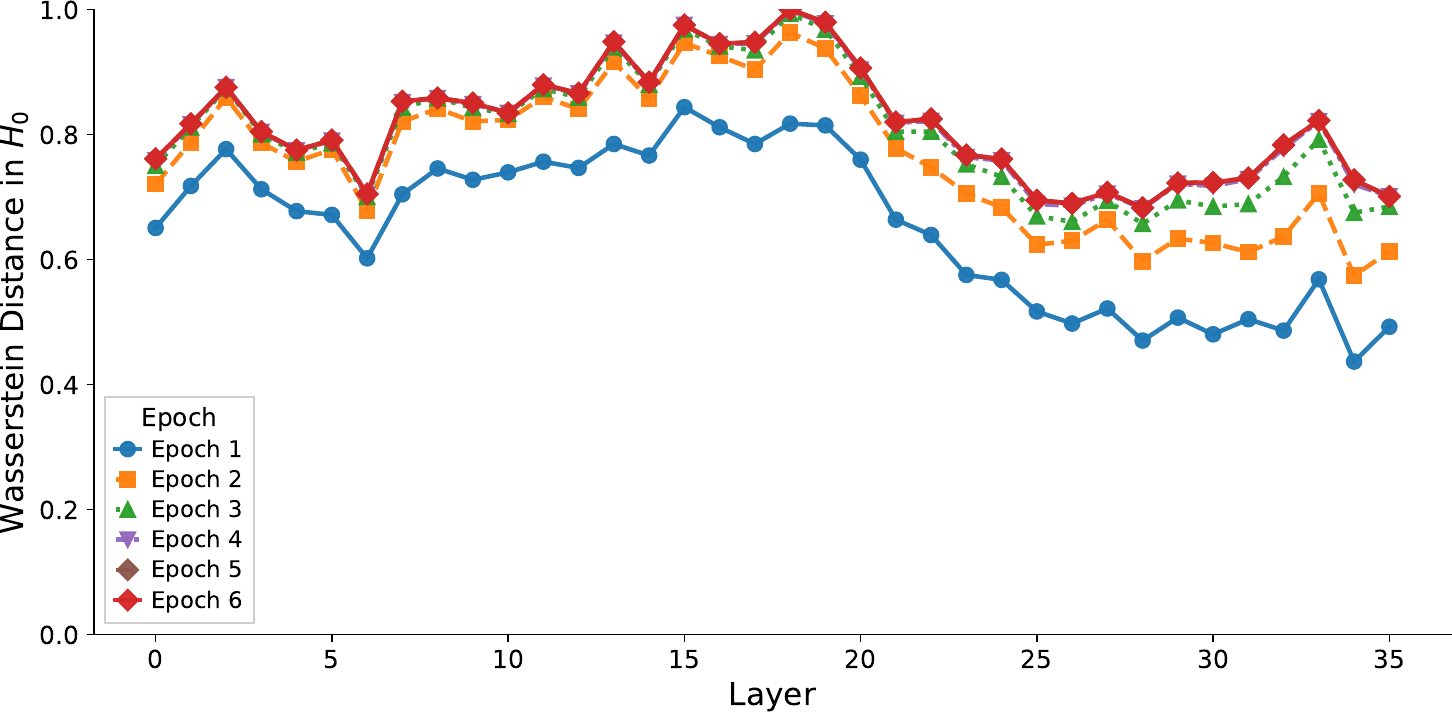} &
\includegraphics[width=0.32\textwidth]{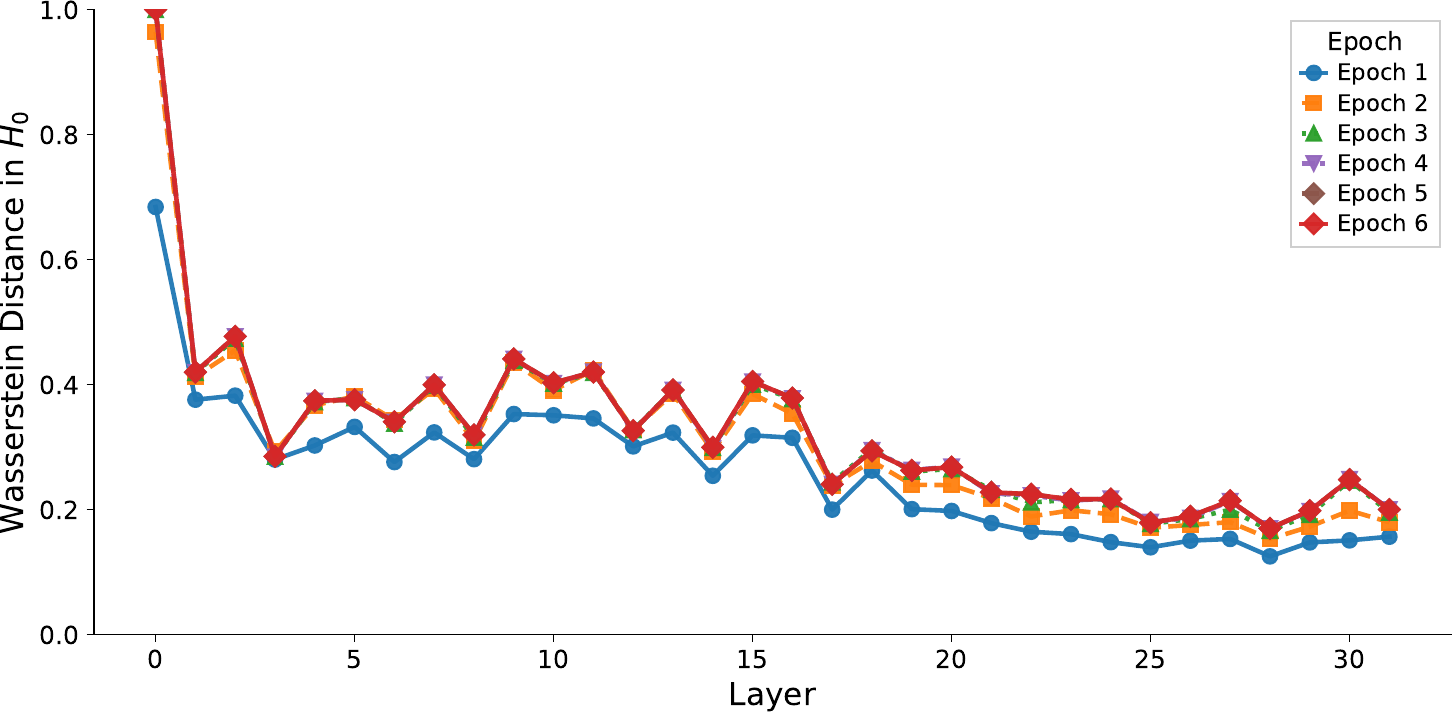} \\
\small \texttt{LLaMA-3.1-8B} & \small \texttt{Qwen3-8B-Base} & \small \texttt{Mistral-7B-v0.3} \\
\end{tabular}
\caption{Topological drift of the $V$ and $O$ projection on \texttt{QA:MMLU} under $H_0$. The top row shows epoch-to-epoch drift bars, and the bottom row shows Wasserstein distance curves across layers.}
\label{fig:topodrift-mmlu-v-h0}
\end{figure*}

\begin{figure*}[!htbp]
\centering
\scriptsize
\setlength{\tabcolsep}{2pt}
\begin{tabular}{ccc}
\multicolumn{3}{c}{\small \textit{\(V\) Projection}} \\
\includegraphics[width=0.32\textwidth]{figs/figs_wass_full_by_task/topological_drift/llama31_8b_imdb_v_h0_driftbars.pdf} &
\includegraphics[width=0.32\textwidth]{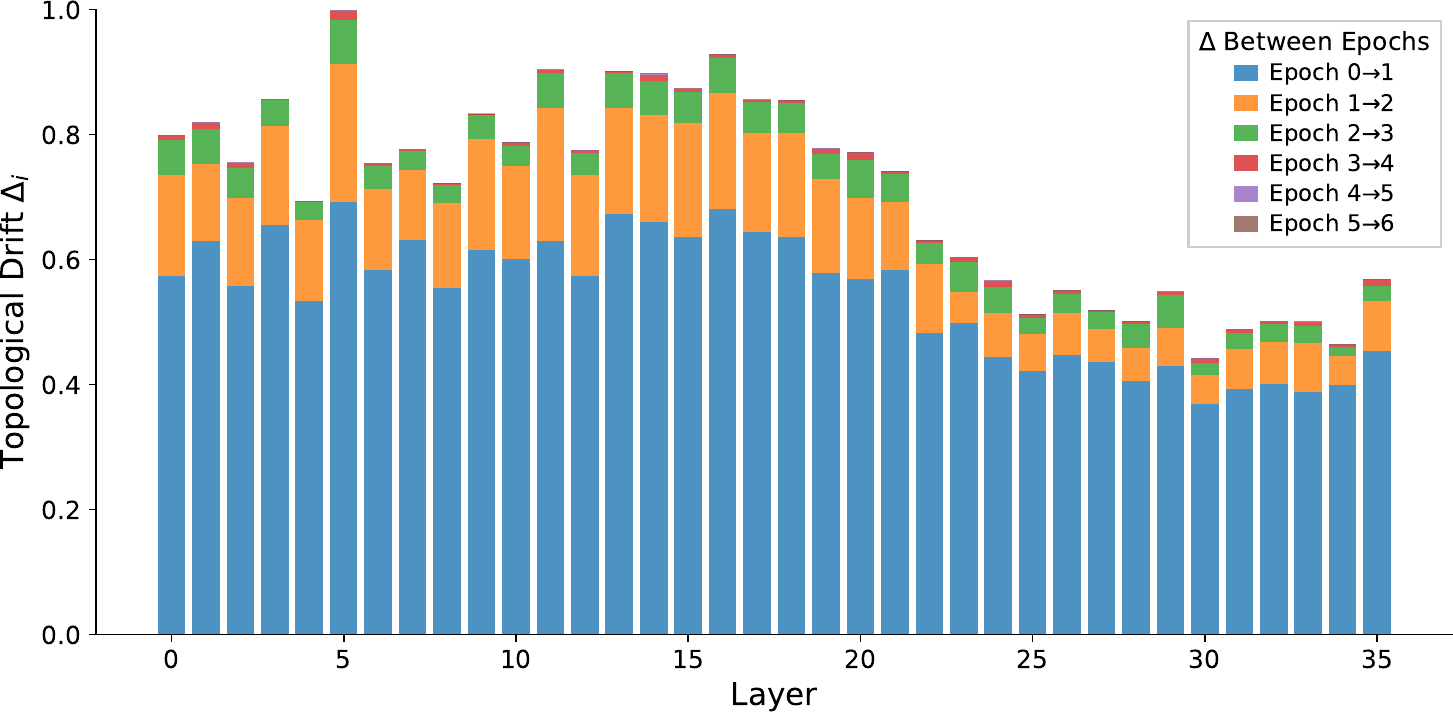} &
\includegraphics[width=0.32\textwidth]{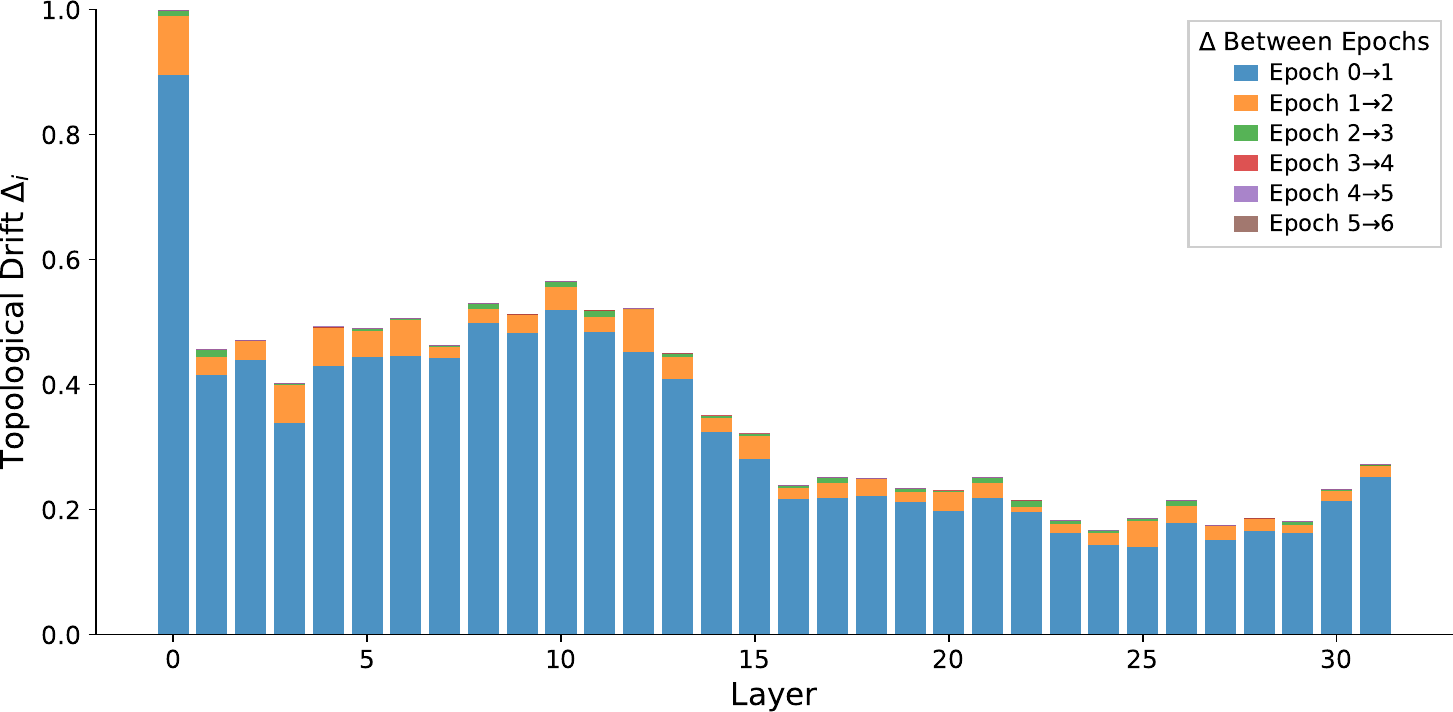} \\
[1em]
\includegraphics[width=0.32\textwidth]{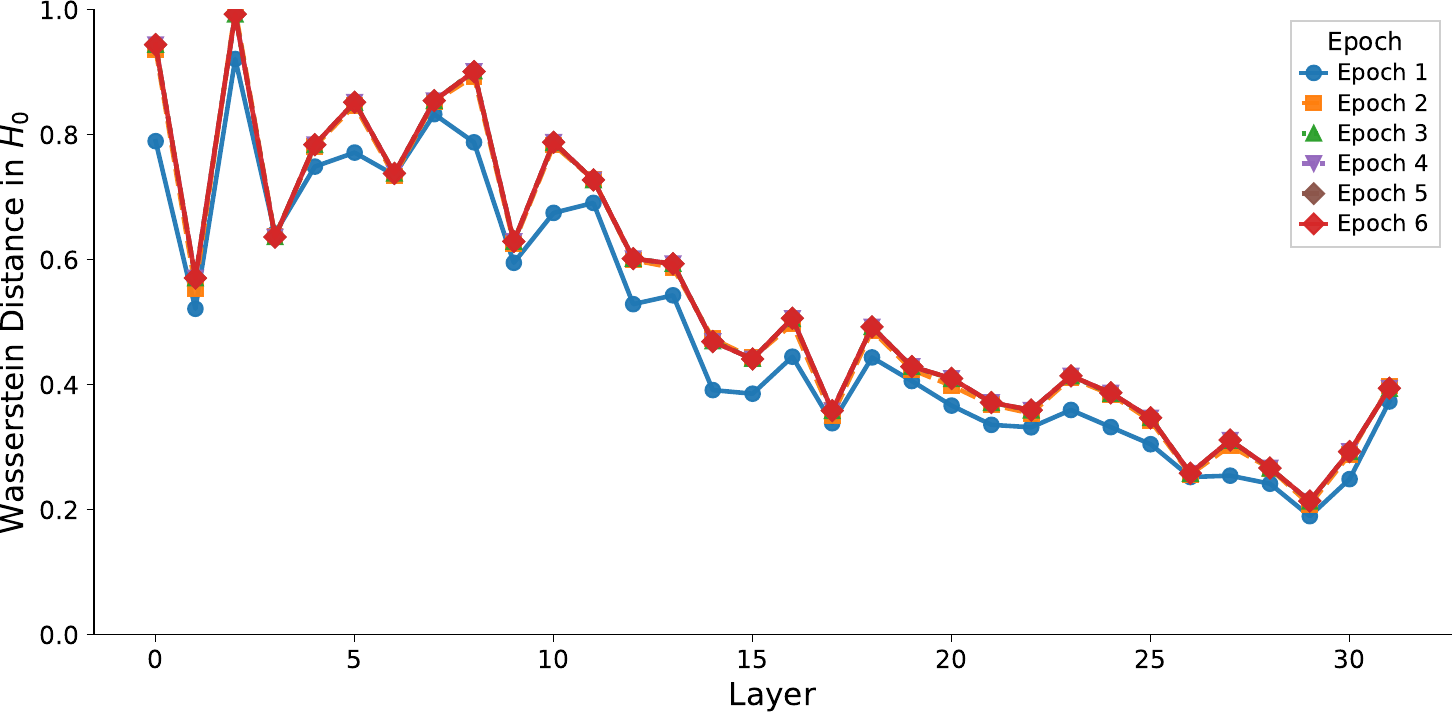} &
\includegraphics[width=0.32\textwidth]{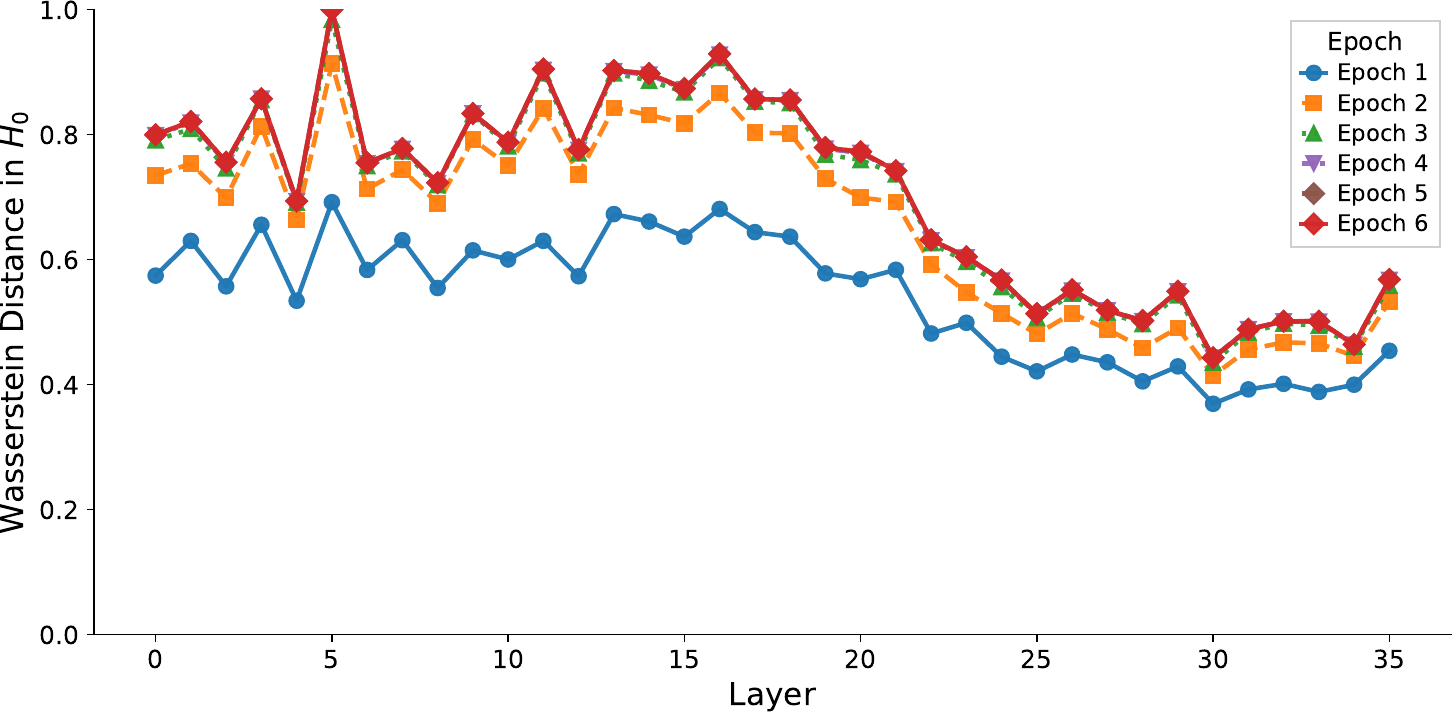} &
\includegraphics[width=0.32\textwidth]{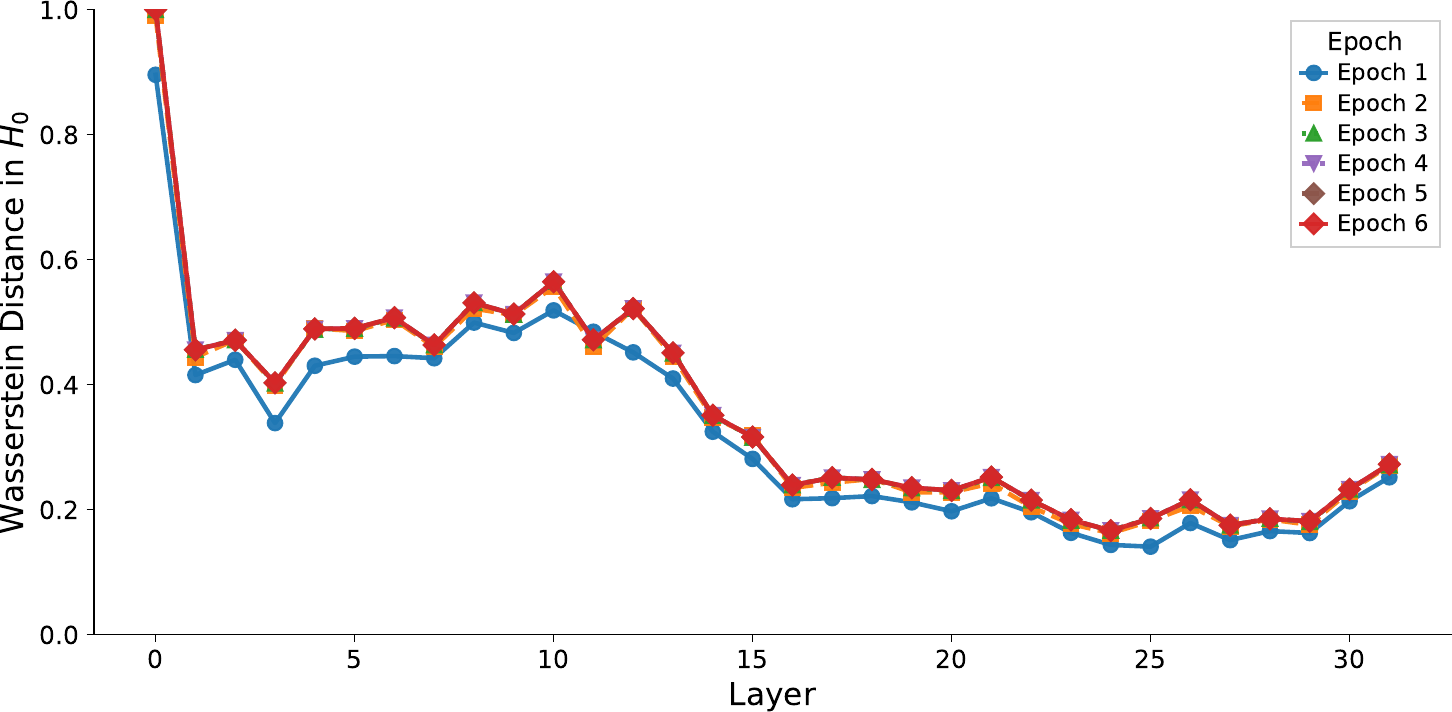} \\
[2em]
\multicolumn{3}{c}{\small \textit{\(O\) Projection}} \\
\includegraphics[width=0.32\textwidth]{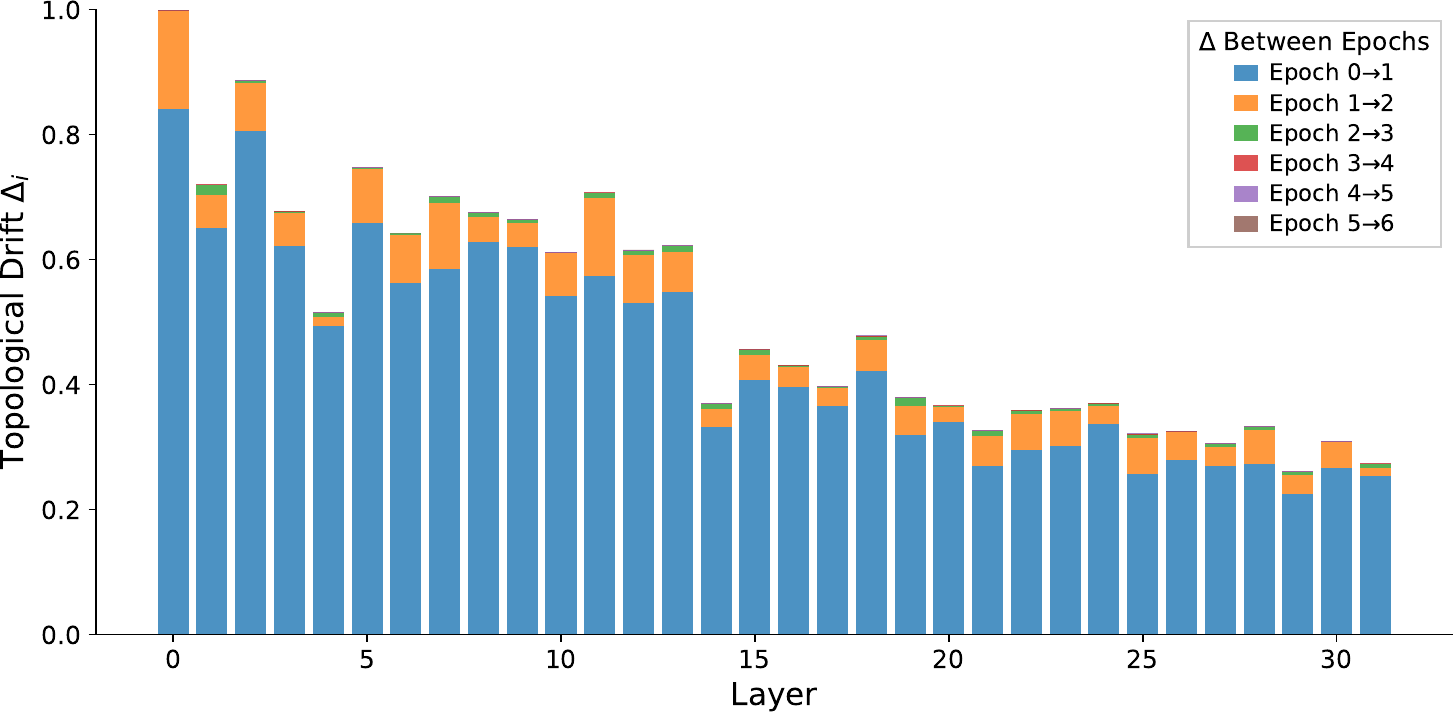} &
\includegraphics[width=0.32\textwidth]{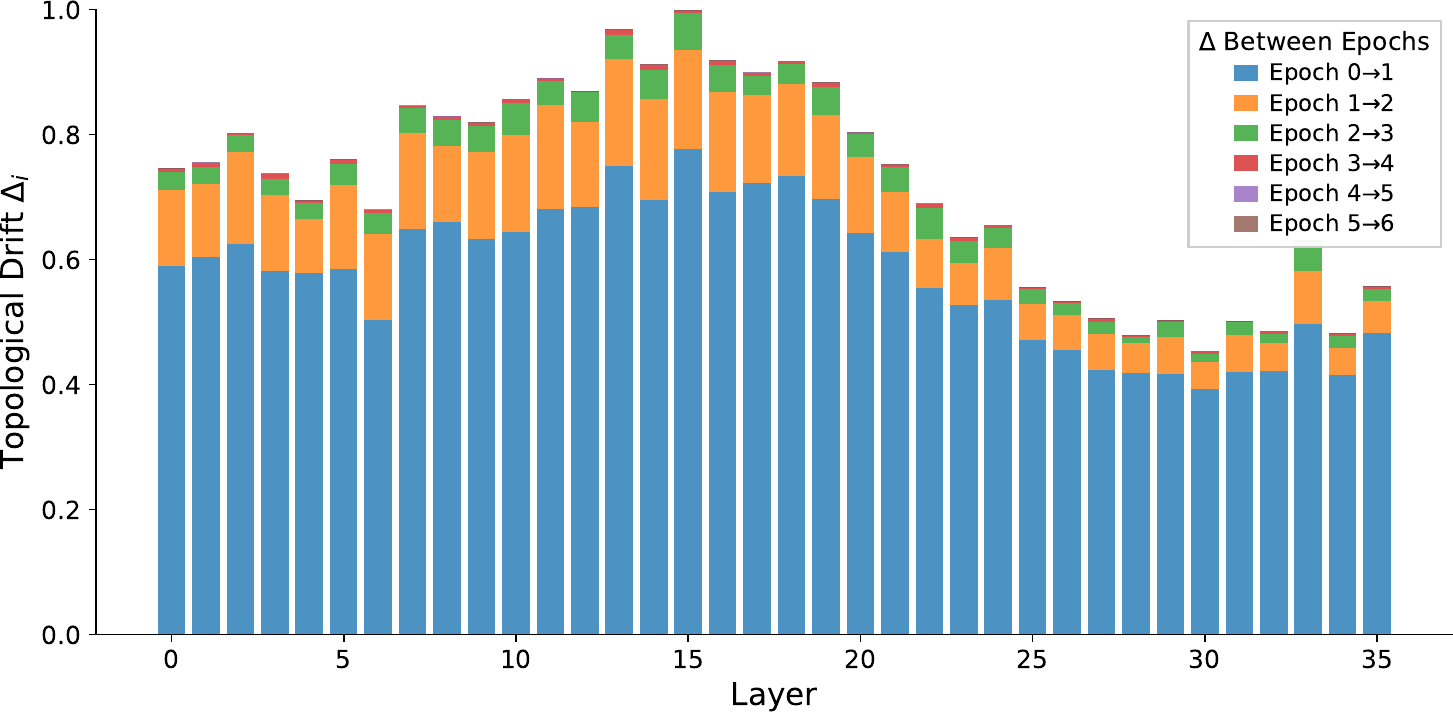} &
\includegraphics[width=0.32\textwidth]{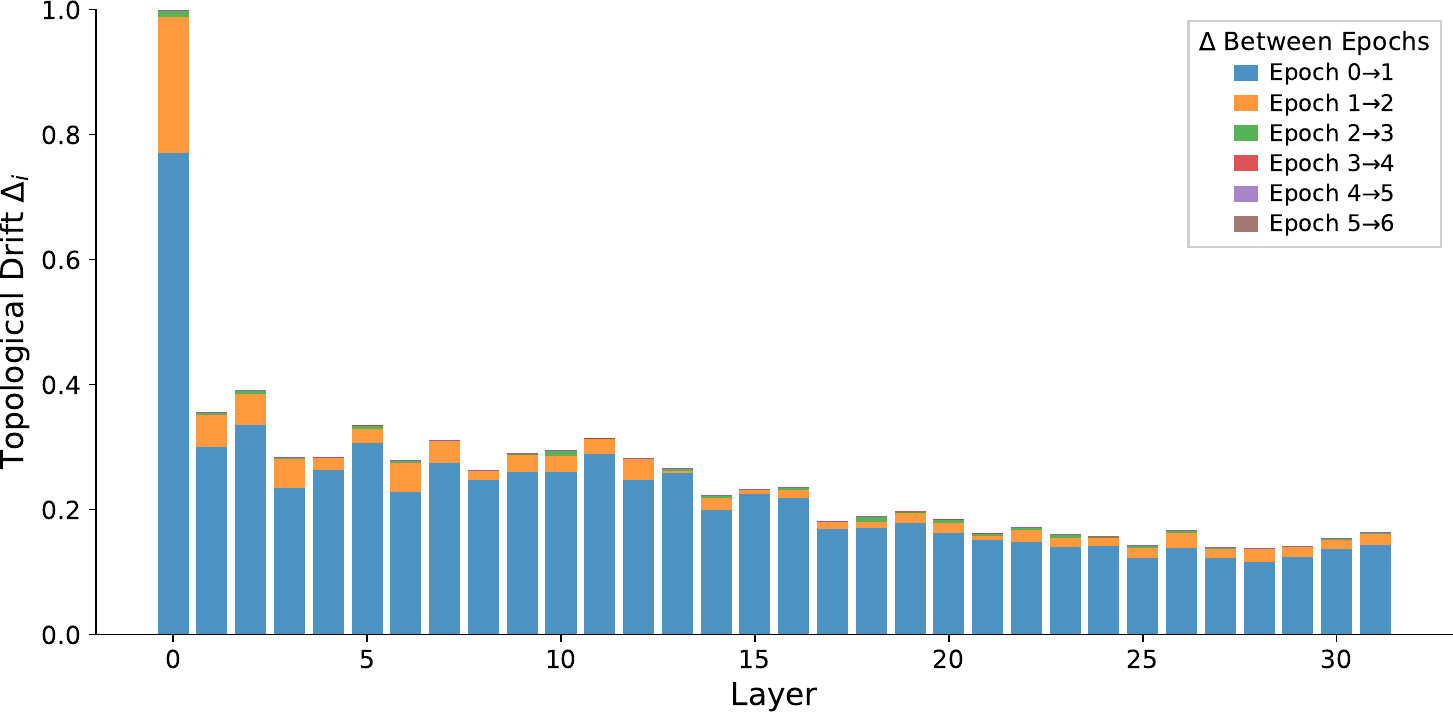} \\
[1em]
\includegraphics[width=0.32\textwidth]{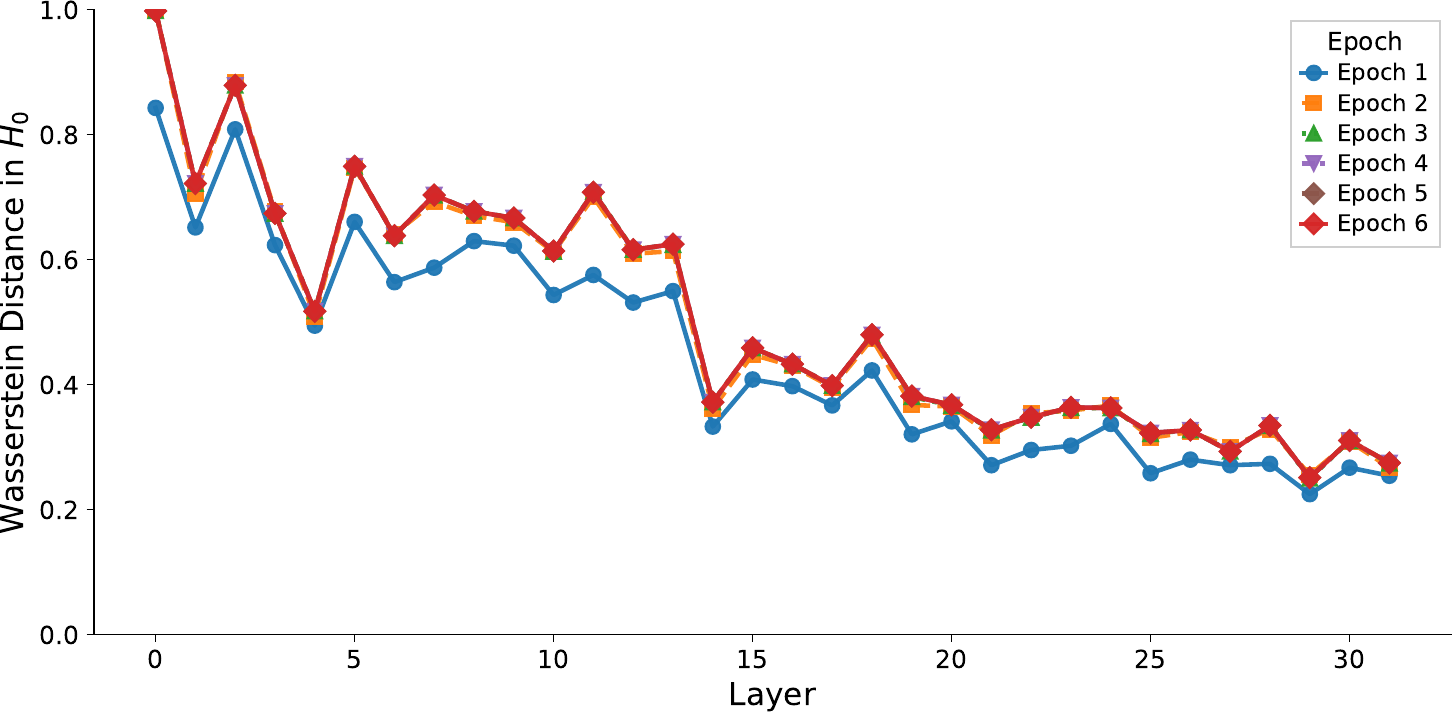} &
\includegraphics[width=0.32\textwidth]{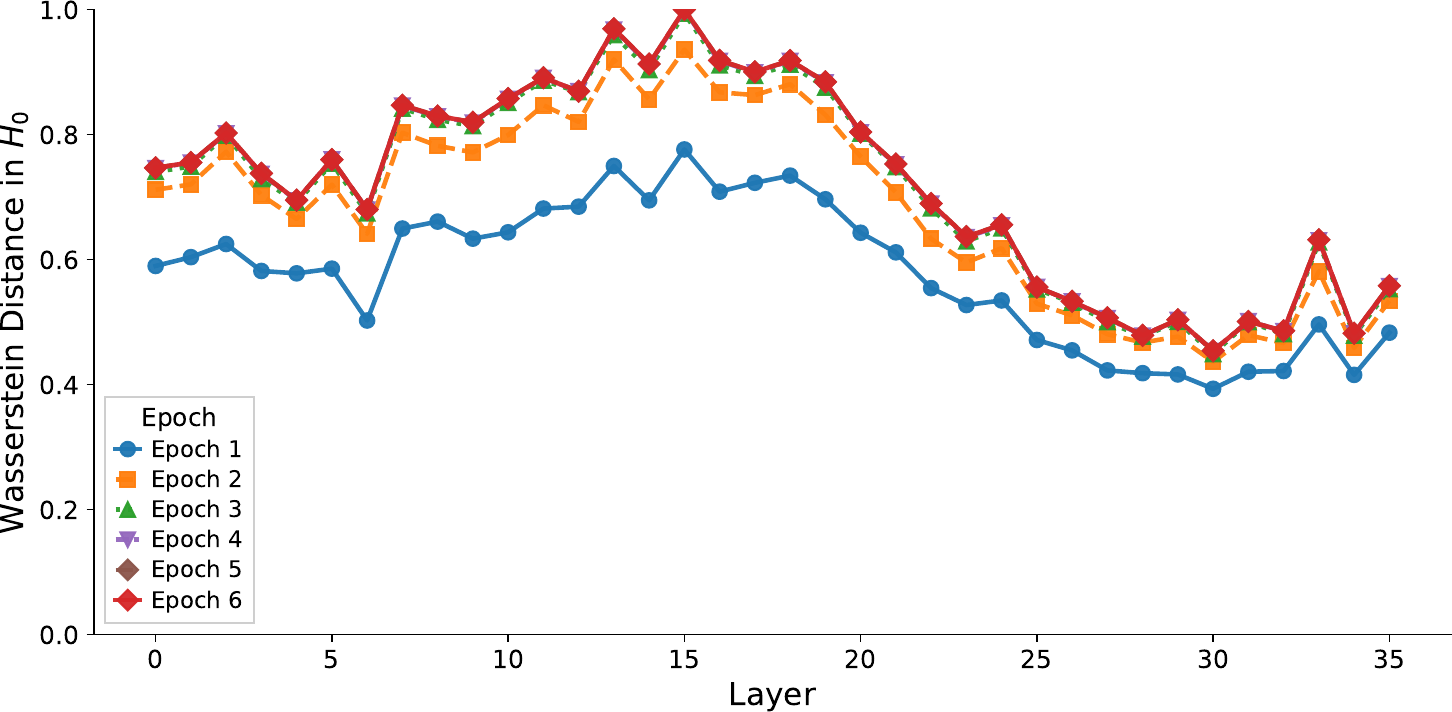} &
\includegraphics[width=0.32\textwidth]{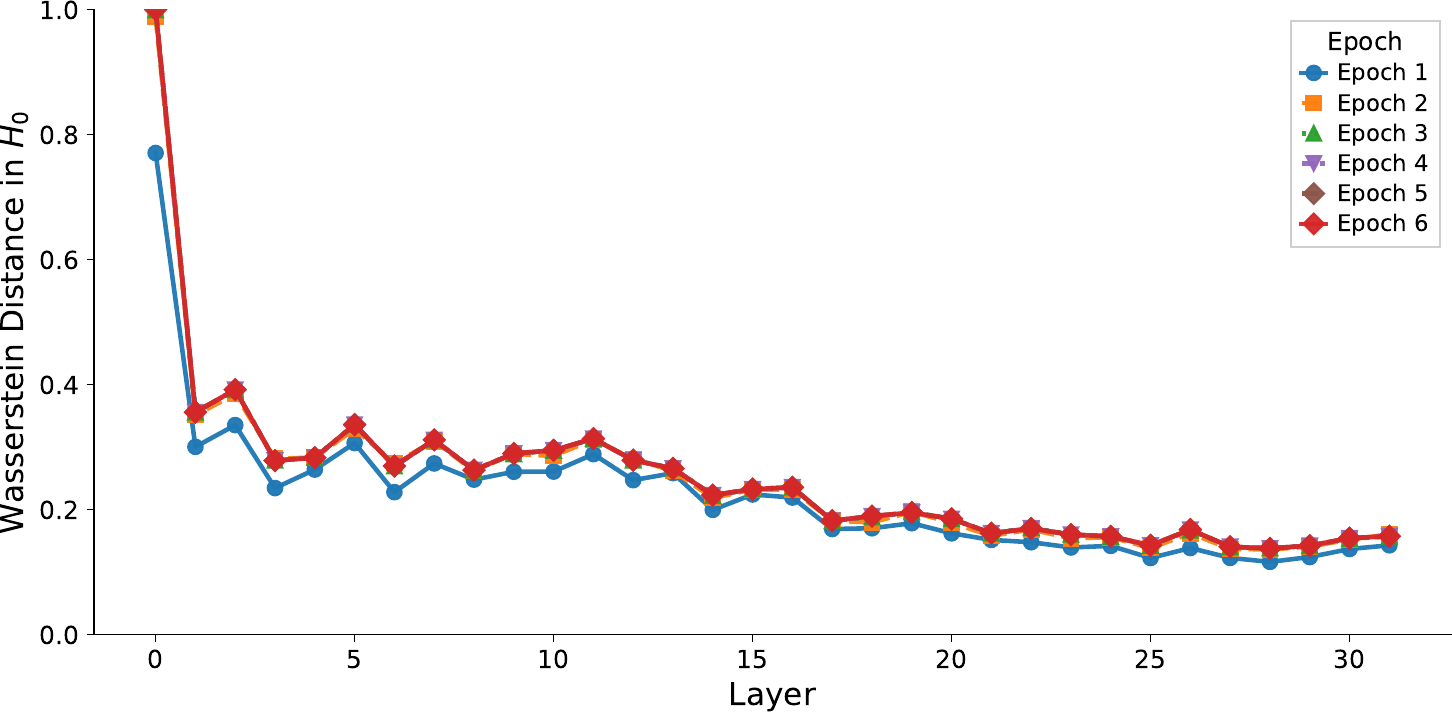} \\

\small \texttt{LLaMA-3.1-8B} & \small \texttt{Qwen3-8B-Base} & \small \texttt{Mistral-7B-v0.3} \\
\end{tabular}
\caption{Topological drift of the $V$ projection on \texttt{SA:IMDB} under $H_0$. The top row shows epoch-to-epoch drift bars, and the bottom row shows Wasserstein distance curves across layers.}
\label{fig:topodrift-imdb-v-h0}
\end{figure*}

\begin{figure*}[!htbp]
\centering
\scriptsize
\setlength{\tabcolsep}{2pt}
\begin{tabular}{ccc}
\multicolumn{3}{c}{\small \textit{\(V\) Projection}} \\
\includegraphics[width=0.32\textwidth]{figs/figs_wass_full_by_task/topological_drift/llama31_8b_sst2_v_h0_driftbars.pdf} &
\includegraphics[width=0.32\textwidth]{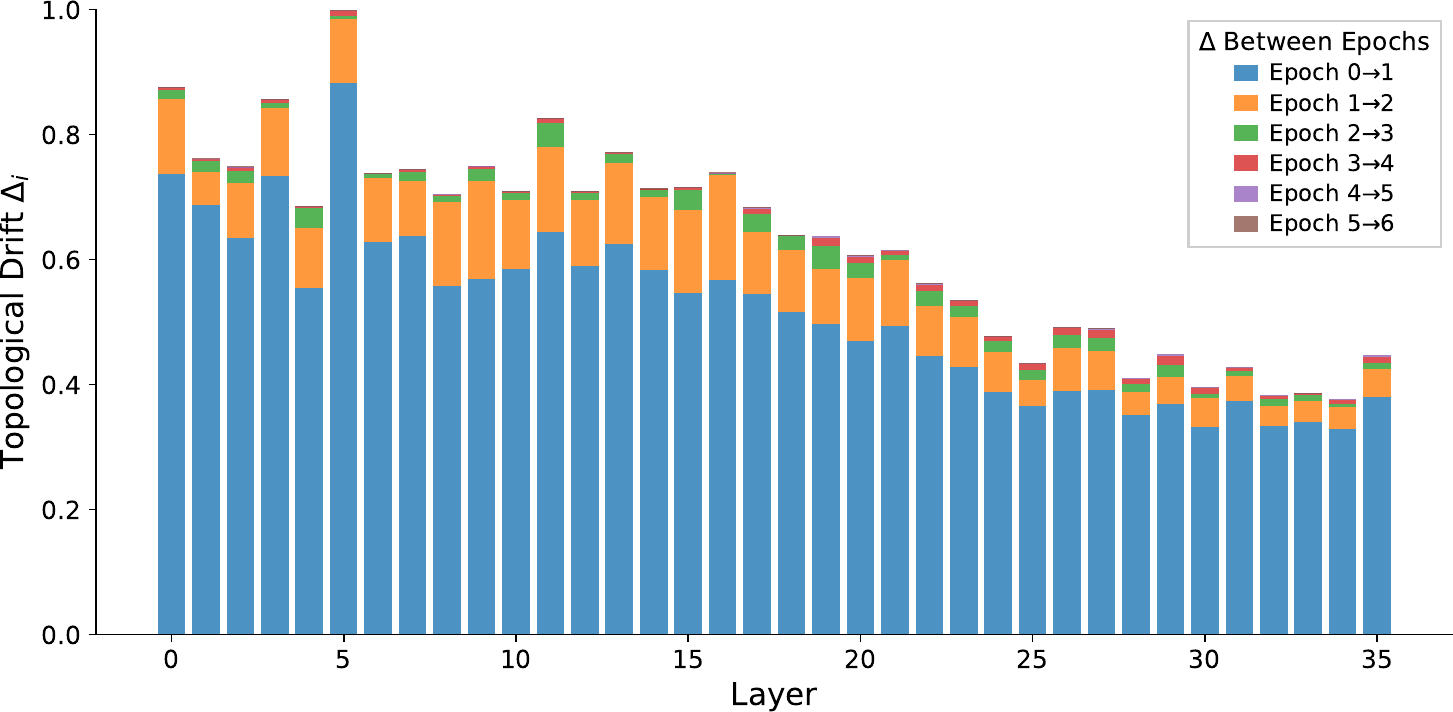} &
\includegraphics[width=0.32\textwidth]{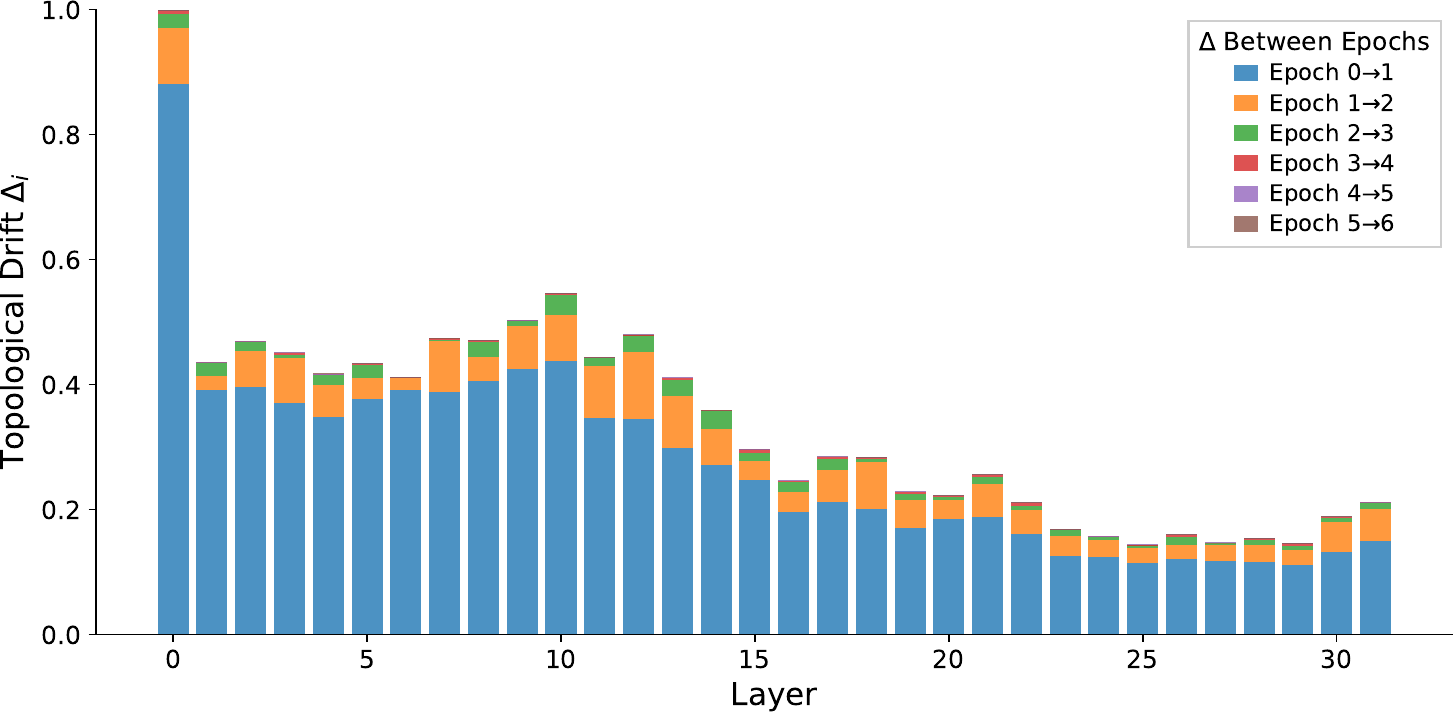} \\
[1em]
\includegraphics[width=0.32\textwidth]{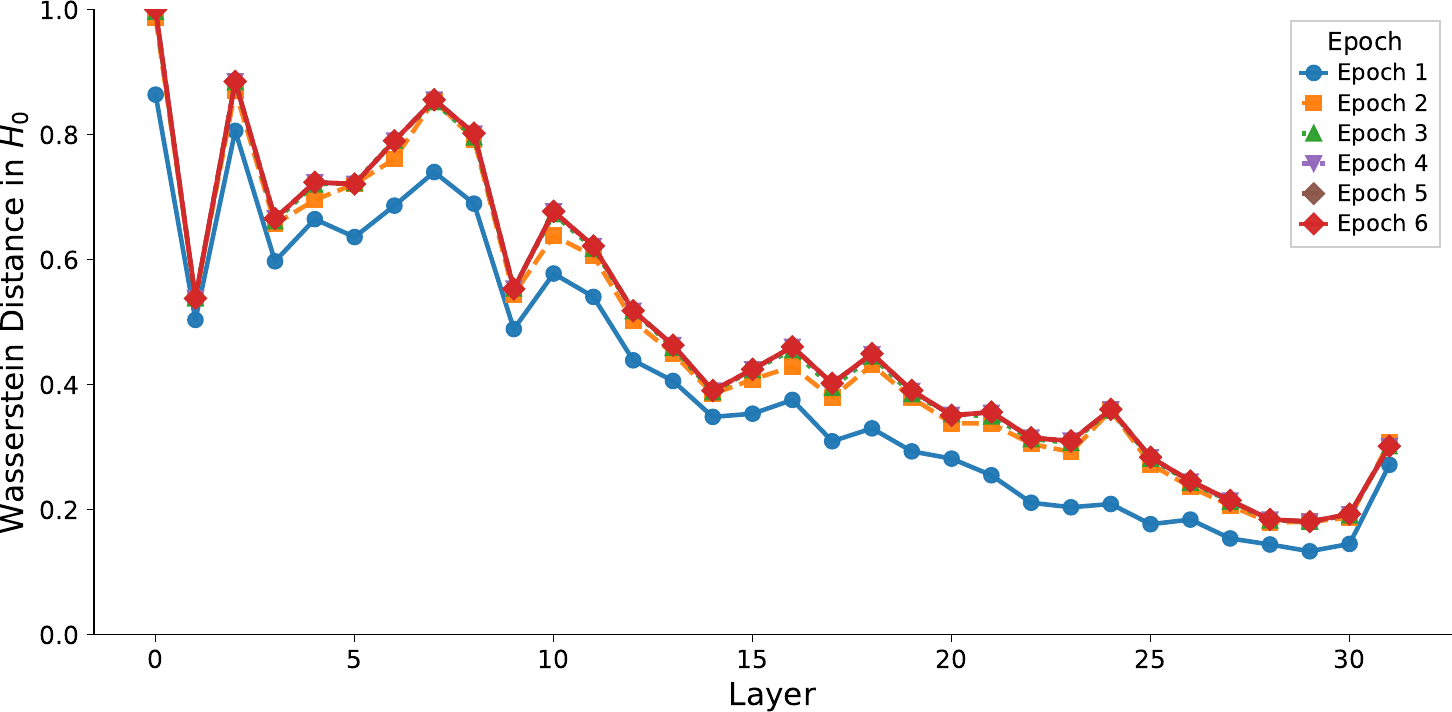} &
\includegraphics[width=0.32\textwidth]{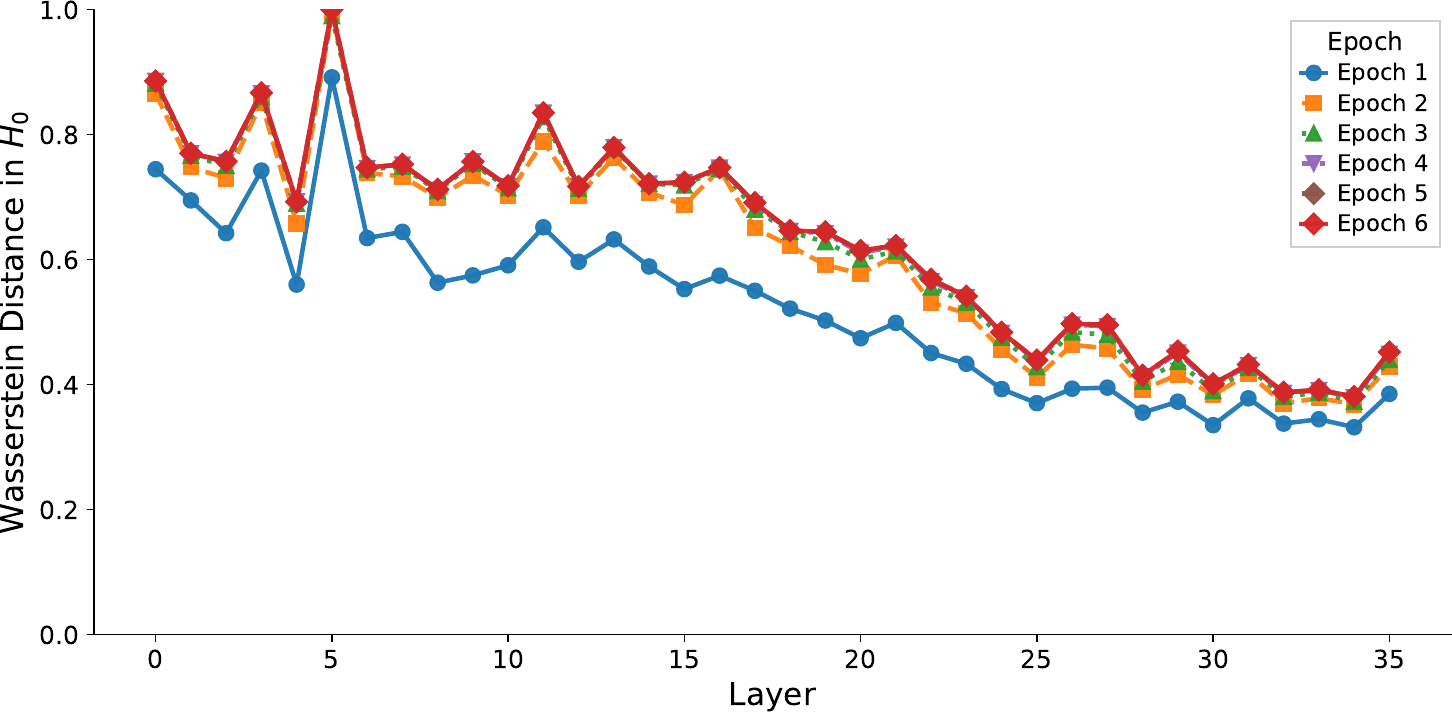} &
\includegraphics[width=0.32\textwidth]{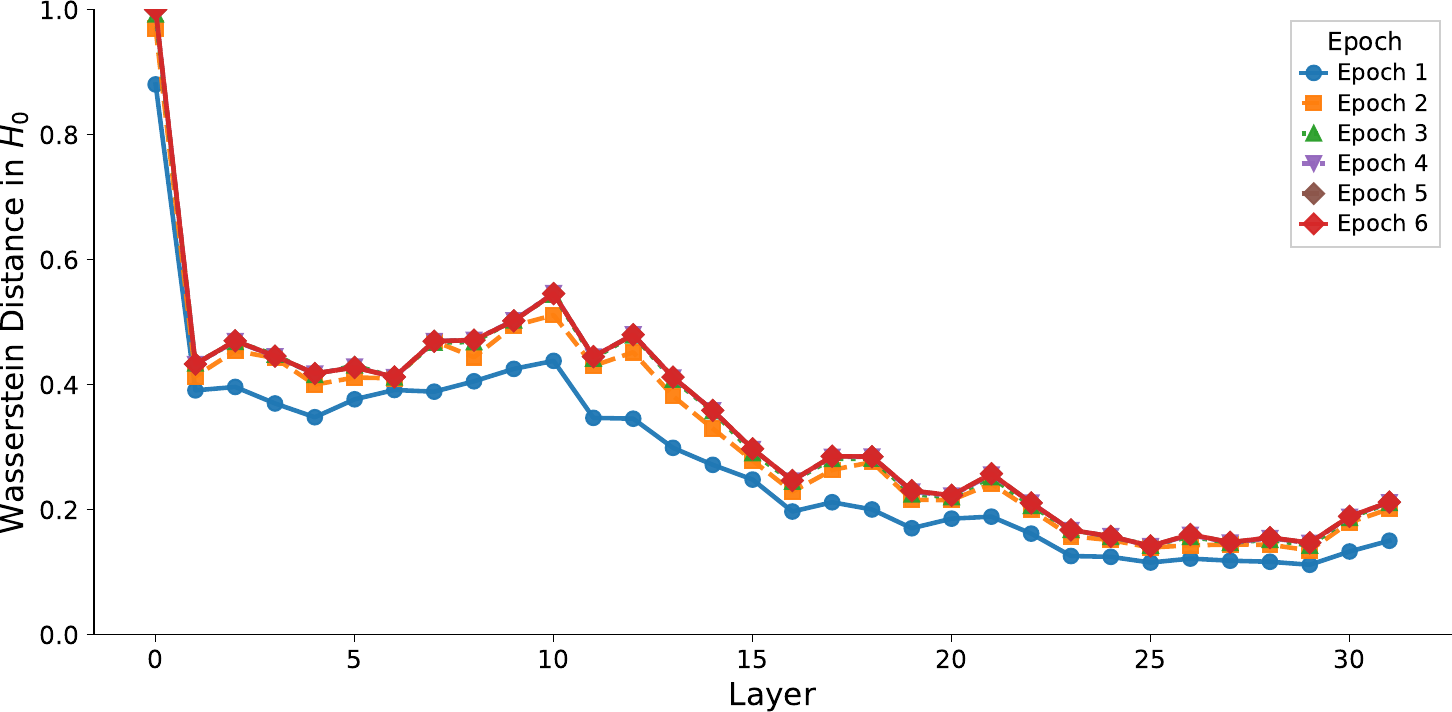} \\
[2em]
\multicolumn{3}{c}{\small \textit{\(O\) Projection}} \\
\includegraphics[width=0.32\textwidth]
{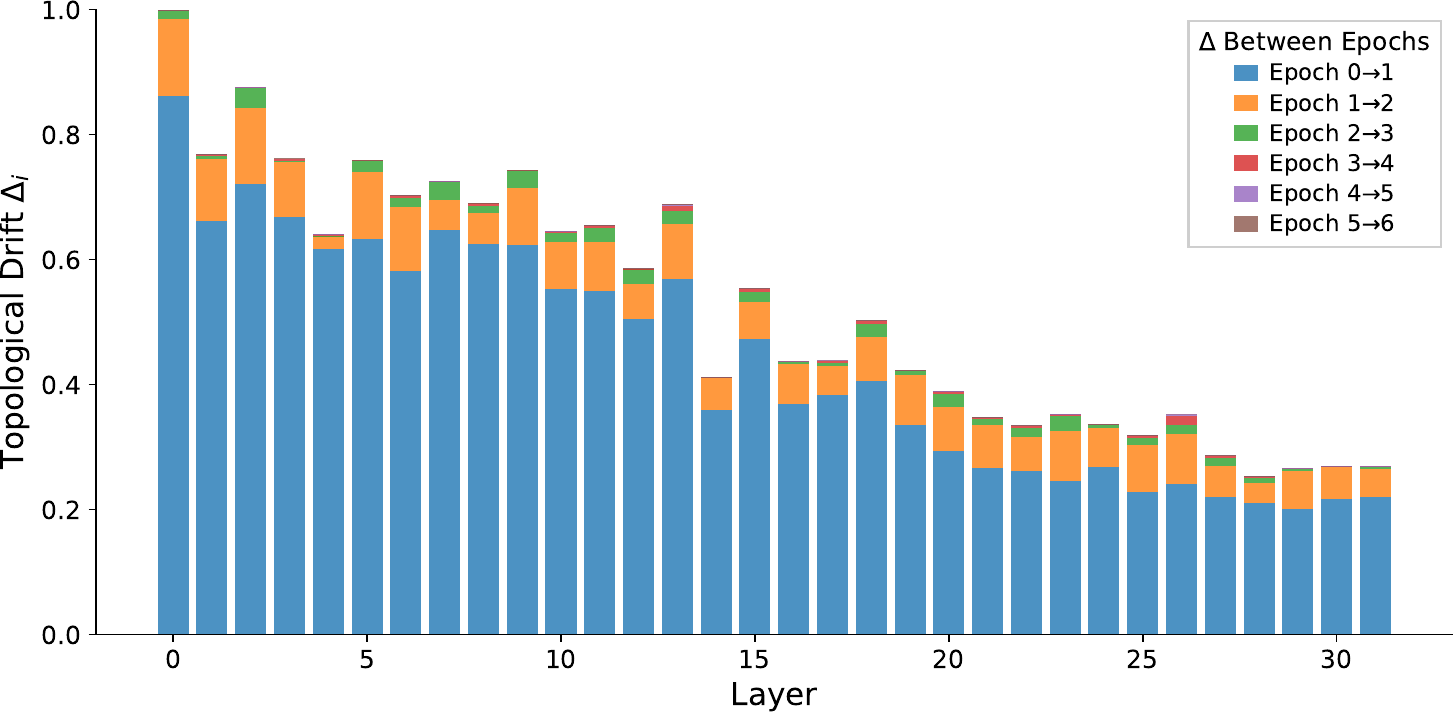} &
\includegraphics[width=0.32\textwidth]{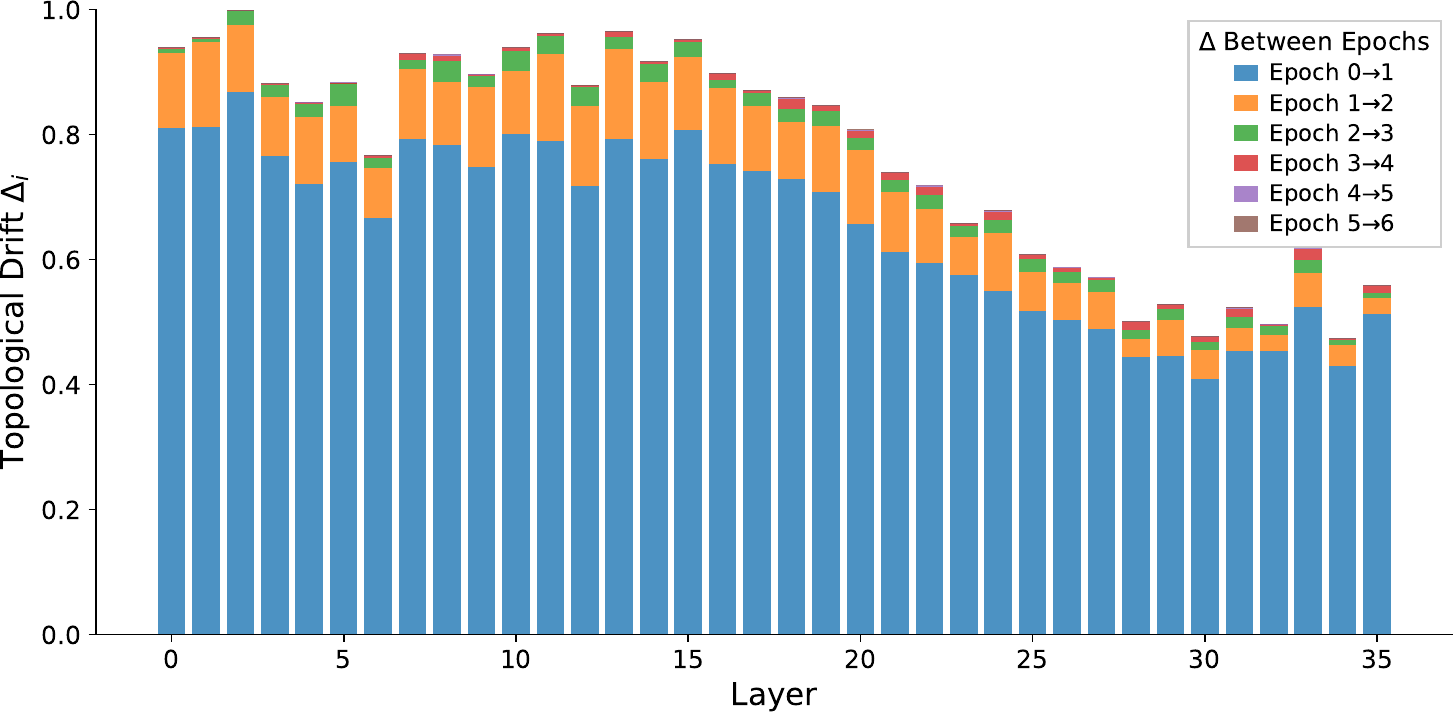} &
\includegraphics[width=0.32\textwidth]{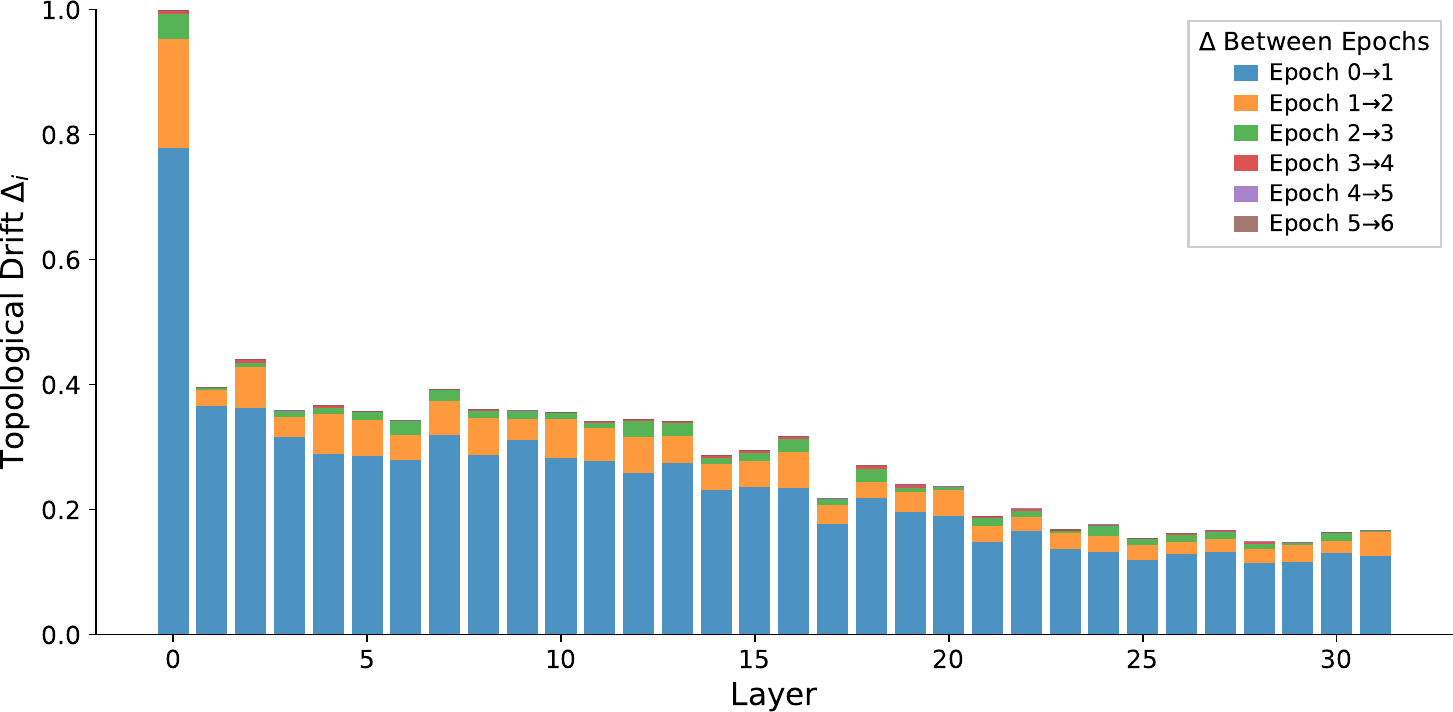} \\
[1em]
\includegraphics[width=0.32\textwidth]{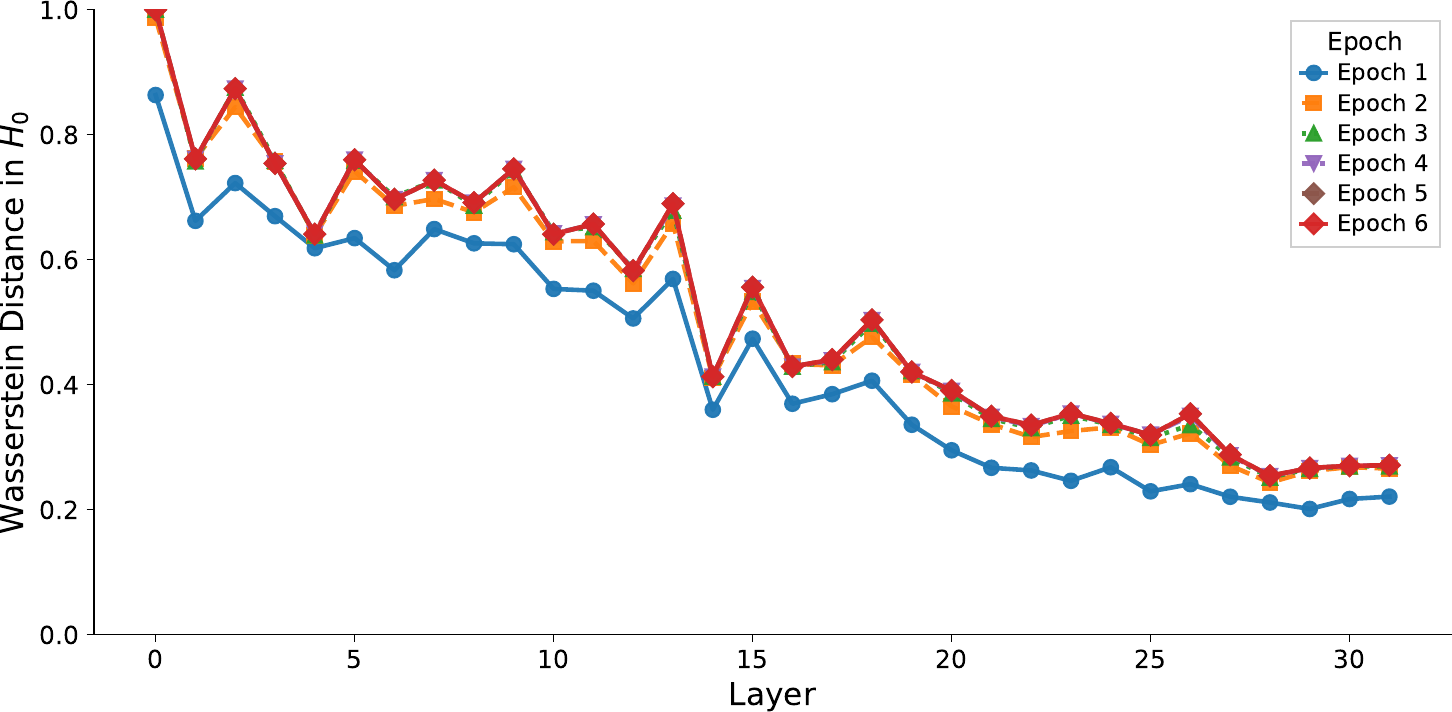} &
\includegraphics[width=0.32\textwidth]{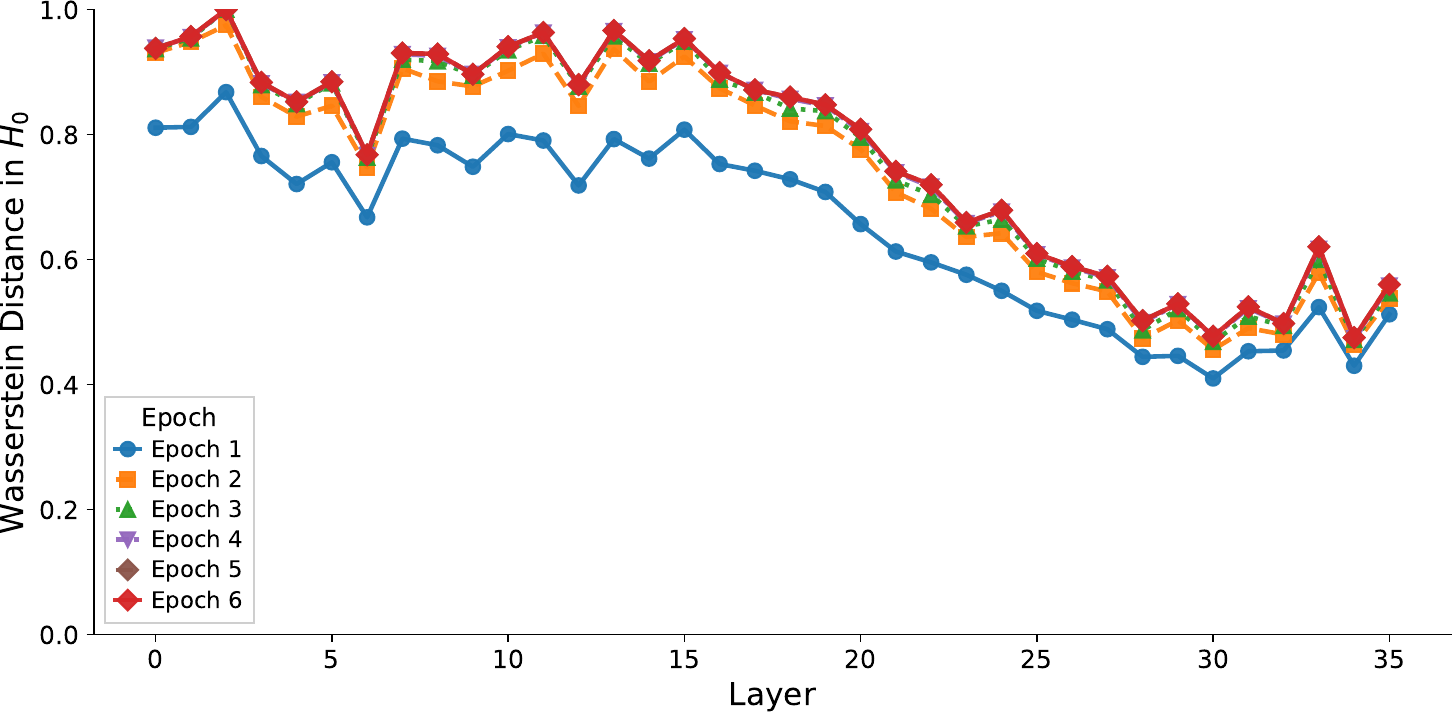} &
\includegraphics[width=0.32\textwidth]{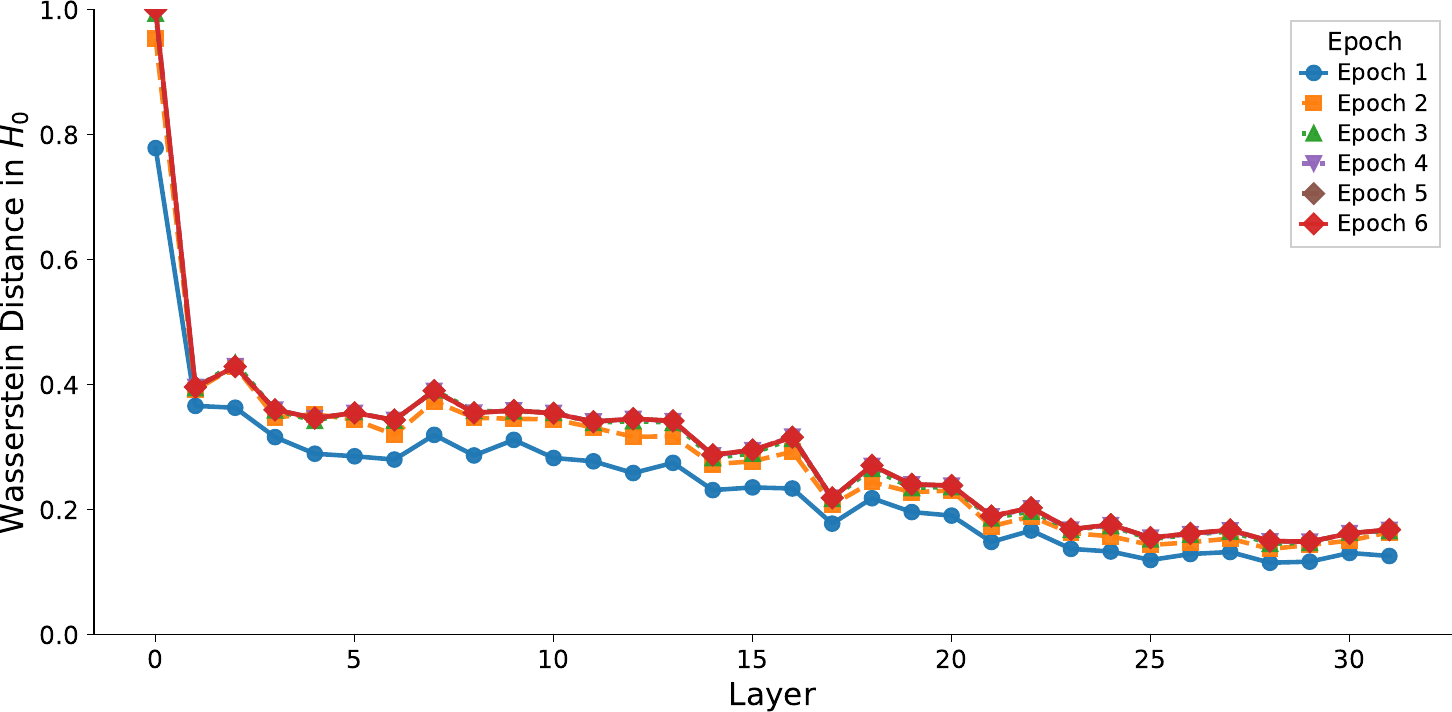} \\

\small \texttt{LLaMA-3.1-8B} & \small \texttt{Qwen3-8B-Base} & \small \texttt{Mistral-7B-v0.3} \\
\end{tabular}
\caption{Topological drift of the $V$ and $O$ projection on \texttt{SA:SST-2} under $H_0$. The top row shows epoch-to-epoch drift bars, and the bottom row shows Wasserstein distance curves across layers.}
\label{fig:topodrift-sst2-v-h0}
\end{figure*}

\newpage
\FloatBarrier
\section{Fine-tuning with Perfect Knowledge}
\label{app:perfect-knowledge}

Table~\ref{tab:combined-gsm8k} reports the perfect-knowledge \texttt{QA:GSM8K} results for \texttt{Qwen3-8B-Base}, \texttt{LLaMA-3.1-8B}, and \texttt{Mistral-7B-v0.3}. Across all three models, selective freezing recovers much of the full-fine-tuning performance while training only about $6\%$ to $8\%$ of the weights. For the best selective configurations in \texttt{Qwen3-8B-Base}, TDA-High$3$ reaches $87.62 \pm 0.18$ with $8.45\%$ trainable parameters, close to LoRA at $87.89 \pm 0.90$ and above full fine-tuning at $86.27 \pm 0.75$. For \texttt{LLaMA-3.1-8B}, TDA-High$6$ reaches $61.17 \pm 0.75$ with $6.79\%$ trainable parameters, slightly above full fine-tuning at $60.70 \pm 0.87$ and above LoRA at $56.90 \pm 1.92$. For \texttt{Mistral-7B-v0.3}, TDA-High$3$ reaches $52.87 \pm 0.46$ with $8.39\%$ trainable parameters, compared with full fine-tuning at $59.59 \pm 0.85$ and LoRA at $49.28 \pm 1.94$.

\begin{table}[!htbp]
\centering
\small
\setlength{\tabcolsep}{2pt}
\renewcommand{\arraystretch}{1.08}
\caption{\textbf{Perfect Knowledge.} Fine-tuning models on \texttt{QA:GSM8K} with selected $V$ and $O$ freezing. Wasserstein distances use $H_0$. Projection IDs and results of other $k$ values are given in Table ~\ref{tab:llama38GSM8KFull}.}
\resizebox{\linewidth}{!}{
\begin{tabular}{l|cccc|cccc|cccc}
\toprule
 & \multicolumn{4}{c|}{\texttt{Qwen3-8B-Base}} & \multicolumn{4}{c|}{\texttt{Llama-3.1-8B}} & \multicolumn{4}{c}{\texttt{Mistral-7B-v0.3}} \\
Method 
& Train.\% & Upd. \% & Acc.$\pm$Std ($\uparrow$) & Time (min) 
& Train.\% & Upd. & Acc.$\pm$Std ($\uparrow$) & Time (min)
& Train.\% & Upd. & Acc.$\pm$Std ($\uparrow$) & Time (min) \\
\midrule
\graycell{~~Full Fine-tuning}
& \graycell{100} & \graycell{14.63} & \graycell{$86.27 \pm 0.75$} & \graycell{47.3}
& \graycell{100} & \graycell{27.54} & \graycell{$60.70 \pm 0.87$} & \graycell{38.7}
& \graycell{100} & \graycell{31.46} & \graycell{$59.59 \pm 0.85$} & \graycell{43.3} \\

\midrule
\multicolumn{10}{l}{\textit{PEFT}} \\
~~LoRA 
& 0.19 & 18.41 & $\mathbf{87.89 \pm 0.90}$ & 35.4 
& 0.17 & 16.70 & $56.90 \pm 1.92$ & 28.6 
& 0.19 & 18.49 & $49.28 \pm 1.94$ & 33.5 \\





\midrule
\multicolumn{10}{l}{\textit{TDA-based tuning (ours)}} \\
~~High3 
& 8.45 & 1.22 & $\underline{87.62 \pm 0.18}$ & 34.5 
& 7.57 & 2.64 & $\underline{59.33 \pm 0.46}$ & 26.2 
& 8.39 & 2.82 & $\mathbf{52.87 \pm 0.46}$ & 30.5 \\

~~High6 
& 7.68 & 1.12 & $86.91 \pm 0.46$ & 34.4 
& 6.79 & 2.33 & $\mathbf{61.17 \pm 0.75}$ & 25.7 
& 6.8 & 2.52 & $\underline{52.64 \pm 0.99}$ & 30.5 \\

~~High9 
& 6.91 & 1.00 & $86.91 \pm 0.16$ & 34.4 
& 6.01 & 2.05 & $58.97 \pm 0.76$ & 25.5 
& 5.21 & 2.23 & $51.40 \pm 0.92$ & 30.3 \\

\bottomrule
\end{tabular}}
\label{tab:combined-gsm8k}
\end{table}

These results show two patterns. First, selective freezing yields large reductions in updated parameters with competitive performance across all three models. Second, TDA-based freezing gives the best selective result for for \texttt{Qwen3-8B-Base} and \texttt{Mistral-7B-v0.3}, and tied in \texttt{LLaMA-3.1-8B}. For more details, see Tables~\ref{tab:llama38GSM8KFull}, \ref{tab:qwenGSM8KFull}, and \ref{tab:mistralGSM8KFull}.

\begin{table*}[!htbp]
\centering
\scriptsize
\caption{\texttt{LLaMA-3.1-8B} on \texttt{QA:GSM8K} with separate $V$ and $O$ freezing. Among selective freezing settings, the best configuration is tied at TDA-High$6$ (\%61.17) and Eltwise-Low$9$ (\%61.17).}
\setlength{\tabcolsep}{3.5pt}
\resizebox{\textwidth}{!}{
\begin{tabular}{lllll}
\toprule
Method & Frozen Layers & Trainable \% & Avg. Accuracy$\pm$Std ($\uparrow$) & Time (min) \\
\midrule
\multicolumn{5}{l}{\textit{Baselines}} \\
Full & -- & $100$ & $60.70 \pm 0.87$ & 38.7 \\
LoRA & -- & $0.17$ & $56.90 \pm 1.92$ & 28.6 \\

\midrule
\multicolumn{5}{l}{\textit{Magnitude-based Tuning (ours)}} \\
Eltwise & Low3: $V{:}[9,10,21] \mid O{:}[29,31,23]$ & $7.57$ & $59.80 \pm 2.00$ & 26.4 \\
Eltwise & Low6: $V{:}[9,10,21,29,8,18] \mid O{:}[29,31,23,27,7,4]$ & $6.79$ & $59.33 \pm 1.93$ & 26.3 \\
\rowcolor{gray!25} Eltwise & Low9: $V{:}[9,10,21,29,8,18,7,2,5] \mid O{:}[29,31,23,27,7,4,3,28,12]$ & $6.01$ & $\mathbf{61.17 \pm 1.42}$ & 26.2 \\
Eltwise & Low12: $V{:}[9,10,21,29,8,18,7,2,5,30,17,23] \mid O{:}[29,31,23,27,7,4,3,28,12,9,6,14]$ & $5.22$ & $60.57 \pm 1.16$ & 26.3 \\
Eltwise & Low15: $V{:}[9,10,21,29,8,18,7,2,5,30,17,23,19,16,4] \mid O{:}[29,31,23,27,7,4,3,28,12,9,6,14,8,1,19]$ & $4.44$ & $60.50 \pm 1.04$ & 26.2 \\
Eltwise & High3: $V{:}[27,13,20] \mid O{:}[26,18,25]$ & $7.57$ & $59.73 \pm 2.06$ & 26.2 \\
Eltwise & High6: $V{:}[25,6,28,27,13,20] \mid O{:}[20,16,0,26,18,25]$ & $6.79$ & $59.60 \pm 0.35$ & 26.3 \\
Eltwise & High9: $V{:}[24,15,14,25,6,28,27,13,20] \mid O{:}[30,13,10,20,16,0,26,18,25]$ & $6.01$ & $59.47 \pm 0.97$ & 26.2 \\
Eltwise & High12: $V{:}[12,0,1,24,15,14,25,6,28,27,13,20] \mid O{:}[15,17,21,30,13,10,20,16,0,26,18,25]$ & $5.22$ & $58.27 \pm 0.31$ & 26.2 \\
Eltwise & High15: $V{:}[26,11,31,12,0,1,24,15,14,25,6,28,27,13,20] \mid O{:}[24,22,5,15,17,21,30,13,10,20,16,0,26,18,25]$ & $4.44$ & $58.30 \pm 1.60$ & 26.1 \\

\midrule

\multicolumn{5}{l}{\textit{TDA-based Tuning (ours)}} \\
$H0$ & Low3: $V{:}[29,26,30] \mid O{:}[31,30,26]$ & $7.57$ & $58.17 \pm 0.42$ & 25.4 \\
$H0$ & Low6: $V{:}[26,27,28,29,30,31] \mid O{:}[24,25,26,29,30,31]$ & $6.79$ & $57.43 \pm 0.25$ & 25.4 \\
$H0$ & Low9: $V{:}[22,23,25,26,27,28,29,30,31] \mid O{:}[21,23,24,25,26,27,29,30,31]$ & $6.01$ & $58.60 \pm 0.56$ & 25.6 \\
$H0$ & Low12: $V{:}[15,21,22,23,24,25,26,27,28,29,30,31] \mid O{:}[20,21,22,23,24,25,26,27,28,29,30,31]$ & $5.22$ & $58.43 \pm 1.40$ & 25.8 \\
$H0$ & Low15: $V{:}[29,26,30,28,31,27,25,22,23,15,24,21,17,3,1] \mid O{:}[31,30,26,29,25,24,27,21,23,28,22,20,16,4,19]$ & $4.44$ & $58.43 \pm 0.58$ & 25.1 \\
$H0$ & High3: $V{:}[8,2,0] \mid O{:}[11,10,13]$ & $7.57$ & $59.33 \pm 0.46$ & 26.2 \\
\rowcolor{gray!25} $H0$ & High6: $V{:}[0,2,7,8,10,11] \mid O{:}[0,5,9,10,11,13]$ & $6.79$ & $\mathbf{61.17 \pm 0.75}$ & 25.7 \\
$H0$ & High9: $V{:}[0,2,5,6,7,8,10,11,16] \mid O{:}[0,2,5,8,9,10,11,12,13]$ & $6.01$ & $58.97 \pm 0.76$ & 25.5 \\
$H0$ & High12: $V{:}[0,2,5,6,7,8,10,11,12,13,16,18] \mid O{:}[0,1,2,5,6,7,8,9,10,11,12,13]$ & $5.22$ & $58.37 \pm 0.60$ & 25.9 \\
$H0$ & High15: $V{:}[0,2,4,5,6,7,8,9,10,11,12,13,14,16,18] \mid O{:}[0,1,2,3,5,6,7,8,9,10,11,12,13,15,18]$ & $4.44$ & $57.80 \pm 1.18$ & 26.0 \\
\bottomrule
\end{tabular}}
\label{tab:llama38GSM8KFull}
\end{table*}

\begin{table*}[!htbp]
\centering
\scriptsize
\caption{\texttt{Qwen3-8B-Base} on \texttt{QA:GSM8K} with separate $V$ and $O$ freezing. Among selective freezing settings, the best configuration is TDA-High$3$ (\%87.62), followed by Eltwise-High$9$ (\%87.43) and Eltwise-Low$12$ (\%87.40).}
\setlength{\tabcolsep}{3.5pt}
\resizebox{\textwidth}{!}{
\begin{tabular}{lllll}
\toprule
Method & Frozen V/O Layers & Trainable \% & Avg. Accuracy$\pm$Std ($\uparrow$) & Time (min) \\
\midrule
\multicolumn{5}{l}{\textit{Baselines}} \\
Full & -- & $100$ & $86.27 \pm 0.75$ & 47.3 \\
LoRA & -- & $0.19$ & $87.89 \pm 0.90$ & 35.4 \\

\midrule
\multicolumn{5}{l}{\textit{Magnitude-based Tuning (ours)}} \\
Eltwise & Low3: $V{:}[11,24,8] \mid O{:}[32,34,24]$ & $8.45$ & $87.30 \pm 0.36$ & 34.6 \\
Eltwise & Low6: $V{:}[11,24,8,12,18,33] \mid O{:}[32,34,24,10,27,33]$ & $7.68$ & $87.23 \pm 0.12$ & 34.5 \\
Eltwise & Low9: $V{:}[11,24,8,12,18,33,5,34,15] \mid O{:}[32,34,24,10,27,33,19,6,4]$ & $6.91$ & $87.23 \pm 0.21$ & 34.4 \\
\rowcolor{gray!25} Eltwise & Low12: $V{:}[11,24,8,12,18,33,5,34,15,29,9,35] \mid O{:}[32,34,24,10,27,33,19,6,4,13,11,20]$ & $6.14$ & $\mathbf{87.40 \pm 0.44}$ & 34.2 \\
Eltwise & Low15: $V{:}[11,24,8,12,18,33,5,34,15,29,9,35,0,2,4] \mid O{:}[32,34,24,10,27,33,19,6,4,13,11,20,28,29,9]$ & $5.38$ & $86.80 \pm 0.17$ & 34.2 \\
Eltwise & High3: $V{:}[32,22,26] \mid O{:}[35,2,30]$ & $8.45$ & $87.33 \pm 0.15$ & 34.6 \\
Eltwise & High6: $V{:}[28,3,13,32,22,26] \mid O{:}[21,5,15,35,2,30]$ & $7.68$ & $87.17 \pm 0.15$ & 34.5 \\
\rowcolor{gray!25} Eltwise & High9: $V{:}[1,19,31,28,3,13,32,22,26] \mid O{:}[14,8,26,21,5,15,35,2,30]$ & $6.91$ & $\mathbf{87.43 \pm 0.57}$ & 34.4 \\
Eltwise & High12: $V{:}[17,27,10,1,19,31,28,3,13,32,22,26] \mid O{:}[17,22,16,14,8,26,21,5,15,35,2,30]$ & $6.14$ & $86.83 \pm 0.29$ & 34.2 \\
Eltwise & High15: $V{:}[25,20,23,17,27,10,1,19,31,28,3,13,32,22,26] \mid O{:}[1,3,25,17,22,16,14,8,26,21,5,15,35,2,30]$ & $5.38$ & $86.37 \pm 0.65$ & 34.1 \\

\midrule
\multicolumn{5}{l}{\textit{TDA-based Tuning (ours)}} \\
$H0$ & Low3: $V{:}[25,10,4] \mid O{:}[6,34,29]$ & $8.45$ & $86.40 \pm 0.38$ & 34.7 \\
$H0$ & Low6: $V{:}[25,10,4,6,12,17] \mid O{:}[6,34,29,25,28,4]$ & $7.68$ & $86.91 \pm 0.27$ & 34.6 \\
$H0$ & Low9: $V{:}[25,10,4,6,12,17,8,15,7] \mid O{:}[6,34,29,25,28,4,30,26,31]$ & $6.91$ & $86.86 \pm 0.42$ & 34.4 \\
$H0$ & Low12: $V{:}[25,10,4,6,12,17,8,15,7,19,26,34] \mid O{:}[6,34,29,25,28,4,30,26,31,32,9,35]$ & $6.14$ & $86.30 \pm 0.35$ & 34.0 \\
$H0$ & Low15: $V{:}[25,10,4,6,12,17,8,15,7,19,26,34,24,28,30] \mid O{:}[6,34,29,25,28,4,30,26,31,32,9,35,27,24,5]$ & $5.38$ & $86.33 \pm 0.19$ & 34.1 \\
\rowcolor{gray!25} $H0$ & High3: $V{:}[1,5,0] \mid O{:}[18,17,15]$ & $8.45$ & $\mathbf{87.62 \pm 0.18}$ & 34.5 \\
$H0$ & High6: $V{:}[31,29,1,5,0,13] \mid O{:}[14,19,17,18,15,16]$ & $7.68$ & $86.91 \pm 0.46$ & 34.4 \\
$H0$ & High9: $V{:}[32,3,21,13,31,29,1,5,0] \mid O{:}[20,13,22,16,14,19,17,18,15]$ & $6.91$ & $86.91 \pm 0.16$ & 34.4 \\
$H0$ & High12: $V{:}[16,33,35,32,3,21,13,31,29,1,5,0] \mid O{:}[2,11,0,20,13,22,16,14,19,17,18,15]$ & $6.14$ & $87.10 \pm 0.12$ & 34.0 \\
$H0$ & High15: $V{:}[14,20,11,16,33,35,32,3,21,13,31,29,1,5,0] \mid O{:}[1,12,23,2,11,0,20,13,22,16,14,19,17,18,15]$ & $5.38$ & $86.88 \pm 0.31$ & 34.1 \\
\bottomrule
\end{tabular}}

\label{tab:qwenGSM8KFull}
\end{table*}


\begin{table*}[!htbp]
\centering
\scriptsize
\caption{\texttt{Mistral-7B-v0.3} on \texttt{QA:GSM8K} with separate $V$ and $O$ freezing. Among selective freezing settings, the best configuration is TDA-Low6 (54.31\%), followed by Eltwise-High6 (53.96\%).}
\setlength{\tabcolsep}{3.5pt}
\resizebox{\textwidth}{!}{
\begin{tabular}{lllll}
\toprule
Method & Frozen Layers & Trainable \% & Avg. Accuracy$\pm$Std ($\uparrow$) & Time (min) \\
\midrule
\multicolumn{5}{l}{\textit{Baselines}} \\
Full & -- & $100$ & $59.59 \pm 0.85$ & 43.3 \\
LoRA & -- & $0.19$ & $49.28 \pm 1.94$ & 33.5 \\

\midrule
\multicolumn{5}{l}{\textit{Magnitude-based Tuning (ours)}} \\
Eltwise & Low3: $V{:}[6,10,26] \mid O{:}[14,18,19]$ & $8.39$ & $53.58 \pm 0.55$ & 30.5 \\
Eltwise & Low6: $V{:}[1,6,10,14,19,26] \mid O{:}[11,14,18,19,27,29]$ & $6.8$ & $53.78 \pm 1.09$ & 30.5 \\
Eltwise & Low9: $V{:}[1,3,6,7,10,11,14,19,26] \mid O{:}[1,11,13,14,18,19,21,27,29]$ & $5.21$ & $53.52 \pm 0.77$ & 30.4 \\
Eltwise & Low12: $V{:}[1,3,6,7,10,11,13,14,17,19,21,26] \mid O{:}[1,2,8,11,13,14,18,19,21,27,29,30]$ & $3.62$ & $52.29 \pm 0.61$ & 30.3 \\
Eltwise & Low15: $V{:}[1,2,3,6,7,10,11,13,14,17,19,21,23,26,31] \mid O{:}[1,2,5,8,11,13,14,15,18,19,21,25,27,29,30]$ & $2.03$ & $51.91 \pm 0.42$ & 30.1 \\
\rowcolor{gray!25}
Eltwise & High3: $V{:}[8,9,30] \mid O{:}[4,6,12]$ & $8.39$ & $\mathbf{53.90 \pm 2.63}$ & 31.3 \\
\rowcolor{gray!25}
Eltwise & High6: $V{:}[4,8,9,16,29,30] \mid O{:}[4,6,7,12,28,31]$ & $6.8$ & $\mathbf{53.96 \pm 1.02}$ & 30.4 \\
Eltwise & High9: $V{:}[0,4,8,9,16,18,25,29,30] \mid O{:}[0,3,4,6,7,12,22,28,31]$ & $5.21$ & $52.29 \pm 1.54$ & 30.3 \\
Eltwise & High12: $V{:}[0,4,8,9,12,16,18,20,24,25,29,30] \mid O{:}[0,3,4,6,7,10,12,16,22,23,28,31]$ & $3.62$ & $53.73 \pm 1.05$ & 30.3 \\
Eltwise & High15: $V{:}[0,4,5,8,9,12,16,18,20,22,24,25,28,29,30] \mid O{:}[0,3,4,6,7,9,10,12,16,17,20,22,23,28,31]$ & $2.03$ & $53.60 \pm 0.23$ & 30.1 \\

\midrule
\multicolumn{5}{l}{\textit{TDA-based Tuning (ours)}} \\
\rowcolor{gray!25}
$H0$ & Low3: $V{:}[24,25,29] \mid O{:}[25,28,31]$ & $8.39$ & $\mathbf{53.93 \pm 1.06}$ & 30.4 \\
\rowcolor{gray!25}
$H0$ & Low6: $V{:}[22,24,25,26,29,30] \mid O{:}[22,25,28,29,30,31]$ & $6.8$ & $\mathbf{54.31 \pm 1.64}$ & 30.5 \\
$H0$ & Low9: $V{:}[22,23,24,25,26,27,28,29,30] \mid O{:}[21,22,24,25,26,28,29,30,31]$ & $5.21$ & $53.88 \pm 0.23$ & 30.4 \\
$H0$ & Low12: $V{:}[16,20,22,23,24,25,26,27,28,29,30,31] \mid O{:}[17,21,22,23,24,25,26,27,28,29,30,31]$ & $3.62$ & $52.79 \pm 1.01$ & 30.1 \\
$H0$ & Low15: $V{:}[16,18,19,20,21,22,23,24,25,26,27,28,29,30,31] \mid O{:}[3,6,17,20,21,22,23,24,25,26,27,28,29,30,31]$ & $2.03$ & $52.82 \pm 0.56$ & 30.0 \\
$H0$ & High3: $V{:}[0,10,12] \mid O{:}[0,1,2]$ & $8.39$ & $52.87 \pm 0.46$ & 30.5 \\
$H0$ & High6: $V{:}[0,8,9,10,11,12] \mid O{:}[0,1,2,9,11,16]$ & $6.8$ & $52.64 \pm 0.99$ & 30.5 \\
$H0$ & High9: $V{:}[0,2,6,8,9,10,11,12,13] \mid O{:}[0,1,2,9,10,11,13,15,16]$ & $5.21$ & $51.40 \pm 0.92$ & 30.3 \\
$H0$ & High12: $V{:}[0,2,5,6,7,8,9,10,11,12,13,14] \mid O{:}[0,1,2,5,7,9,10,11,12,13,15,16]$ & $3.62$ & $50.64 \pm 0.95$ & 30.2 \\
$H0$ & High15: $V{:}[0,1,2,4,5,6,7,8,9,10,11,12,13,14,15] \mid O{:}[0,1,2,4,5,7,8,9,10,11,12,13,15,16,19]$ & $2.03$ & $48.60 \pm 0.60$ & 30.0 \\
\bottomrule
\end{tabular}}

\label{tab:mistralGSM8KFull}
\end{table*}

\FloatBarrier
\section{Transferred-Knowledge Accuracy Curves}
\label{app:transferred-accuracy-curves}

Figure~\ref{fig:transferred-accuracy-curves} shows the epoch-wise accuracy trajectories for \texttt{QA:MMLU}, \texttt{SA:SST-2}, and \texttt{SA:IMDB} in the transferred-knowledge setting. The curves compare full fine-tuning, LoRA, and \method after reusing the freezing profile learned from \texttt{QA:GSM8K}. These trajectories provide a per-epoch view of model adaptation across the transferred evaluation tasks.

\begin{figure*}[!htbp]
\centering
\scriptsize
\setlength{\tabcolsep}{2pt}
\begin{tabular}{ccc}
\multicolumn{3}{c}{\small \texttt{SA: SST-2}} \\
\includegraphics[width=0.32\linewidth]{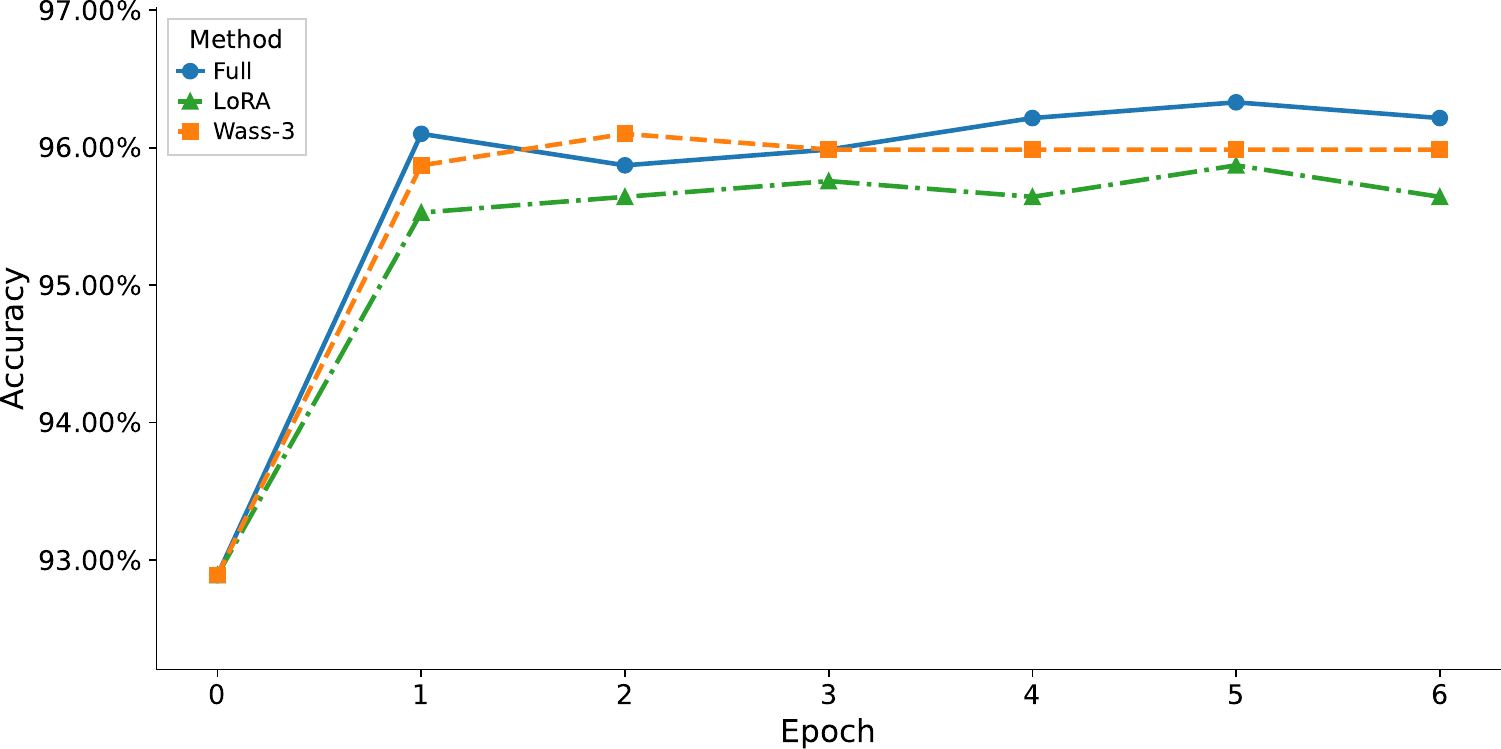} &
\includegraphics[width=0.32\linewidth]{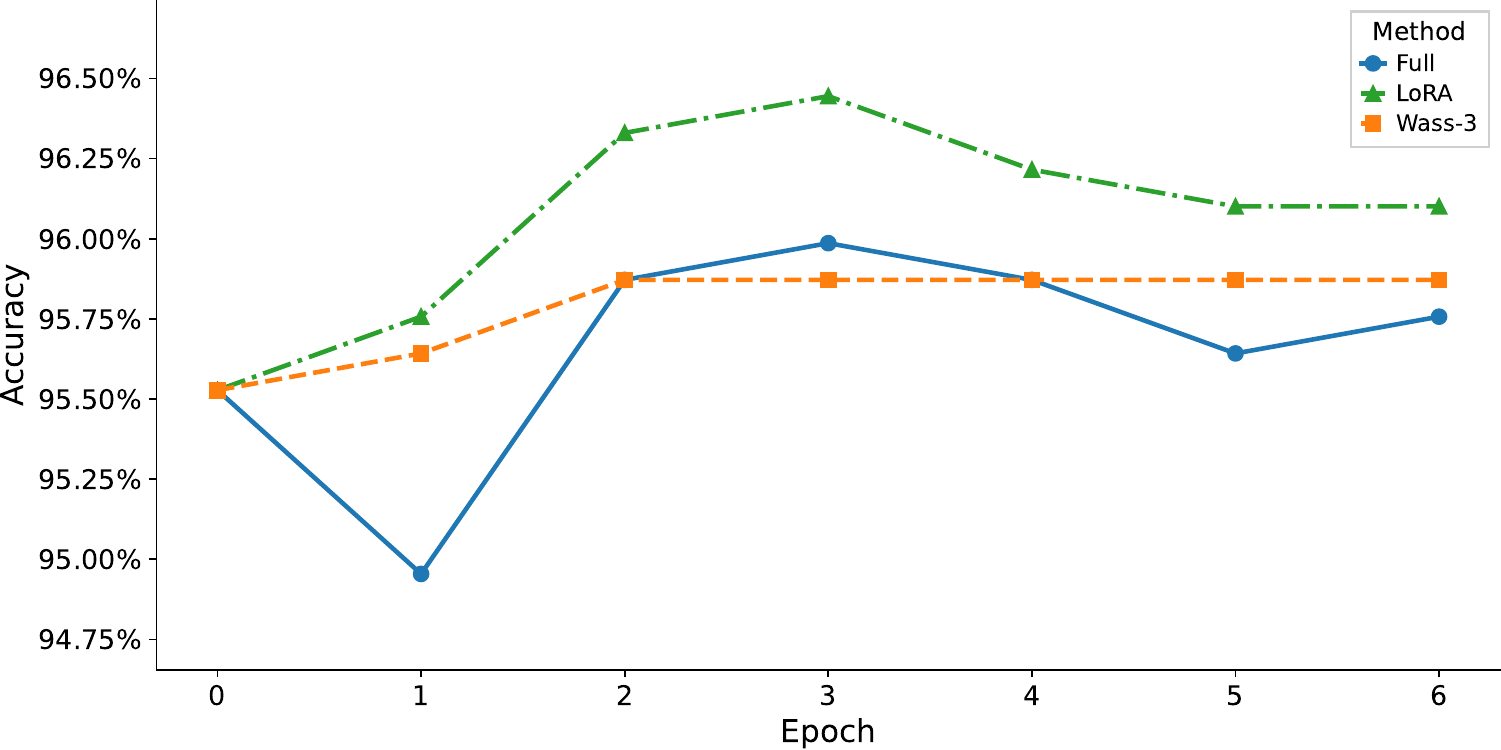} &
\includegraphics[width=0.32\linewidth]{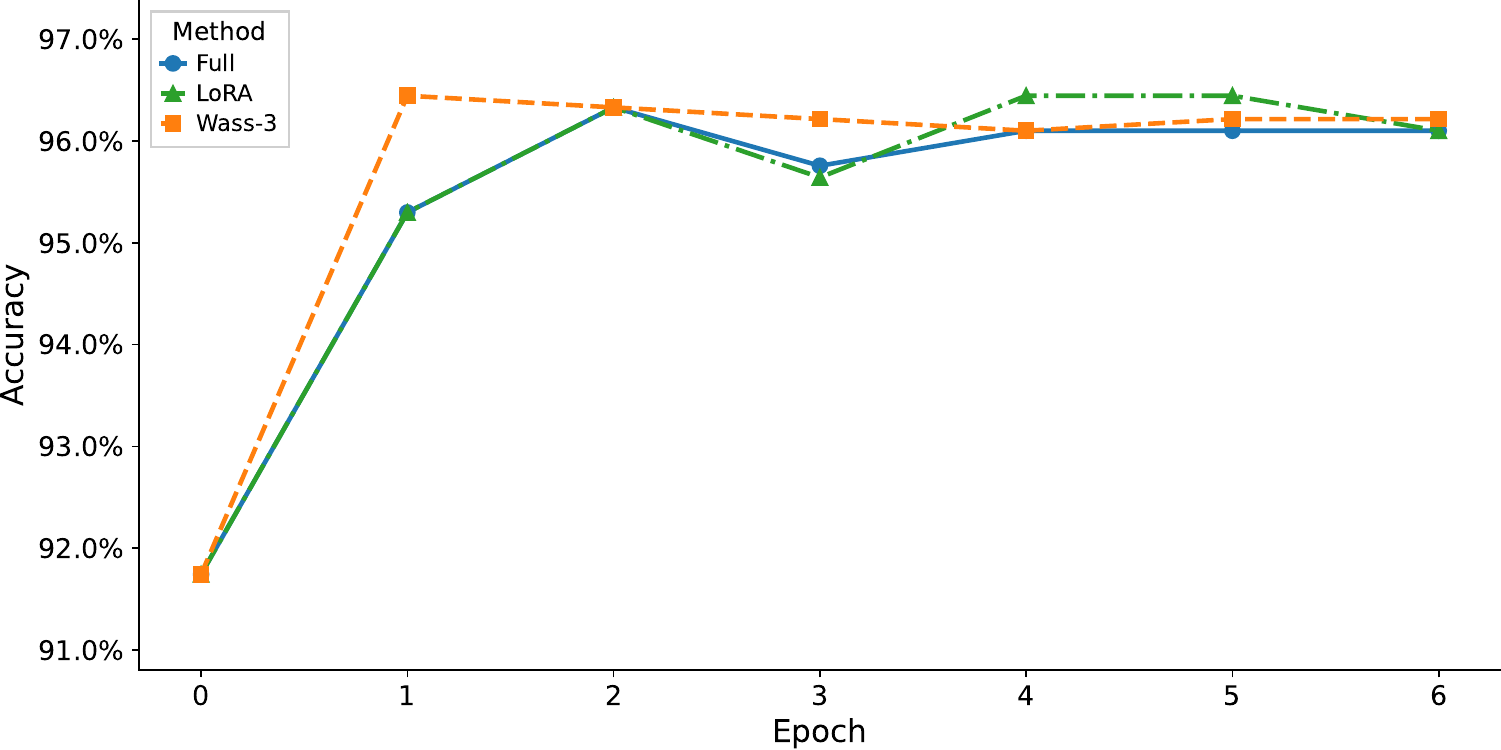} \\
[2em]
\multicolumn{3}{c}{\small \texttt{SA: IMDB}} \\
\includegraphics[width=0.32\linewidth]{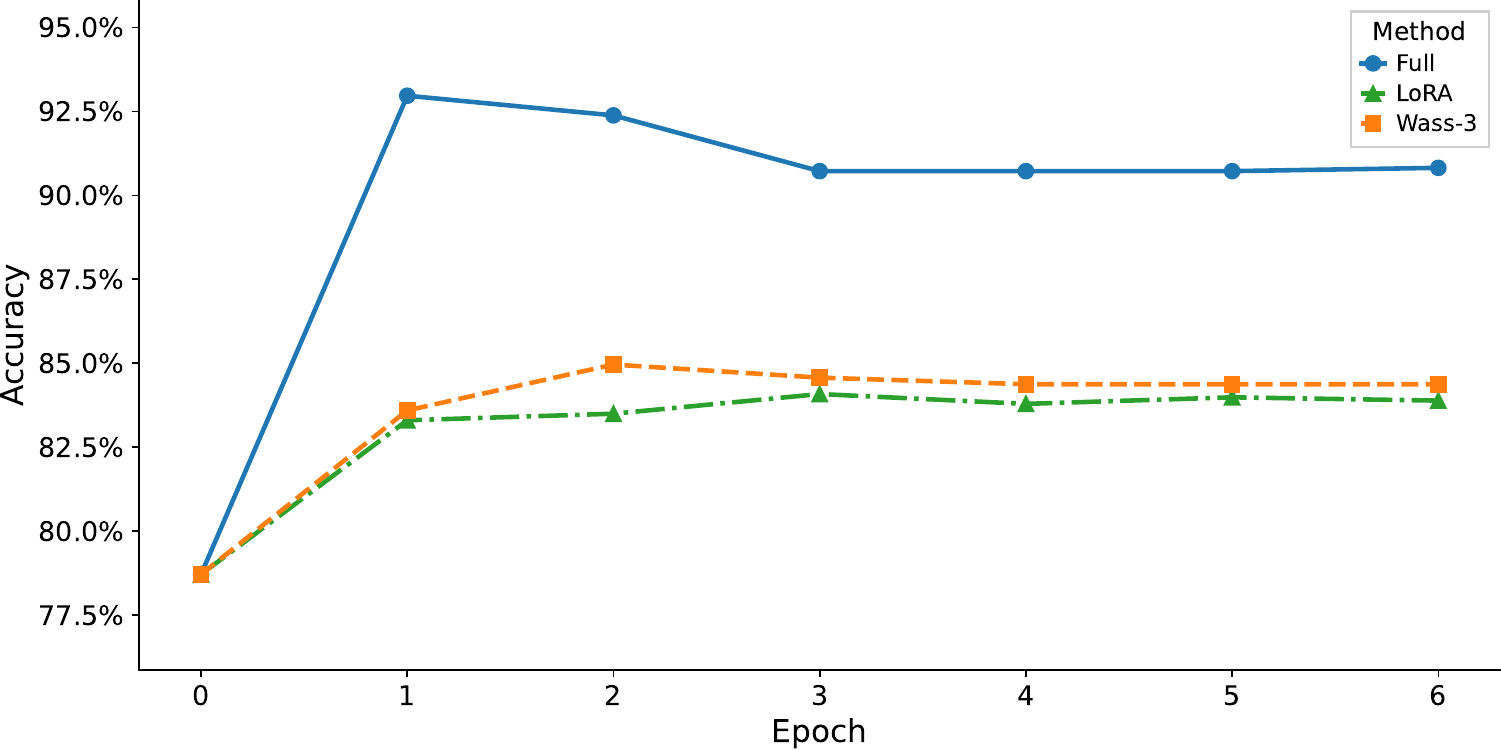} &
\includegraphics[width=0.32\linewidth]{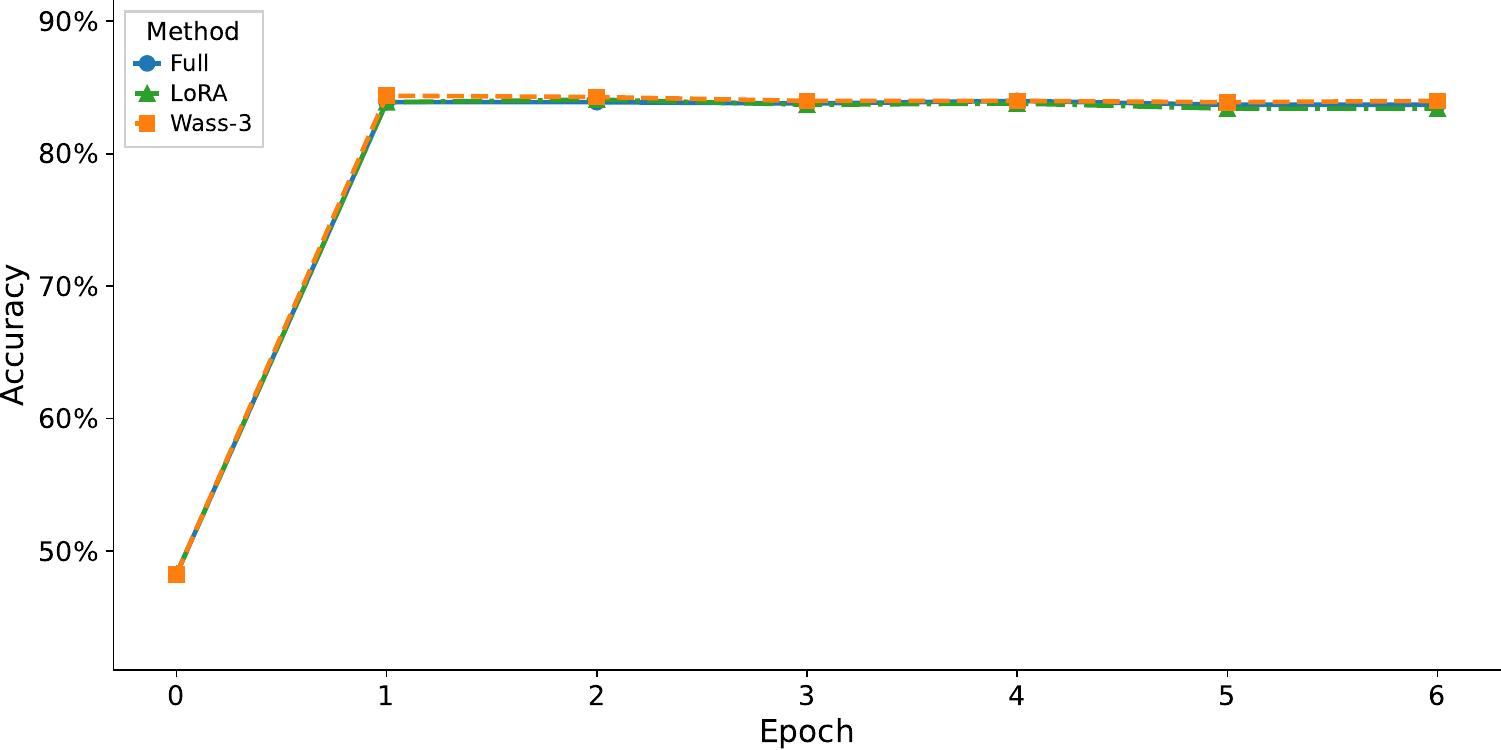} &
\includegraphics[width=0.32\linewidth]{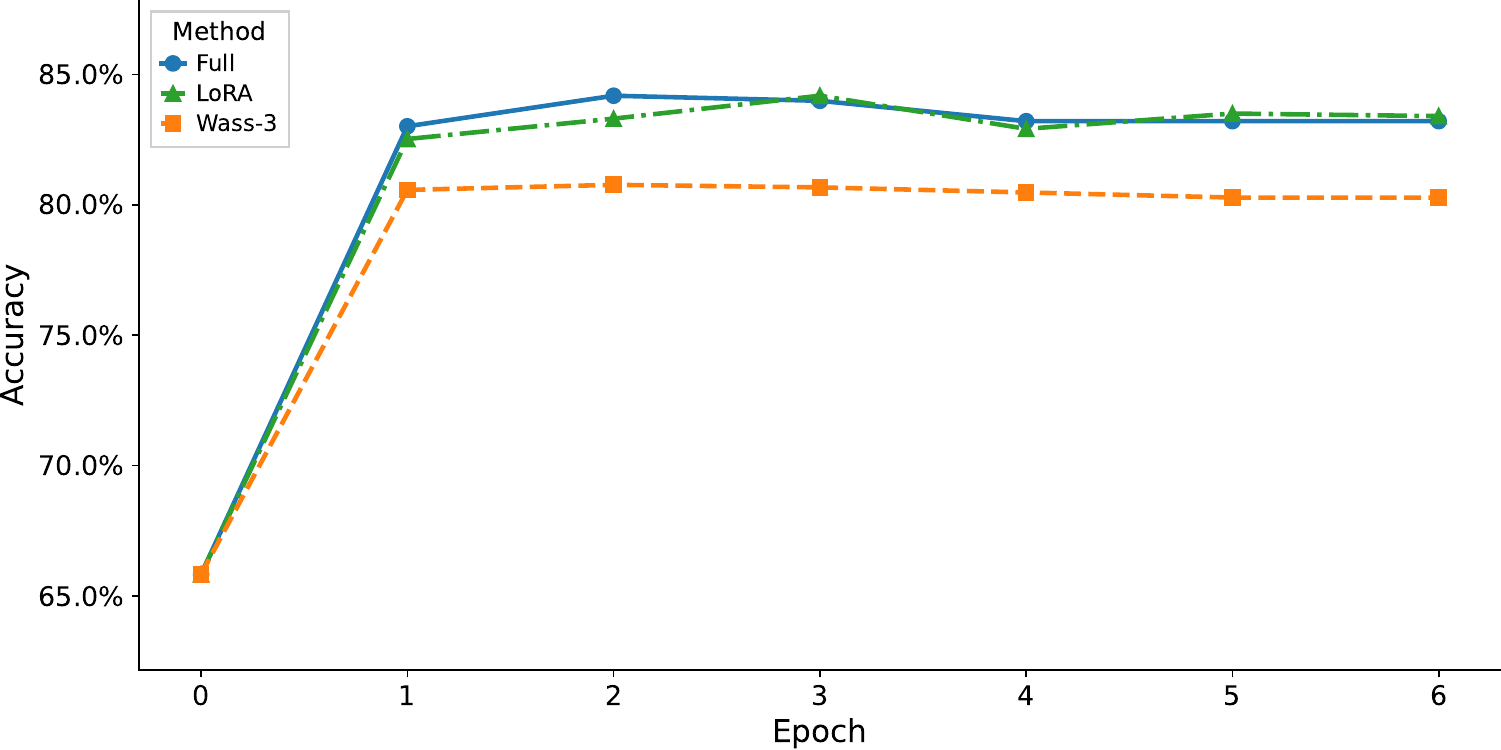} \\
[2em]
\multicolumn{3}{c}{\small \texttt{QA: MMLU}} \\
\includegraphics[width=0.32\linewidth]{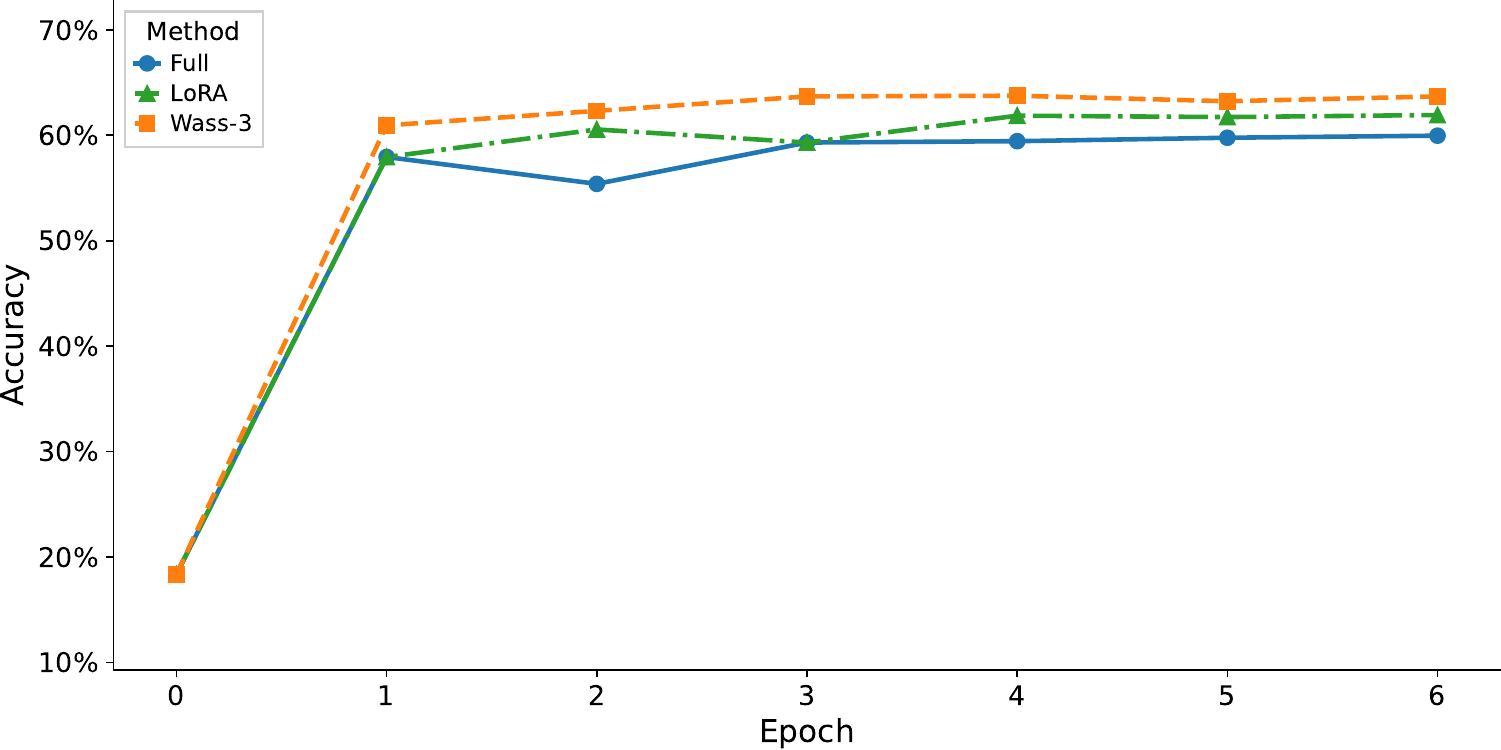} &
\includegraphics[width=0.32\linewidth]{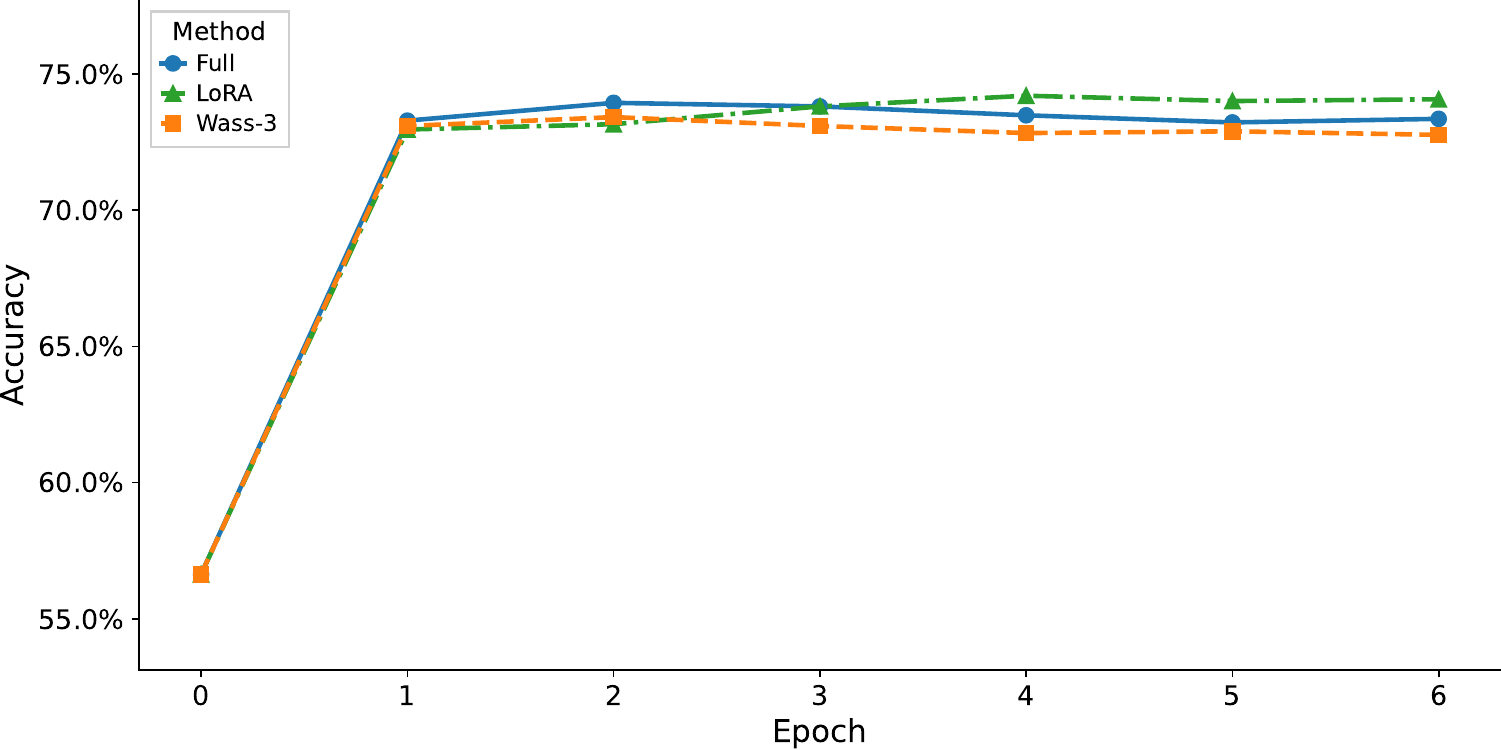} &
\includegraphics[width=0.32\linewidth]{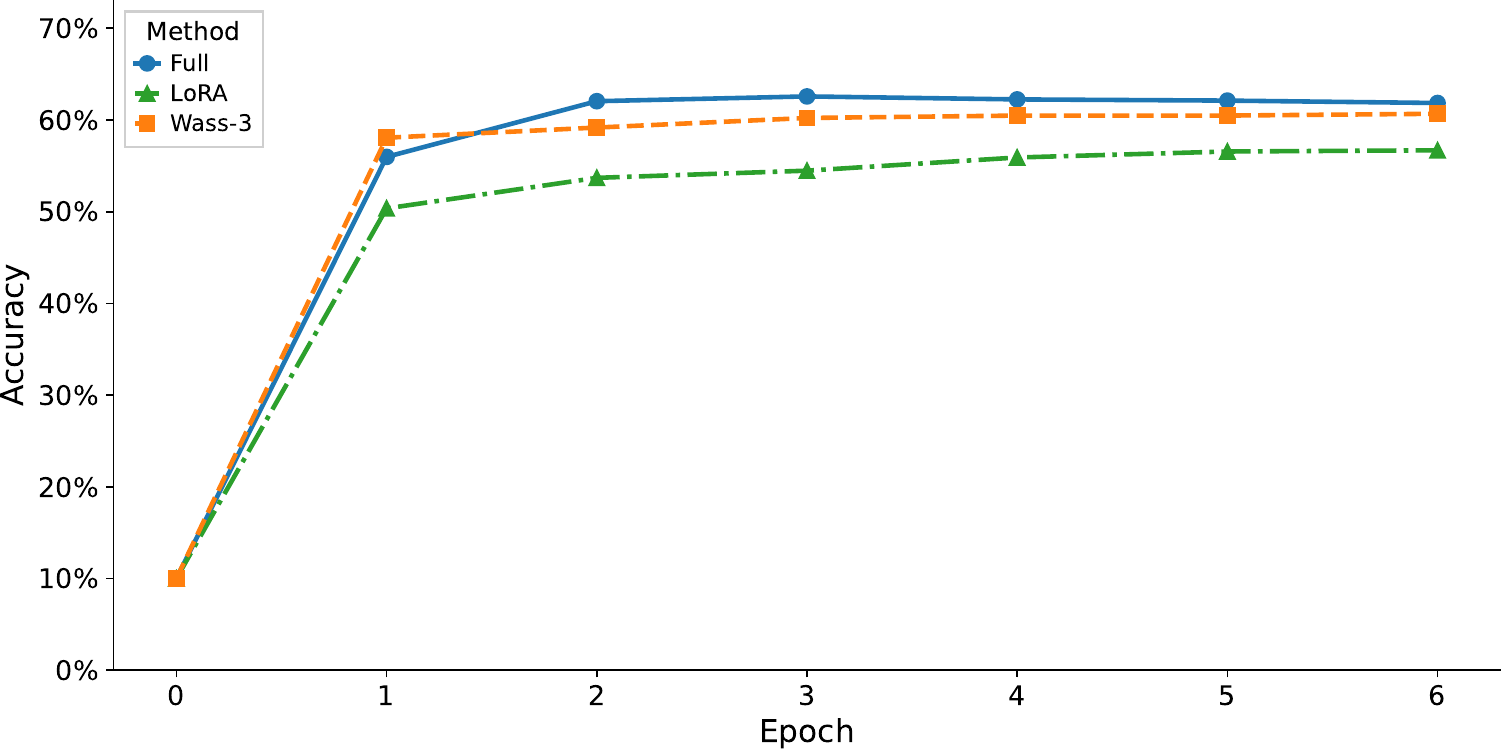} \\

\small \texttt{LLaMA-3.1-8B} & \small \texttt{Qwen3-8B-Base} & \small \texttt{Mistral-7B-v0.3} \\
\end{tabular}
\caption{Accuracies on each epoch for all the models across datasets}
\label{fig:transferred-accuracy-curves}
\end{figure*}
\FloatBarrier
\section{Spectrum and DropBP}
\label{app:competitors}



\paragraph{Spectrum.}
Spectrum~\cite{spectrum2024} is a parameter-selection baseline that uses signal-to-noise ratio (SNR) statistics to decide which parameters should remain trainable. In our implementation, we first tested several top-$N\%$ SNR configurations, including $N \in \{50,10,3\}$, to understand how the selected trainable budget changes with the masking ratio. We report the $N=5$ setting because its updatable parameter fraction is closest to \method, making the comparison more controlled in terms of effective parameter updates.

\paragraph{DropBP.}
DropBP~\cite{woo2024dropbp} is a compute- and memory-oriented training baseline that reduces backward computation during fine-tuning. Since DropBP can be combined with either full fine-tuning or LoRA, we considered both DropBP+Full and DropBP+LoRA as possible baselines. We use DropBP+LoRA in the main comparison because its updatable parameter fraction is closer to \method, making the LoRA-based DropBP setting the fairer comparison.

\paragraph{Experimental protocol.}
For both baselines, we keep the experimental setup aligned with the rest of our transfer experiments. Spectrum is run with selective full-parameter fine-tuning over the SNR-selected unfrozen weights, while DropBP is run with LoRA adapters on the attention projections. We use the same model-dataset grid, training data regime, and evaluation pipeline as our main \texttt{SA:SST-2}, \texttt{SA:IMDB}, and \texttt{QA:MMLU} experiments. The results are provided in Table~\ref{tab:mainPerformance}. Experiment setups and hyperparameter details are provided in the Appendix~\ref{app:hyperparameters}.
\FloatBarrier
\section{Fine-tuning with Dynamic Stopping}
\label{app:dynamic-stopping}






\begin{figure}[!htbp]
\centering
\setlength{\tabcolsep}{2pt}

\begin{tabular}{ccc}
\includegraphics[width=0.48\linewidth]{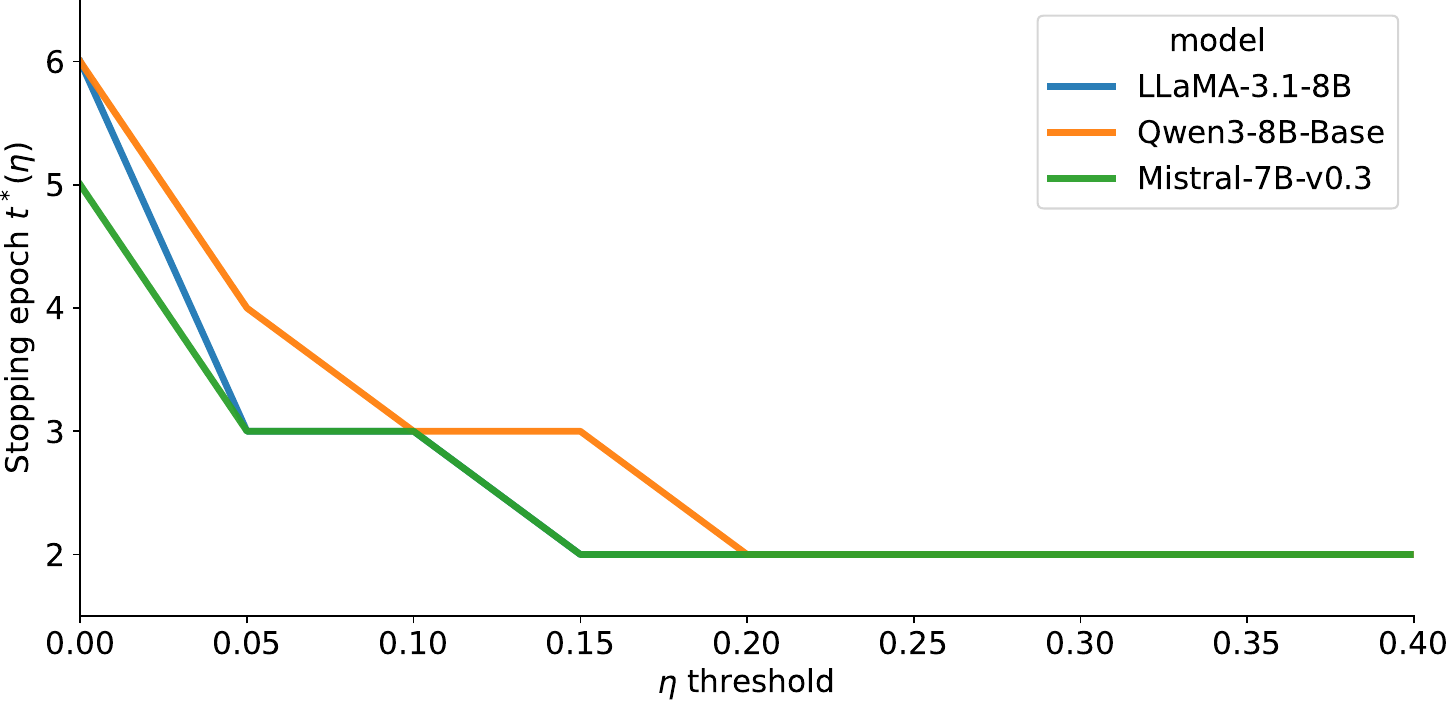} &
\includegraphics[width=0.48\linewidth]{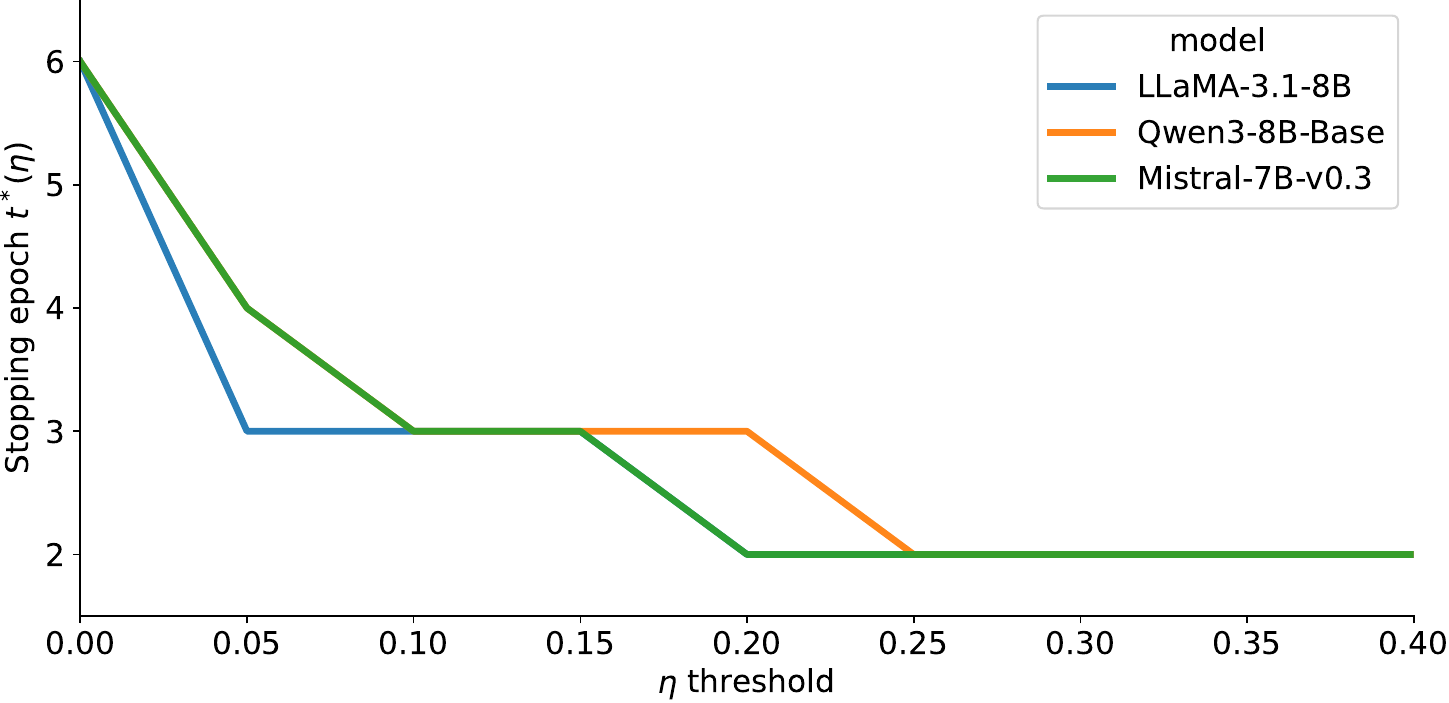} \\

\small \texttt{SA:IMDB} & \small \texttt{QA:MMLU}
\end{tabular}

\caption{Topology-based stopping epochs under full fine-tuning. The plots show selected stopping epochs across $\eta$ thresholds for \texttt{SA:IMDB} and \texttt{QA:MMLU} across the three model families.}
\label{fig:eta-stop-datasets}
\end{figure}
\FloatBarrier
\section{Cross-task Representation Geometry}
\label{app:cross-task-geometry}

For each model, projection $p\in\{K,Q,V,O\}$, layer $\ell$, and dataset pair $(D_a,D_b)$, we compute the Wasserstein distance between their final-epoch persistence diagrams:
\[
d_{p,\ell}(D_a,D_b)
=
W\!\left(PD^{(6)}_{D_a,p,\ell},PD^{(6)}_{D_b,p,\ell}\right).
\]
We then normalize this distance within each layer and projection:
\[
s_{p,\ell}(D_a,D_b)
=
1-
\frac{d_{p,\ell}(D_a,D_b)}
{\max_{D_i,D_j} d_{p,\ell}(D_i,D_j)}.
\]
Thus, $1$ indicates identical diagrams and $0$ indicates the most different dataset pair for that layer and projection. We average $s_{p,\ell}$ over layers to obtain the reported cross-dataset similarity.

Tables~\ref{tab:cross-dataset-sim-llama}--\ref{tab:cross-dataset-sim-mistral} report Wasserstein similarities between the task-induced representations. Across models, \texttt{SA:IMDB}, \texttt{SA:SST-2}, and \texttt{QA:MMLU} show higher mutual similarity than their similarities with \texttt{QA:GSM8K}. This separation is most visible for\texttt{Qwen3-8B-Base}, where \texttt{SA:IMDB}--\texttt{QA:MMLU} reaches $0.50/0.44/0.45/0.38$ across $K/Q/V/O$, while \texttt{SA:SST-2}--\texttt{QA:GSM8K} is near zero. \texttt{LLaMA-3.1-8B} and \texttt{Mistral-7B-v0.3} show the same pattern: NLU-style tasks are more aligned with one another, whereas \texttt{QA:GSM8K} remains the least aligned task in most projection families. These results suggest that mathematical reasoning induces a different topological adaptation pattern than the sentiment and general QA-style tasks.

\begin{table}[!htbp]
\centering
\small
\caption{Cross-dataset topological similarity for full-finetuned \texttt{LLaMA-3.1-8B}. Each entry reports normalized Wasserstein similarity between final-epoch persistence diagrams, averaged over 32 layers for the $K/Q/V/O$ projections.}
\setlength{\tabcolsep}{4pt}
\resizebox{\textwidth}{!}{
\begin{tabular}{lcccc}
\toprule
 & \texttt{SA:IMDB} $(K/Q/V/O)$ & \texttt{SA:SST-2} $(K/Q/V/O)$ & \texttt{QA:MMLU} $(K/Q/V/O)$ & \texttt{QA:GSM8K} $(K/Q/V/O)$ \\
\midrule
\texttt{SA:IMDB} 
& $1/1/1/1$ &  &  &  \\
\texttt{SA:SST-2}
& $0.36/0.29/0.33/0.28$ 
& $1/1/1/1$ &  &  \\
\texttt{QA:MMLU} 
& $0.14/0.15/0.13/0.12$ 
& $0.23/0.19/0.18/0.16$ 
& $1/1/1/1$ &  \\
\texttt{QA:GSM8K} 
& $0.10/0.11/0.09/0.08$ 
& $0.17/0.10/0.12/0.11$ 
& $0.01/0.07/0.02/0.02$ 
& $1/1/1/1$ \\
\bottomrule
\end{tabular}}
\label{tab:cross-dataset-sim-llama}
\end{table}

\begin{table}[!htbp]
\centering
\small
\caption{Cross-dataset topological similarity for full-finetuned \texttt{Qwen3-8B-Base}. Each entry reports normalized Wasserstein similarity between final-epoch persistence diagrams, averaged over 36 layers for the $K/Q/V/O$ projections.}
\setlength{\tabcolsep}{4pt}
\resizebox{\textwidth}{!}{
\begin{tabular}{lcccc}
\toprule
 & \texttt{SA:IMDB} $(K/Q/V/O)$ & \texttt{SA:SST-2} $(K/Q/V/O)$ & \texttt{QA:MMLU} $(K/Q/V/O)$ & \texttt{QA:GSM8K} $(K/Q/V/O)$ \\
\midrule
\texttt{SA:IMDB} 
& $1/1/1/1$ &  &  &  \\
\texttt{SA:SST-2}
& $0.41/0.35/0.38/0.33$ 
& $1/1/1/1$ &  &  \\
\texttt{QA:MMLU} 
& $0.50/0.44/0.45/0.38$ 
& $0.36/0.31/0.33/0.28$ 
& $1/1/1/1$ &  \\
\texttt{QA:GSM8K} 
& $0.06/0.05/0.05/0.05$ 
& $0.01/0.00/0.00/0.01$ 
& $0.06/0.05/0.05/0.05$ 
& $1/1/1/1$ \\
\bottomrule
\end{tabular}}
\label{tab:cross-dataset-sim-qwen}
\end{table}

\begin{table}[!htbp]
\centering
\small
\caption{Cross-dataset topological similarity for full-finetuned \texttt{Mistral-7B-v0.3}. Each entry reports normalized Wasserstein similarity between final-epoch persistence diagrams, averaged over 32 layers for the $K/Q/V/O$ projections.}
\setlength{\tabcolsep}{4pt}
\resizebox{\textwidth}{!}{
\begin{tabular}{lcccc}
\toprule
 & \texttt{SA:IMDB} $(K/Q/V/O)$ & \texttt{SA:SST-2} $(K/Q/V/O)$ & \texttt{QA:MMLU} $(K/Q/V/O)$ & \texttt{QA:GSM8K} $(K/Q/V/O)$ \\
\midrule
\texttt{SA:IMDB} 
& $1/1/1/1$ &  &  &  \\
\texttt{SA:SST-2}
& $0.29/0.27/0.32/0.27$ 
& $1/1/1/1$ &  &  \\
\texttt{QA:MMLU} 
& $0.12/0.12/0.13/0.12$ 
& $0.17/0.14/0.18/0.15$ 
& $1/1/1/1$ &  \\
\texttt{QA:GSM8K} 
& $0.07/0.04/0.08/0.05$ 
& $0.11/0.05/0.12/0.10$ 
& $0.02/0.02/0.01/0.03$ 
& $1/1/1/1$ \\
\bottomrule
\end{tabular}}
\label{tab:cross-dataset-sim-mistral}
\end{table}

\FloatBarrier
\section{Catastrophic Forgetting Evaluation}
\label{app:catastrophic-forgetting}



The retention results show a second advantage of selective tuning. Tables~\ref{tab:forgetting-llama-11datasets}~\ref{tab:forgetting-mistral-11datasets} evaluate the models on 11 different datasets after prior fine-tuning on \texttt{QA:GSM8K}. The clearest retention gains appear for \texttt{LLaMA-3.1-8B} and \texttt{Qwen3-8B-Base}. For both models, \method improves over LoRA on \texttt{QA:MMLU}, \texttt{SA:SST-2}, \texttt{SA:IMDB}, \texttt{IR:SQuAD}, and \texttt{IR:HotpotQA}, showing stronger preservation on question-answering, sentiment-analysis, and information-retrieval evaluations.

\texttt{Qwen3-8B-Base} shows a more mixed pattern. LoRA remains stronger on several generation-style metrics, especially summarization, instruction following, and code generation. However, \method stays very close to LoRA on most metrics and still improves over LoRA on \texttt{SA:IMDB} and \texttt{IR:HotpotQA}. Overall, these results suggest that selective fine-tuning reduces forgetting most clearly on question-answering, sentiment-analysis, and information-retrieval tasks, while its effect on generation-style evaluations depend on architecture.

\begin{table*}[!htbp]
\centering
\scriptsize
\caption{Catastrophic forgetting evaluation on \texttt{QA:GSM8K} fine-tuned \texttt{LLaMA-3.1-8B}. Metrics are avg.\ $\pm$ std over three fixed, non-overlapping evaluation subsets. Base and full finetuned models are shown as reference rows. \textbf{Bold} indicates the better result between LoRA and \method for each metric.}
\label{tab:forgetting-llama-11datasets}
\setlength{\tabcolsep}{2pt}

\resizebox{\textwidth}{!}{
\begin{tabular}{l|c|cc|cccc}
\toprule
& \multicolumn{1}{c|}{\textit{Question Answering}}
& \multicolumn{2}{c|}{\textit{Sentiment Analysis}}
& \multicolumn{4}{c}{\textit{Information Retrieval}} \\
& \multicolumn{1}{c|}{\texttt{QA:MMLU}}
& \multicolumn{1}{c}{\texttt{SA:SST-2}}
& \multicolumn{1}{c|}{\texttt{SA:IMDB}}
& \multicolumn{2}{c|}{\texttt{IR:SQuAD v1.1}}
& \multicolumn{2}{c}{\texttt{IR:HotpotQA}} \\
Model
& Acc.
& Acc. & Acc.
& EM & F1
& EM & F1 \\
\midrule

\graycell{Base model}
& \graycell{$54.5\pm1.25$}
& \graycell{$94.0\pm0.72$}
& \graycell{$89.8\pm0.38$}
& \graycell{$76.0\pm0.32$}
& \graycell{$86.9\pm0.20$}
& \graycell{$23.1\pm1.25$}
& \graycell{$31.9\pm0.64$} \\

\graycell{Full finetuned}
& \graycell{$52.9\pm0.44$}
& \graycell{$94.3\pm0.80$}
& \graycell{$89.0\pm0.47$}
& \graycell{$74.4\pm0.39$}
& \graycell{$85.0\pm0.15$}
& \graycell{$21.3\pm1.40$}
& \graycell{$29.4\pm0.76$} \\

\midrule

LoRA
& $51.8\pm1.01$
& $92.7\pm0.52$
& $89.3\pm0.48$
& $72.3\pm0.49$
& $83.6\pm0.66$
& $19.2\pm1.03$
& $27.2\pm0.97$ \\

\method
& $\mathbf{53.6\pm0.51}$
& $\mathbf{94.4\pm0.19}$
& $\mathbf{90.3\pm0.36}$
& $\mathbf{76.9\pm0.24}$
& $\mathbf{86.0\pm0.48}$
& $\mathbf{21.8\pm1.12}$
& $\mathbf{30.2\pm0.71}$ \\

\bottomrule
\end{tabular}
}

\vspace{1em}

\resizebox{\textwidth}{!}{
\begin{tabular}{l|ccc|ccc|cc}
\toprule
& \multicolumn{6}{c|}{\textit{Summarization}}
& \multicolumn{2}{c}{\textit{Code Generation}} \\
& \multicolumn{3}{c|}{\texttt{SU:XSum}}
& \multicolumn{3}{c|}{\texttt{SU:CNN/DM}}
& \multicolumn{1}{c|}{\texttt{CG:HumanEval}}
& \multicolumn{1}{c}{\texttt{CG:MBPP}} \\
Model
& R1 & R2 & RL
& R1 & R2 & RL
& pass@1 & pass@1 \\
\midrule

\graycell{Base model}
& \graycell{$25.2\pm0.05$}
& \graycell{$6.7\pm0.08$}
& \graycell{$18.8\pm0.10$}
& \graycell{$31.7\pm0.29$}
& \graycell{$13.6\pm0.18$}
& \graycell{$20.6\pm0.23$}
& \graycell{$0.0\pm0.00$}
& \graycell{$2.7\pm1.35$} \\

\graycell{Full finetuned}
& \graycell{$14.8\pm0.16$}
& \graycell{$4.1\pm0.11$}
& \graycell{$11.7\pm0.17$}
& \graycell{$29.5\pm0.34$}
& \graycell{$12.4\pm0.24$}
& \graycell{$18.9\pm0.28$}
& \graycell{$1.8\pm1.85$}
& \graycell{$3.5\pm3.07$} \\

\midrule

LoRA
& $\mathbf{27.2\pm0.21}$
& $\mathbf{7.8\pm0.22}$
& $\mathbf{21.6\pm0.35}$
& $\mathbf{29.8\pm0.27}$
& $\mathbf{12.6\pm0.22}$
& $\mathbf{19.2\pm0.19}$
& $1.2\pm1.05$
& $\mathbf{4.7\pm3.48}$ \\

\method
& $\mathbf{27.2\pm0.17}$
& $7.7\pm0.08$
& $21.0\pm0.21$
& $27.4\pm0.45$
& $11.8\pm0.25$
& $17.7\pm0.34$
& $\mathbf{1.8\pm3.21}$
& $3.1\pm1.33$ \\

\bottomrule
\end{tabular}
}

\vspace{1em}

\resizebox{\textwidth}{!}{
\begin{tabular}{l|cccc|cccc}
\toprule
& \multicolumn{8}{c}{\textit{Instruction Following}} \\
& \multicolumn{4}{c|}{\texttt{IF:Dolly-15k}}
& \multicolumn{4}{c}{\texttt{IF:Alpaca}} \\
Model
& R1 & R2 & RL & METEOR
& R1 & R2 & RL & METEOR \\
\midrule

\graycell{Base model}
& \graycell{$38.0\pm0.32$}
& \graycell{$19.3\pm0.25$}
& \graycell{$29.9\pm0.25$}
& \graycell{$27.2$}
& \graycell{$32.6\pm0.27$}
& \graycell{$15.0\pm0.25$}
& \graycell{$26.1\pm0.18$}
& \graycell{$25.7$} \\

\graycell{Full finetuned}
& \graycell{$38.5\pm0.17$}
& \graycell{$19.4\pm0.14$}
& \graycell{$30.5\pm0.09$}
& \graycell{$27.1$}
& \graycell{$39.4\pm0.21$}
& \graycell{$19.0\pm0.28$}
& \graycell{$32.3\pm0.17$}
& \graycell{$27.2$} \\

\midrule

LoRA
& $\mathbf{38.0\pm0.31}$
& $\mathbf{19.5\pm0.14}$
& $\mathbf{30.5\pm0.20}$
& $27.0$
& $\mathbf{37.0\pm0.15}$
& $\mathbf{18.2\pm0.16}$
& $\mathbf{30.7\pm0.13}$
& $\mathbf{26.0}$ \\

\method
& $37.2\pm0.09$
& $18.9\pm0.26$
& $29.4\pm0.18$
& $\mathbf{27.5}$
& $31.9\pm0.27$
& $14.3\pm0.27$
& $25.7\pm0.28$
& $25.5$ \\

\bottomrule
\end{tabular}
}

\end{table*}

\begin{table*}[!htbp]
\centering
\scriptsize
\caption{Catastrophic forgetting evaluation on \texttt{QA:GSM8K} fine-tuned \texttt{Qwen3-8B-Base}. Metrics are avg.\ $\pm$ std over three fixed, non-overlapping evaluation subsets. Base and full finetuned models are shown as reference rows. \textbf{Bold} indicates the better result between LoRA and \method for each metric.}
\label{tab:forgetting-qwen-11datasets}
\setlength{\tabcolsep}{2pt}

\resizebox{\textwidth}{!}{
\begin{tabular}{l|c|cc|cccc}
\toprule
& \multicolumn{1}{c|}{\textit{Question Answering}}
& \multicolumn{2}{c|}{\textit{Sentiment Analysis}}
& \multicolumn{4}{c}{\textit{Information Retrieval}} \\
& \multicolumn{1}{c|}{\texttt{QA:MMLU}}
& \multicolumn{1}{c}{\texttt{SA:SST-2}}
& \multicolumn{1}{c|}{\texttt{SA:IMDB}}
& \multicolumn{2}{c|}{\texttt{IR:SQuAD v1.1}}
& \multicolumn{2}{c}{\texttt{IR:HotpotQA}} \\
Model
& Acc.
& Acc. & Acc.
& EM & F1
& EM & F1 \\
\midrule

\graycell{Base model}
& \graycell{$64.2\pm0.78$}
& \graycell{$94.0\pm0.79$}
& \graycell{$90.0\pm0.26$}
& \graycell{$67.2\pm0.57$}
& \graycell{$82.9\pm0.38$}
& \graycell{$21.7\pm0.94$}
& \graycell{$30.4\pm0.60$} \\

\graycell{Full finetuned}
& \graycell{$64.4\pm0.66$}
& \graycell{$93.9\pm0.99$}
& \graycell{$89.6\pm0.25$}
& \graycell{$75.2\pm0.19$}
& \graycell{$87.3\pm0.44$}
& \graycell{$21.9\pm0.40$}
& \graycell{$30.4\pm0.05$} \\

\midrule

LoRA
& $\mathbf{64.4\pm0.58}$
& $\mathbf{94.4\pm0.79}$
& $89.9\pm0.35$
& $\mathbf{77.8\pm0.19}$
& $\mathbf{88.7\pm0.20}$
& $22.0\pm0.39$
& $30.6\pm0.24$ \\

\method
& $64.3\pm0.63$
& $93.8\pm0.91$
& $\mathbf{90.3\pm0.22}$
& $74.4\pm0.25$
& $86.9\pm0.36$
& $\mathbf{22.2\pm0.64}$
& $\mathbf{30.7\pm0.16}$ \\

\bottomrule
\end{tabular}
}

\vspace{1em}

\resizebox{\textwidth}{!}{
\begin{tabular}{l|ccc|ccc|cc}
\toprule
& \multicolumn{6}{c|}{\textit{Summarization}}
& \multicolumn{2}{c}{\textit{Code Generation}} \\
& \multicolumn{3}{c|}{\texttt{SU:XSum}}
& \multicolumn{3}{c|}{\texttt{SU:CNN/DM}}
& \multicolumn{1}{c|}{\texttt{CG:HumanEval}}
& \multicolumn{1}{c}{\texttt{CG:MBPP}} \\
Model
& R1 & R2 & RL
& R1 & R2 & RL
& pass@1 & pass@1 \\
\midrule

\graycell{Base model}
& \graycell{$25.9\pm0.09$}
& \graycell{$6.4\pm0.04$}
& \graycell{$18.2\pm0.03$}
& \graycell{$38.5\pm0.13$}
& \graycell{$16.2\pm0.08$}
& \graycell{$24.8\pm0.09$}
& \graycell{$6.1\pm1.13$}
& \graycell{$0.4\pm0.67$} \\

\graycell{Full finetuned}
& \graycell{$25.8\pm0.08$}
& \graycell{$6.4\pm0.04$}
& \graycell{$18.1\pm0.07$}
& \graycell{$38.6\pm0.13$}
& \graycell{$16.2\pm0.10$}
& \graycell{$24.8\pm0.08$}
& \graycell{$4.9\pm2.18$}
& \graycell{$0.8\pm0.68$} \\

\midrule

LoRA
& $\mathbf{26.3\pm0.08}$
& $\mathbf{6.6\pm0.07}$
& $\mathbf{18.5\pm0.04}$
& $\mathbf{38.6\pm0.06}$
& $\mathbf{16.3\pm0.09}$
& $\mathbf{24.9\pm0.08}$
& $\mathbf{8.6\pm2.22}$
& $\mathbf{1.2\pm1.18}$ \\

\method
& $26.0\pm0.04$
& $6.5\pm0.02$
& $18.3\pm0.05$
& $\mathbf{38.6\pm0.16}$
& $\mathbf{16.3\pm0.12}$
& $\mathbf{24.9\pm0.09}$
& $7.3\pm3.74$
& $0.8\pm0.68$ \\

\bottomrule
\end{tabular}
}

\vspace{1em}

\resizebox{\textwidth}{!}{
\begin{tabular}{l|cccc|cccc}
\toprule
& \multicolumn{8}{c}{\textit{Instruction Following}} \\
& \multicolumn{4}{c|}{\texttt{IF:Dolly-15k}}
& \multicolumn{4}{c}{\texttt{IF:Alpaca}} \\
Model
& R1 & R2 & RL & METEOR
& R1 & R2 & RL & METEOR \\
\midrule

\graycell{Base model}
& \graycell{$36.8\pm0.21$}
& \graycell{$17.6\pm0.16$}
& \graycell{$27.5\pm0.16$}
& \graycell{$27.0$}
& \graycell{$37.6\pm0.22$}
& \graycell{$17.9\pm0.20$}
& \graycell{$28.8\pm0.13$}
& \graycell{$29.2$} \\

\graycell{Full finetuned}
& \graycell{$37.2\pm0.29$}
& \graycell{$18.0\pm0.17$}
& \graycell{$27.9\pm0.20$}
& \graycell{$27.5$}
& \graycell{$37.3\pm0.18$}
& \graycell{$17.5\pm0.15$}
& \graycell{$28.5\pm0.06$}
& \graycell{$28.9$} \\

\midrule

LoRA
& $\mathbf{38.0\pm0.33}$
& $\mathbf{18.7\pm0.14}$
& $\mathbf{28.7\pm0.17}$
& $\mathbf{28.0}$
& $\mathbf{38.3\pm0.22}$
& $\mathbf{18.3\pm0.26}$
& $\mathbf{29.4\pm0.17}$
& $\mathbf{29.5}$ \\

\method
& $37.3\pm0.24$
& $18.1\pm0.16$
& $28.0\pm0.14$
& $27.3$
& $37.8\pm0.26$
& $18.0\pm0.31$
& $29.0\pm0.21$
& $29.3$ \\

\bottomrule
\end{tabular}
}

\end{table*}

\begin{table*}[!htbp]
\centering
\scriptsize
\caption{Catastrophic forgetting evaluation on \texttt{QA:GSM8K} fine-tuned \texttt{Mistral-7B-v0.3}. Metrics are avg.\ $\pm$ std over three fixed, non-overlapping evaluation subsets. Base and full finetuned models are shown as reference rows. \textbf{Bold} indicates the better result between LoRA and \method for each metric.}
\label{tab:forgetting-mistral-11datasets}
\setlength{\tabcolsep}{2pt}

\resizebox{\textwidth}{!}{
\begin{tabular}{l|c|cc|cccc}
\toprule
& \multicolumn{1}{c|}{\textit{Question Answering}}
& \multicolumn{2}{c|}{\textit{Sentiment Analysis}}
& \multicolumn{4}{c}{\textit{Information Retrieval}} \\
& \multicolumn{1}{c|}{\texttt{QA:MMLU}}
& \multicolumn{1}{c}{\texttt{SA:SST-2}}
& \multicolumn{1}{c|}{\texttt{SA:IMDB}}
& \multicolumn{2}{c|}{\texttt{IR:SQuAD v1.1}}
& \multicolumn{2}{c}{\texttt{IR:HotpotQA}} \\
Model
& Acc.
& Acc. & Acc.
& EM & F1
& EM & F1 \\
\midrule

\graycell{Base model}
& \graycell{$54.5\pm0.41$}
& \graycell{$89.0\pm0.61$}
& \graycell{$61.4\pm0.48$}
& \graycell{$71.3\pm1.02$}
& \graycell{$83.9\pm0.72$}
& \graycell{$23.6\pm1.08$}
& \graycell{$31.8\pm0.80$} \\

\graycell{Full finetuned}
& \graycell{$52.9\pm0.33$}
& \graycell{$89.0\pm1.84$}
& \graycell{$60.2\pm0.50$}
& \graycell{$73.2\pm0.64$}
& \graycell{$84.4\pm0.54$}
& \graycell{$21.6\pm1.28$}
& \graycell{$29.0\pm0.83$} \\

\midrule

LoRA
& $50.2\pm0.79$
& $70.6\pm2.81$
& $59.2\pm0.56$
& $69.8\pm0.98$
& $81.6\pm0.80$
& $18.7\pm0.88$
& $26.5\pm0.66$ \\

\method
& $\mathbf{53.2\pm0.33}$
& $\mathbf{89.9\pm0.86}$
& $\mathbf{60.6\pm0.56}$
& $\mathbf{70.2\pm0.74}$
& $\mathbf{83.0\pm0.63}$
& $\mathbf{22.0\pm1.34}$
& $\mathbf{29.9\pm0.86}$ \\

\bottomrule
\end{tabular}
}

\vspace{1em}

\resizebox{\textwidth}{!}{
\begin{tabular}{l|ccc|ccc|cc}
\toprule
& \multicolumn{6}{c|}{\textit{Summarization}}
& \multicolumn{2}{c}{\textit{Code Generation}} \\
& \multicolumn{3}{c|}{\texttt{SU:XSum}}
& \multicolumn{3}{c|}{\texttt{SU:CNN/DM}}
& \multicolumn{1}{c|}{\texttt{CG:HumanEval}}
& \multicolumn{1}{c}{\texttt{CG:MBPP}} \\
Model
& R1 & R2 & RL
& R1 & R2 & RL
& pass@1 & pass@1 \\
\midrule

\graycell{Base model}
& \graycell{$22.2\pm0.10$}
& \graycell{$6.2\pm0.12$}
& \graycell{$16.4\pm0.17$}
& \graycell{$39.5\pm0.16$}
& \graycell{$17.4\pm0.14$}
& \graycell{$25.0\pm0.14$}
& \graycell{$1.8\pm1.82$}
& \graycell{$3.5\pm1.15$} \\

\graycell{Full finetuned}
& \graycell{$22.2\pm0.16$}
& \graycell{$7.2\pm0.11$}
& \graycell{$16.9\pm0.21$}
& \graycell{$39.6\pm0.18$}
& \graycell{$17.5\pm0.20$}
& \graycell{$25.0\pm0.14$}
& \graycell{$1.8\pm1.85$}
& \graycell{$1.6\pm1.77$} \\

\midrule

LoRA
& $\mathbf{29.2\pm0.23}$
& $\mathbf{8.9\pm0.08}$
& $\mathbf{22.4\pm0.24}$
& $36.5\pm0.18$
& $15.5\pm0.22$
& $23.6\pm0.14$
& $0.0\pm0.00$
& $0.4\pm0.67$ \\

\method
& $27.7\pm0.14$
& $8.5\pm0.15$
& $20.6\pm0.26$
& $\mathbf{39.8\pm0.21}$
& $\mathbf{17.6\pm0.22}$
& $\mathbf{25.1\pm0.16}$
& $\mathbf{2.4\pm1.04}$
& $\mathbf{1.9\pm1.34}$ \\

\bottomrule
\end{tabular}
}

\vspace{1em}

\resizebox{\textwidth}{!}{
\begin{tabular}{l|cccc|cccc}
\toprule
& \multicolumn{8}{c}{\textit{Instruction Following}} \\
& \multicolumn{4}{c|}{\texttt{IF:Dolly-15k}}
& \multicolumn{4}{c}{\texttt{IF:Alpaca}} \\
Model
& R1 & R2 & RL & METEOR
& R1 & R2 & RL & METEOR \\
\midrule

\graycell{Base model}
& \graycell{$28.7\pm0.21$}
& \graycell{$12.8\pm0.05$}
& \graycell{$21.7\pm0.13$}
& \graycell{$24.4$}
& \graycell{$26.8\pm0.23$}
& \graycell{$11.3\pm0.18$}
& \graycell{$20.0\pm0.15$}
& \graycell{$24.2$} \\

\graycell{Full finetuned}
& \graycell{$35.5\pm0.13$}
& \graycell{$17.3\pm0.02$}
& \graycell{$27.9\pm0.07$}
& \graycell{$26.7$}
& \graycell{$32.8\pm0.08$}
& \graycell{$14.8\pm0.04$}
& \graycell{$25.4\pm0.02$}
& \graycell{$26.5$} \\

\midrule

LoRA
& $\mathbf{34.7\pm0.29}$
& $\mathbf{17.1\pm0.15}$
& $\mathbf{28.2\pm0.12}$
& $24.7$
& $\mathbf{35.0\pm0.14}$
& $\mathbf{16.3\pm0.19}$
& $\mathbf{28.3\pm0.10}$
& $\mathbf{25.4}$ \\

\method
& $33.4\pm0.35$
& $16.2\pm0.09$
& $26.4\pm0.17$
& $\mathbf{26.4}$
& $29.1\pm0.25$
& $12.8\pm0.15$
& $22.3\pm0.18$
& $25.2$ \\

\bottomrule
\end{tabular}
}

\end{table*}

Taken together, the experiments support four claims. First, fine-tuning changes are small in magnitude but structured in location, which makes naive layer-order freezing ineffective. Second, Wasserstein ranking and norm ranking each recover useful selective-tuning configurations, but neither signal alone induces a universally dominant ordering. Third, topology reveals cross-task geometry that norm does not express, with \texttt{QA:GSM8K} separated cleanly from \texttt{SA:IMDB}, \texttt{SA:SST-2}, and \texttt{QA:MMLU} in Wasserstein space. Fourth, selective $V/O$ tuning can preserve substantial target-task performance while reducing trainable parameters and retaining more out-of-domain capability than full fine-tuning. This combination of efficiency, retention, and task-geometry analysis is the central empirical message of the paper.

\FloatBarrier
\section{Overfitting Analysis}
\label{app:overfitting}

We analyze overfitting behavior across full fine-tuning, LoRA, TDA-High3, TDA-Low3, Eltwise-High3, and Eltwise-Low3 for \texttt{LLaMA-3.1-8B} Table~\ref{tab:overfitting-results} reports the corresponding generalization gaps. For \texttt{LLaMA-3.1-8B}, full fine-tuning obtains the highest final test accuracy, but it also shows the largest train-test gap. In contrast, TDA-High3 provides the best selective test accuracy while yielding the smallest train-test gap. Overall, these results support the \method can act as a structural regularizer, preserving most of the useful adaptation while avoiding the stronger overfitting behavior observed in full fine-tuning.

\begin{table*}[!htbp]
\centering
\scriptsize
\caption{Overfitting analysis on \texttt{LLaMA-3.1-8B} across fine-tuning methods. Train-validation and train-test gaps are computed from epoch-6 accuracy values. \textbf{Bold} indicates the best result among LoRA, Eltwise, and TDA variants. Smaller gaps indicate better generalization stability.}
\label{tab:overfitting-results}
\setlength{\tabcolsep}{6pt}
\renewcommand{\arraystretch}{1.08}
\begin{tabular}{lccccc}
\toprule
Method & Train & Val & Test & Train-Val Gap & Train-Test Gap \\
\midrule
\multicolumn{6}{l}{\textit{LLaMA-3.1-8B}} \\
\graycell{Full FT}
& \graycell{$71.61$} & \graycell{$58.74$} & \graycell{$59.59$} & \graycell{$12.87$} & \graycell{$12.02$} \\
LoRA
& $62.50$ & $\mathbf{58.51}$ & $57.32$ & $\mathbf{3.99}$ & $5.18$ \\
Eltwise-High3
& $61.62$ & $56.08$ & $\mathbf{58.07}$ & $5.54$ & $3.55$ \\
Eltwise-Low3
& $62.03$ & $55.62$ & $56.79$ & $6.41$ & $5.24$ \\
TDA-High3
& $61.15$ & $56.53$ & $\mathbf{58.07}$ & $4.61$ & $\mathbf{3.08}$ \\
TDA-Low3
& $\mathbf{62.53}$ & $57.14$ & $57.32$ & $5.39$ & $5.21$ \\

\bottomrule
\end{tabular}
\end{table*}

\FloatBarrier

\section{Proofs}
\label{sec:proofs}

\xhdr{Stability motivation}
The usefulness of $\Delta_i$ rests on the standard stability principle of persistent homology: small perturbations in the underlying metric induce small perturbations in the resulting diagrams. In our setting, this means that if fine-tuning changes pairwise distances between rows only slightly, then the corresponding topological drift must also remain small. The next proposition makes this statement explicit.

\begin{proposition}[Stability of topological sensitivity]
\label{def:app-sensitivity}
Let $X=\{x_1,\dots,x_d\}$ and $\tilde X=\{\tilde x_1,\dots,\tilde x_d\}$ be the row clouds of $W_i$ and $\tilde W_i$, indexed consistently. Let $d$ and $\tilde d$ be the metrics used to construct the Vietoris--Rips filtrations of $X$ and $\tilde X$. If $\|d-\tilde d\|_\infty := \max_{a,b}|d(a,b)-\tilde d(a,b)| \le \eta$, then for every $k \ge 0$ we have $d_B(PD_k(X),PD_k(\tilde X)) \le \eta$, and therefore $W_p(PD_k(X),PD_k(\tilde X)) \le N_k^{1/p}\eta$, where $d_B$ is the bottleneck distance and $N_k$ is the maximum number of off-diagonal points in the two diagrams.
\end{proposition}

\begin{proof}
\label{proof:proposition}
For any simplex $S \subseteq \{1,\dots,d\}$, let $\phi(S)=\max_{a,b\in S} d(a,b)$ and $\tilde\phi(S)=\max_{a,b\in S}\tilde d(a,b)$. The assumption $\|d-\tilde d\|_\infty \le \eta$ implies $|\phi(S)-\tilde\phi(S)| \le \eta$ for every simplex $S$. Hence $\mathrm{VR}_\varepsilon(X) \subseteq \mathrm{VR}_{\varepsilon+\eta}(\tilde X)$ and $\mathrm{VR}_\varepsilon(\tilde X) \subseteq \mathrm{VR}_{\varepsilon+\eta}(X)$ for all $\varepsilon \ge 0$, so the two filtrations are $\eta$-interleaved. By the stability theorem for persistence diagrams under interleavings \cite{chazal2016structure}, it follows that $d_B(PD_k(X),PD_k(\tilde X)) \le \eta$. Finally, a bottleneck matching with cost $d_B$ moves at most $N_k$ off-diagonal points by at most $d_B$ in $\|\cdot\|_\infty$, so the total $p$-cost is at most $N_k d_B^p$. Taking the $p$th root yields $W_p(PD_k(X),PD_k(\tilde X)) \le N_k^{1/p} d_B \le N_k^{1/p}\eta$.
\end{proof}

We use this stability result in the next section to justify when a prior-task topological profile remains stable under transfer to a new task.

We consider both highest-drift and lowest-drift selections. For a budget $b$, we define $\mathcal F_{\mathrm{prior}}^{\mathrm{TDA\text{-}H}}(b)$ as the $b$ matrices with the largest values of $S_i$, and $\mathcal F_{\mathrm{prior}}^{\mathrm{TDA\text{-}L}}(b)$ as the $b$ matrices with the smallest values.  In the magnitude-based baseline, we analogously rank matrices using $M_i$. These selected matrices are frozen from the start of fine-tuning on the target task $\mathcal T^\ast$. This setting tests whether structural drift observed in one prior dataset can serve as a reusable fine-tuning profile for selective adaptation to other datasets, and whether topology-based rankings provide a different freezing signal from simple relative weight changes.

\begin{theorem}[Stability of reusable freezing profiles]
\label{thm:app-profile-stability}
Fix a pretrained transformer $f_{\theta^{(0)}}$ and a monitored matrix set $\mathcal{I}$. Let $\mathcal{T}$ be a prior task and let $\mathcal{T}^{\ast}$ be a target task. For each task, let $\theta_{\mathcal{T}}^{(t)}=\{W_{i,\mathcal{T}}^{(t)}\}_{i=1}^m$ and $\theta_{\mathcal{T}^{\ast}}^{(t)}=\{W_{i,\mathcal{T}^{\ast}}^{(t)}\}_{i=1}^m$ denote the fine-tuned parameters after $t$ optimization steps, both initialized at $\theta^{(0)}$. Let $L_{\mathcal{T}}(\theta)$ and $L_{\mathcal{T}^{\ast}}(\theta)$ be the corresponding population fine-tuning objectives, and write $G_{\mathcal{T}}(\theta)=\nabla_\theta L_{\mathcal{T}}(\theta)$ and $G_{\mathcal{T}^{\ast}}(\theta)=\nabla_\theta L_{\mathcal{T}^{\ast}}(\theta)$. Assume both trajectories remain in a ball $B(\theta^{(0)},R)$, both gradient fields are $L$-Lipschitz on this ball, and $\sup_{\theta\in B(\theta^{(0)},R)}\|G_{\mathcal{T}}(\theta)-G_{\mathcal{T}^{\ast}}(\theta)\|_2\leq\epsilon$. If both tasks are fine-tuned with the same gradient descent update and step size $\alpha$, then for every $t\leq t_{\max}$ we have $\|\theta_{\mathcal{T}}^{(t)}-\theta_{\mathcal{T}^{\ast}}^{(t)}\|_2\leq R_t$, where $R_t:=\alpha\epsilon\sum_{r=0}^{t-1}(1+\alpha L)^r$.

Assume that the row-cloud metric $\delta_i$ is fixed across tasks and stable under row perturbations. For each $i\in\mathcal{I}$ and fixed homological dimension $k$, define $D_{i,\mathcal{T}}^{(t)}:=\W{ass}_p(\mathrm{PD}_k(X_i^{(0)}),\mathrm{PD}_k(X_{i,\mathcal{T}}^{(t)}))$ and $D_{i,\mathcal{T}^{\ast}}^{(t)}:=\W{ass}_p(\mathrm{PD}_k(X_i^{(0)}),\mathrm{PD}_k(X_{i,\mathcal{T}^{\ast}}^{(t)}))$. Then $|D_{i,\mathcal{T}}^{(t)}-D_{i,\mathcal{T}^{\ast}}^{(t)}|\leq 2N_{i,k}^{1/p}R_t$, where $N_{i,k}$ is the maximum number of off-diagonal points in the two persistence diagrams for matrix $i$ in dimension $k$.

For Scenario A, take $k=0$ and define the prior and target ranking scores by $S_i^{\mathrm{prior}}:=D_{i,\mathcal{T}}^{(t_{\max})}$ and $S_i^{\ast}:=D_{i,\mathcal{T}^{\ast}}^{(t_{\max})}$. Let $S_{(1)}^{\mathrm{prior}}\geq \cdots \geq S_{(|\mathcal{I}|)}^{\mathrm{prior}}$ be the sorted prior scores, and let $g_H(b):=S_{(b)}^{\mathrm{prior}}-S_{(b+1)}^{\mathrm{prior}}$ be the high-drift margin for freezing budget $b$. Let $\delta_{\max}:=2\max_{i\in\mathcal{I}}N_{i,0}^{1/p}R_{t_{\max}}$. If $g_H(b)>2\delta_{\max}$, then $F_{\mathrm{prior}}^{\mathrm{TDA-H}}(b)=F_{\ast}^{\mathrm{TDA-H}}(b)$.  The same statement holds for $F_{\mathrm{prior}}^{\mathrm{TDA-L}}(b)$ after sorting the scores in increasing order and using the corresponding low-drift margin.
\end{theorem}

\begin{proof}
\label{proof:theorem}
Let $e_t:=\|\theta_{\mathcal{T}}^{(t)}-\theta_{\mathcal{T}^{\ast}}^{(t)}\|_2$. Since both runs start from the same pretrained model, $e_0=0$. The two update equations are $\theta{\mathcal{T}}^{(t+1)}=\theta_{\mathcal{T}}^{(t)}-\alpha G_{\mathcal{T}}(\theta_{\mathcal{T}}^{(t)})$ and $\theta_{\mathcal{T}^{\ast}}^{(t+1)}=\theta_{\mathcal{T}^{\ast}}^{(t)}-\alpha G_{\mathcal{T}^{\ast}}(\theta_{\mathcal{T}^{\ast}}^{(t)})$. Therefore $e_{t+1}\leq\|\theta_{\mathcal{T}}^{(t)}-\theta_{\mathcal{T}^{\ast}}^{(t)}-\alpha(G_{\mathcal{T}}(\theta_{\mathcal{T}}^{(t)})-G_{\mathcal{T}}(\theta_{\mathcal{T}^{\ast}}^{(t)}))\|_2+\alpha\|G{\mathcal{T}}(\theta_{\mathcal{T}^{\ast}}^{(t)})-G_{\mathcal{T}^{\ast}}(\theta_{\mathcal{T}^{\ast}}^{(t)})\|_2$. By Lipschitzness and task-gradient closeness, $e{t+1}\leq(1+\alpha L)e_t+\alpha\epsilon$. Iterating this recurrence gives $e_t\leq\alpha\epsilon\sum_{r=0}^{t-1}(1+\alpha L)^r=R_t$.

For a monitored matrix $W_i$, the block difference is bounded by the full parameter difference, so $|W_{i,\mathcal{T}}^{(t)}-W_{i,\mathcal{T}^{\ast}}^{(t)}|F\leq R_t$. Hence every row satisfies $\|x_{a,\mathcal{T}}^{(i,t)}-x_{a,\mathcal{T}^{\ast}}^{(i,t)}\|_2\leq R_t$. Since the chosen row-cloud metric is fixed across tasks and stable under row perturbations, we have
$\|\delta_{i,\mathcal{T}}^{(t)}-\delta_{i,\mathcal{T}^{\ast}}^{(t)}\|_\infty \leq 2R_t$.

Applying Proposition~\ref{def:sensitivity} with $\eta=2R_t$ gives $\W{ass}_p(\mathrm{PD}_k(X_{i,\mathcal{T}}^{(t)}),\mathrm{PD}_k(X_{i,\mathcal{T}^{\ast}}^{(t)}))\leq2N_{i,k}^{1/p}R_t$. Since both distances $D_{i,\mathcal{T}}^{(t)}$ and $D_{i,\mathcal{T}^{\ast}}^{(t)}$ are measured from the same pretrained diagram $\mathrm{PD}_k(X_i^{(0)})$, the triangle inequality for $\W{ass}_p$ gives $|D_{i,\mathcal{T}}^{(t)}-D_{i,\mathcal{T}^{\ast}}^{(t)}|\leq\W{ass}_p(\mathrm{PD}_k(X_{i,\mathcal{T}}^{(t)}),\mathrm{PD}_k(X_{i,\mathcal{T}^{\ast}}^{(t)}))\leq2N_{i,k}^{1/p}R_t$.

Set $\delta_{\max}:=2\max_{i\in\mathcal{I}}N_{i,0}^{1/p}R_{t_{\max}}$. Then every Scenario A ranking score changes by at most $\delta_{\max}$ between the prior task and the target task. If the high-drift gap $g_H(b)$ is larger than $2\delta_{\max}$, no matrix outside the prior top-$b$ set can pass a matrix inside that set after moving all scores by at most $\delta_{\max}$. Hence $F_{\mathrm{prior}}^{\mathrm{TDA-H}}(b)=F_{\ast}^{\mathrm{TDA-H}}(b)$. The low-drift claim is identical after sorting in increasing order.
\end{proof}

Theorem~\ref{thm:profile-stability} gives a sufficient condition for reusable freezing profiles while avoiding the claim that all datasets produce identical fine-tuning dynamics. It states that, for a fixed pretrained model $f_{\theta^{(0)}}$, the profile learned on a prior task remains valid on a target task when the two tasks induce close gradient fields in the local fine-tuning region and the prior ranking has a nonzero margin. The fixed model matters because the local Jacobians, attention geometry, and row-cloud geometry are inherited from $\theta^{(0)}$. The dataset changes the coefficients of the update, but under the gradient-field condition it cannot move the monitored matrices far enough to change their topological ranking. Proposition~\ref{def:sensitivity} then converts this parameter-level closeness into persistence-diagram closeness, and the margin condition converts persistence-diagram closeness into equality of the selected freezing profile.

\xhdr{Connection to existing evidence.}
This theorem is consistent with prior evidence that fine-tuning remains constrained near a pretrained initialization. Intrinsic-dimension results show that language-model adaptation can often be described by a low-dimensional update subspace \cite{aghajanyan2021intrinsic}. LoRA shows that strong adaptation can be achieved through low-rank updates to pretrained weight matrices \cite{hu2022lora}, while BitFit shows that changing only bias parameters can be competitive in several language-understanding settings \cite{zaken2022bitfit}. These results support the premise that downstream adaptation does not use the full parameter space uniformly. Task arithmetic and model soups give a more direct precedent for reusing fine-tuning changes from a shared initialization: task vectors behave as reusable weight-space objects, and independently fine-tuned models from the same initialization can often be averaged successfully \cite{ilharcoediting,wortsman2022model}. The topological part of our argument relies on the standard stability theorem for persistence diagrams \cite{chazal2016structure}, and previous neural-network studies have already used persistent homology to compare trained models and layer representations \cite{perez2021characterizing,purvine2023experimental}. In this setting, \method adds a matrix-level conclusion: if two tasks remain in the same local adaptation regime of $f_{\theta^{(0)}}$, then their Wasserstein drift scores remain close, and the same high-drift or low-drift freezing profile can be reused.

\section*{NeurIPS Paper Checklist}

\begin{enumerate}

\item {\bf Claims}
    \item[] Question: Do the main claims made in the abstract and introduction accurately reflect the paper's contributions and scope?
    \item[] Answer: \answerYes{} 
    \item[] Justification: The abstract and introduction state the method, scope, baselines, and empirical outcomes, and these align with the experimental and theoretical results in Sections ~\ref{sec:sec4}-\ref{sec:results}.

\item {\bf Limitations}
    \item[] Question: Does the paper discuss the limitations of the work performed by the authors?
    \item[] Answer: \answerYes{} 
    \item[] Justification: We provide the limitations in Section~\ref{sec:reusable-fine-tuning}

\item {\bf Theory assumptions and proofs}
    \item[] Question: For each theoretical result, does the paper provide the full set of assumptions and a complete (and correct) proof?
    \item[] Answer: \answerYes{} 
    \item[] Justification: Theoretical results include clearly stated assumptions and complete proofs in Sections ~\ref{sec:sec4}-\ref{sec:sec5}, with propositions and theorems properly defined and referenced.

\item {\bf Experimental result reproducibility}
    \item[] Question: Does the paper fully disclose all the information needed to reproduce the main experimental results of the paper to the extent that it affects the main claims and/or conclusions of the paper (regardless of whether the code and data are provided or not)?
    \item[] Answer: \answerYes{}
    \item[] Justification: The paper provides datasets, models, training procedure, evaluation metrics, and checkpointing details in Section~\ref{sec:results} and Appendix~\ref{app:perfect-knowledge} sufficient to reproduce the main results.

\item {\bf Open access to data and code}
    \item[] Question: Does the paper provide open access to the data and code, with sufficient instructions to faithfully reproduce the main experimental results, as described in supplemental material?
    \item[] Answer: \answerYes{} 
    \item[] Justification: The implementation is provided through an anonymized repository and experiments use public datasets and open-source pretrained models.

\item {\bf Experimental setting/details}
    \item[] Question: Does the paper specify all the training and test details (e.g., data splits, hyperparameters, how they were chosen, type of optimizer) necessary to understand the results?
    \item[] Answer: \answerYes{}
    \item[] Justification: The paper specifies training details, dataset splits, metrics, hyperparameters, and evaluation protocol in Section~\ref{sec:results} and Appendix~\ref{app:hyperparameters}.

\item {\bf Experiment statistical significance}
    \item[] Question: Does the paper report error bars suitably and correctly defined or other appropriate information about the statistical significance of the experiments?
    \item[] Answer: \answerYes{} 
    \item[] Justification: Tables ~\ref{tab:mainPerformance} and ~\ref{tab:llama38GSM8KFull} report accuracy with standard deviations computed from three fixed, non-overlapping subsets, indicating variability of results.

\item {\bf Experiments compute resources}
    \item[] Question: For each experiment, does the paper provide sufficient information on the computer resources (type of compute workers, memory, time of execution) needed to reproduce the experiments?
    \item[] Answer: \answerYes{} 
    \item[] Justification: The paper reports GPU type, wall-clock times, memory, storage, or total compute required across experiments.

\item {\bf Code of ethics}
    \item[] Question: Does the research conducted in the paper conform, in every respect, with the NeurIPS Code of Ethics \url{https://neurips.cc/public/EthicsGuidelines}?
    \item[] Answer: \answerYes{}
   \item[] Justification:  The work uses public datasets and pretrained models, preserves anonymity, and does not involve human subjects or sensitive data.

\item {\bf Broader impacts}
    \item[] Question: Does the paper discuss both potential positive societal impacts and negative societal impacts of the work performed?
    \item[] Answer: \answerYes{} 
     \item[] Justification: The paper briefly mentions broader impacts and safety and alignment in Appendix~\ref{app:broader-impacts} but does not provide a structured discussion of positive and negative societal impacts.
    
\item {\bf Safeguards}
    \item[] Question: Does the paper describe safeguards that have been put in place for responsible release of data or models that have a high risk for misuse (e.g., pre-trained language models, image generators, or scraped datasets)?
    \item[] Answer: \answerNA{}.
    \item[] Justification: The work does not release high-risk models or datasets that require special safeguards.

\item {\bf Licenses for existing assets}
    \item[] Question: Are the creators or original owners of assets (e.g., code, data, models), used in the paper, properly credited and are the license and terms of use explicitly mentioned and properly respected?
    \item[] Answer: \answerYes{}
    \item[] Justification: We use publicly available datasets, pretrained models, and libraries (including GSM8K, MMLU, IMDB, SST-2, HotpotQA, SQuAD v1.1, XSum, CNN/DailyMail, Dolly-15k, Alpaca, HumanEval, MBPP; LLaMA-3.1-8B, Mistral-7B-v0.3, Qwen3-8B-Base; PyTorch, Ripser, Gudhi), and we follow their respective licenses and terms of use as specified in their official releases.

\item {\bf New assets}
    \item[] Question: Are new assets introduced in the paper well documented and is the documentation provided alongside the assets?
    \item[] Answer: \answerYes{}
    \item[] Justification: The method and code are documented and released via an anonymized repository with sufficient detail for use.

\item {\bf Crowdsourcing and research with human subjects}
    \item[] Question: For crowdsourcing experiments and research with human subjects, does the paper include the full text of instructions given to participants and screenshots, if applicable, as well as details about compensation (if any)? 
    \item[] Answer: \answerNA{}.
    \item[] Justification: The paper does not involve human subjects or crowdsourcing.

\item {\bf Institutional review board (IRB) approvals or equivalent for research with human subjects}
    \item[] Question: Does the paper describe potential risks incurred by study participants, whether such risks were disclosed to the subjects, and whether Institutional Review Board (IRB) approvals (or an equivalent approval/review based on the requirements of your country or institution) were obtained?
    \item[] Answer: \answerNA{}.
    \item[] Justification: The work does not involve human subjects requiring IRB approval.

\item {\bf Declaration of LLM usage}
    \item[] Question: Does the paper describe the usage of LLMs if it is an important, original, or non-standard component of the core methods in this research? Note that if the LLM is used only for writing, editing, or formatting purposes and does \emph{not} impact the core methodology, scientific rigor, or originality of the research, declaration is not required.
    \item[] Answer: \answerNo{}
    \item[] Justification: The paper does not use LLMs as an important component in its design.  

\end{enumerate}
\end{document}